\documentclass[accepted,startpage=2244,endpage=2273]{uai2026} %

    \usepackage[american]{babel}
    
    \usepackage{natbib}
    \usepackage{mathtools}
    \usepackage{amsthm}
    \usepackage{booktabs}
    \usepackage{tikz}
    \usepackage{amsfonts}
    \usepackage{hyphenat}
    \usepackage{algorithm}
    \usepackage{algorithmicx}
    \usepackage{algpseudocode}
    \usepackage{multirow}
    \usepackage{placeins}
    \usepackage{graphicx}  
    \usepackage{enumitem}  
    \usepackage{titletoc}
    \usepackage{changepage}
    
    \newcommand{\printappendixtoc}{%
    \begingroup
    \section*{Table of Contents}
    
    \begin{adjustwidth}{0.06\textwidth}{0.008\textwidth}
    \hrule
    
    \setcounter{tocdepth}{2}
    
    \titlecontents{section}
      [0pt]
      {\addvspace{0.45em}\bfseries}
      {\contentslabel{2.8em}}
      {}
      {\titlerule*[0.6pc]{.}\contentspage}
    
    \titlecontents{subsection}
      [1.5em]
      {}
      {\contentslabel{3.8em}}
      {}
      {\titlerule*[0.6pc]{.}\contentspage}
    
    \printcontents[appendix]{}{1}{\setcounter{tocdepth}{2}}
    \end{adjustwidth}
    
    \endgroup
    }

    \definecolor{darkblue}{RGB}{0,0,100}
    \hypersetup{
      colorlinks=true,
      linkcolor=darkblue,
      citecolor=darkblue,
      urlcolor=blue
    }
    
    \newtheorem{theorem}{Theorem}
    
    \newtheorem{corollary}{Corollary}

    \newcommand{\pbterm}{\mathrm{KL}(\rho\Vert\pi)+\ln\tfrac{C_0\sqrt{m}}{\delta}}

    \title{Model-Agnostic Online Certificate-Driven Calibration \\ for Time Series Forecasting Under Distribution Shift}
    
    \author[1]{Chenfeng Huang}
    \author[1]{Zixuan Ma}
    \author[1]{George Michailidis}

    \affil[1]{%
        Department of Statistics and Data Science\\
        University of California, Los Angeles\\
        Los Angeles, California, USA
    }
    
\begin{document}
    \maketitle
    
    \begin{abstract}
    Time series out-of-distribution generalization requires forecasters to remain reliable when deployment dynamics differ
    from training conditions due to covariate shift, concept shift, and temporal dependence. Probably Approximately Correct
    Bayesian domain adaptation provides computable certificates by decomposing target risk into a source risk term, a
    source-to-target mismatch term, and a complexity term, but standard analyses rely on independent sampling and
    distributional stability, assumptions that are violated in time series by serial dependence and nonstationary shift. We
    propose a model-agnostic online martingale Probably Approximately Correct Bayesian framework that yields finite-sample
    certificates under temporal dependence and distribution shift. The certificate replaces independent-sample
    concentration with martingale concentration that adapts to loss scale and predictable variation. We use the certificate
    as a surrogate regularizer for online calibration by training a gated residual Bayesian head on top of a fixed
    forecasting backbone, producing a corrective update that reverts to the backbone prediction when the gate is closed.
    Online calibration combines a source risk anchor, a posterior-shift penalty, and a time-adaptive mismatch term computed
    from target windows observed before forecasting. It follows a predict-then-update protocol in which outcomes become
    available only after forecasting and are used to update subsequent predictions. Experiments across convolutional,
    attention-based, and large language model-based forecasters show improved stability and accuracy under covariate and
    concept shift.
    
    \end{abstract}

    \section{Introduction}\label{sec:intro}
    Time series forecasting models are frequently deployed in dynamic environments where the data distribution at inference diverges from the historical distribution used for training. 
    Time series out-of-distribution (OOD) research focuses on maintaining robustness under covariate and concept shifts, particularly when high serial correlation limits the effective information available to distinguish genuine distributional drift from noise \citep{wu2025tsoog,kuznetsov2015learning}. In practice, these effects are often inextricably linked: seasonality, operational fluctuations, and sensor noise can shift the input distribution even when the underlying mechanisms remain stable, whereas regime changes and interventions can fundamentally alter the predictive mapping from historical data to future data.
    
    
    A common operational response to nonstationarity is online time series forecasting, where a deployed model is updated
    continually as new observations arrive. Recent methods move beyond naive retraining and design specialized adaptation
    mechanisms. OneNet improves robustness to concept shift by maintaining an ensemble and dynamically adjusting
    combination weights as the stream evolves \citep{zhang2023onenet}. SOLID targets context-driven distribution shift by
    detecting context changes at test time and calibrating a pretrained forecaster through selective retrieval and
    lightweight updates to a prediction layer before producing a forecast \citep{chen2024solid}. However, these adaptive strategies face a fundamental structural constraint in multi-step settings: outcome feedback is inevitably delayed by the prediction horizon, causing the available labeled data to lag behind the current regime. As formalized by PROCEED, ignoring this delay leads to infeasible information leakage, whereas realistic ``predict-then-observe" protocols often degrade performance, as heavy online tuning can lead to overfitting when feedback is lagged \citep{zhao2025proceed}. This creates a critical tension: while frequent updates are necessary to catch up with distribution shifts, the latency of reliable supervision limits the effectiveness of purely heuristic online adaptation.
    
    These findings underscore two critical prerequisites for online forecasting under distribution shift. First, the method must detect distribution changes relying exclusively on inputs available at the \textit{moment of prediction}, as recent outcomes are often latent. Second, when feedback eventually arrives, updates must be conservative; because delayed feedback may reflect a regime that has since passed, aggressive tuning risks overfitting or destabilizing the model trajectory. This necessitates a \textit{certificate-driven} approach that quantifies mismatch using real-time inputs to regularize online updates against the noise of delayed supervision.
    
    To instantiate this certificate-driven strategy, Probably Approximately Correct Bayesian (PAC-Bayes) Domain Adaptation offers a rigorous framework that provides computable certificates by decomposing target risk into three components: source risk, complexity, and a source-to-target mismatch term \citep{mcallester1999pac,germain2016pac}.
    This mismatch term is particularly important for online forecasting because it relies solely on target inputs, enabling shift-aware regularization even when ground-truth outcomes are delayed. However, classical PAC-Bayes analyses assume independent sampling and distributional stability—conditions that are violated in time series settings by serial dependence and nonstationary shifts. Consequently, naive application of these bounds yields certificates that are either invalid or overly optimistic \citep{ralaivola2010chromatic}.
    
    To address these limitations, we propose Online Martingale PAC-Bayes (OMPB)\footnote{Code implementation is available at \url{https://github.com/chenfeng-huang/OMPB-UAI-2026}}, a certificate-driven calibration framework grounded in martingale theory for time series OOD generalization. OMPB substitutes standard independent-sample concentration with variance-adaptive martingale concentration tailored for dependent streams \citep{seldin2012pac}, leveraging the resulting certificate as a surrogate regularizer for online adaptation. At each deployment step, the method computes a time-adaptive mismatch signal based on posterior-weighted prediction disagreement over strictly antecedent windows, providing an immediate shift indicator that functions without instantaneous outcome feedback. Subsequently, as delayed outcomes materialize, OMPB calibrates a lightweight gated residual Bayesian head atop a fixed backbone. This \textit{gating mechanism} enforces conservative updates, automatically reverting to the backbone’s baseline prediction whenever adaptation is uncertain—yielding a practical, model-agnostic solution for robust forecasting across diverse architectures.
    
    The key contributions of the paper are:
    \begin{itemize}[leftmargin=*]
    \item We introduce OMPB, an online certificate-driven calibration framework that leverages variance-adaptive, martingale-style PAC-Bayes regularization to provide finite-sample guarantees for time series forecasting under distribution shift.
    \item We derive a time-adaptive mismatch signal computed strictly from antecedent target windows, enabling immediate shift detection that circumvents the latency of delayed outcome supervision.
    \item We develop an online calibration procedure for a gated residual Bayesian head that optimizes the certificate as a surrogate objective, improving adaptation stability while maintaining the fixed backbone as a safe fallback.
    \end{itemize}

    \section{Related Work}\label{sec:related}
    
    \paragraph{Time Series OOD Generalization Bounds:}
    Generalization bounds characteristically upper-bound the true population risk using the empirical risk incurred on a sample, augmented by a complexity penalty \citep{dong2023kernelized-renyi}.
    
    Consider a hypothesis $h=\mathcal{A}(z)$ learned from an independent and identically distributed (i.i.d.) sample $z\sim P^{n}$. A standard high-probability guarantee takes the form:
    \begin{equation}
    \begin{aligned}
    \mathbb{P}_{z\sim P^{n}}\!\left[L_{P}(h)\le \widehat{L}_{z}(h)+\varepsilon\right]
    \ \ge\ 1-\delta,
    \end{aligned}
    \label{eq:gen-iid}
    \end{equation}
    where $L_{P}(h)=\mathbb{E}_{(x,y)\sim P}[\ell(h(x),y)]$ denotes the population risk and $\widehat{L}_{z}(h)=\frac{1}{n}\sum_{i=1}^{n}\ell(h(x_i),y_i)$ denotes the empirical risk.
    
    Time series data inherently violate the independent and identically distributed (i.i.d.) assumption due to the presence of temporal dependencies. To rigorously address this, mixing conditions—such as $\beta$-mixing and $\phi$-mixing—are employed to quantify the rate at which dependence decays, thereby enabling the extension of learning theory to dependent streams. For non-stationary mixing processes, \citet{kuznetsov2017generalization} derived non-asymptotic bounds that explicitly incorporate mixing coefficients and distribution shift. These typically take the form of an average-case bound:
    \begin{equation}
    \begin{aligned}
    L_{P}(h)
    \ \le\ \widehat{L}_{z}(h)
    \ +\ \Psi\!\big(\beta,\ \mathcal{D}_{\mathrm{src}\to\mathrm{tgt}},\ m\big),
    \end{aligned}
    \label{eq:avg-path-bound}
    \end{equation}
    and a stronger, path-dependent form:
    \begin{equation}
    \begin{aligned}
    L_{P}(h)
    \ \le\ \widehat{L}_{z}(h)
    \ +\ C\cdot \Phi\!\big(\phi,\ \mathcal{C}_{\mathrm{seq}},\ m\big),
    \end{aligned}
    \label{eq:path-dependent-bound}
    \end{equation}
    where the first bound Equation~\eqref{eq:avg-path-bound} relies on an average dependence measure and source-to-target discrepancy, while the second Equation~\eqref{eq:path-dependent-bound} holds uniformly over sample paths, governed by a worst-case dependence measure and the realized sequence complexity. In practice, however, utilizing these bounds is challenging due to the difficulty of estimating mixing coefficients and the looseness of constants under regime shifts.

    \paragraph{PAC-Bayes Generalization and Domain Adaptation:}
    PAC-Bayes theory provides data-dependent generalization guarantees for randomized predictors by bounding the population risk of a Gibbs predictor in terms of its empirical risk and a complexity penalty, defined by the Kullback\textendash Leibler (KL) divergence from a prior distribution \citep{mcallester1999pac}. A classical i.i.d. formulation is stated next.
    
    \begin{theorem}[PAC-Bayes Bound for i.i.d. Data]\footnote{Proof is included in Appendix~\ref{sec:proof_pacbayes}.}
    \label{th:PAC_Bayes}
    We restate the classical PAC-Bayes generalization bound due to \citet{mcallester1999pac}, with related refinements discussed in \citet{seeger2002,catoni2007}. For any distribution $\mathcal{D}$ over $\mathcal{X}\times\mathcal{Y}$, any hypothesis class $\mathcal{H}$, any prior $\pi$ over $\mathcal{H}$, and any $\delta\in(0,1]$, with probability at least $1-\delta$ over $S\sim \mathcal{D}^{m}$, the following holds for every posterior distribution $\rho$ on $\mathcal{H}$:
    \begin{equation}
    \begin{aligned}
    R_{\mathcal{D}}(h_{\rho})
    \ \le\
    \widehat{R}_{S}(h_{\rho})
    \ +\
    \sqrt{
    \frac{\mathrm{KL}(\rho \Vert \pi) + \ln \frac{2\sqrt{m}}{\delta}}{2m}
    }.
    \end{aligned}
    \label{eq:pacbayes-iid}
    \end{equation}
    Here, the KL divergence is defined as:
    \begin{equation}
    \begin{aligned}
    \mathrm{KL}(\rho \Vert \pi)
    \ :=\
    \sum_{h \in \mathcal{H}} \rho(h)\,
    \ln \frac{\rho(h)}{\pi(h)}.
    \end{aligned}
    \label{eq:kl-def}
    \end{equation}
    \end{theorem}
    
    PAC-Bayes Domain Adaptation extends Theorem~\ref{th:PAC_Bayes} by introducing terms that account for distribution shift. Specifically, it adds a \textit{discrepancy term} measuring the \textit{mismatch between source and target input marginals}—often instantiated as posterior-weighted disagreement—along with an irreducible joint-error term \citep{germain2016pac}. A representative i.i.d. PAC-Bayes Domain Adaptation bound is given next.
    
    \begin{theorem}[i.i.d. PAC-Bayes Domain Adaptation \citep{germain2013pac,germain2016pac}]\footnote{Proof is included in Appendix~\ref{sec:proof_pac_da}.}
    \label{th:PAC_DA}
    For any hypothesis class $\mathcal{H}$, any prior $\pi$ on $\mathcal{H}$, and any $\delta\in(0,1]$, with probability at
    least $1-\delta$ over an i.i.d. sample $S\sim(\mathcal{S})^{m}$, for every posterior $\rho$ on $\mathcal{H}$,
    \begin{equation}
    \begin{aligned}
    R_{\mathcal{T}}(h_{\rho})
    &\le
    \widehat{R}_{S}(h_{\rho})
    +
    \sqrt{
    \frac{\mathrm{KL}(\rho \Vert \pi) + \ln \frac{2\sqrt{m}}{\delta}}{2m}
    }\\
    &\quad
    +\frac{1}{2}\,\mathrm{dis}_{\rho}(\mathcal{S},\mathcal{T})
    +\lambda_{\rho},
    \end{aligned}
    \label{eq:iid-pbda}
    \end{equation}
    where $d(\cdot,\cdot;w)$ is a bounded disagreement on an input window $w$, the disagreement risk on a domain
    $\mathcal{D}$ is
    \begin{equation}
    \begin{aligned}
    R_{\mathcal{D}}(h,h')
    :=
    \mathbb{E}_{w\sim \mathcal{D}_{\mathcal{X}}}\!\big[d(h,h';w)\big],
    \end{aligned}
    \label{eq:disagreement-risk}
    \end{equation}
    the posterior-weighted source\textendash target disagreement is
    \begin{equation}
    \begin{aligned}
    \mathrm{dis}_{\rho}(\mathcal{S},\mathcal{T})
    &=
    \left|
    \mathbb{E}_{h,h' \sim \rho}\, R_{\mathcal{S}}(h,h')
    -
    \mathbb{E}_{h,h' \sim \rho}\, R_{\mathcal{T}}(h,h')
    \right|,
    \end{aligned}
    \label{eq:domain-disagreement}
    \end{equation}
    and the irreducible joint-error term is defined via the joint error
    \begin{equation}
    \begin{aligned}
    e_{\mathcal{D}}(h,h')
    :=
    \mathbb{E}_{(w,y)\sim \mathcal{D}}
    \Big[
    \mathbf{1}\{h(w)\neq y\}\,\mathbf{1}\{h'(w)\neq y\}
    \Big],
    \end{aligned}
    \label{eq:joint-error-def}
    \end{equation}
    as
    \begin{equation}
    \begin{aligned}
    \lambda_{\rho}
    :=
    \left|
    \mathbb{E}_{h,h' \sim \rho}\, e_{\mathcal{T}}(h,h')
    -
    \mathbb{E}_{h,h' \sim \rho}\, e_{\mathcal{S}}(h,h')
    \right|,
    \end{aligned}
    \label{eq:lambda-rho}
    \end{equation}
    which captures mismatch in the labeling mechanism and is not estimable without target labels.
    
    \end{theorem}
    In our forecasting experiments, we later replace the binary disagreement indicator $d(h,h';w)$ with a bounded prediction-discrepancy
    surrogate for shift sensing; this surrogate is used as an input-driven regularizer rather than as the exact disagreement
    indicator underlying the classification decomposition.
    
    \paragraph{PAC-Bayes Certificates as Surrogate Objectives:}
    Beyond their traditional role in post hoc certification, PAC-Bayes bounds can be directly optimized to learn a posterior distribution by treating the certificate \textit{itself} as the training objective. Fundamentally, these bounds express an explicit trade-off between empirical risk and a KL complexity penalty relative to a prior distribution, effectively defining a regularized learning rule for stochastic predictors \citep{catoni2007}. In the context of deep learning, optimizing these objectives has been shown to yield predictors with informative, non-vacuous generalization guarantees \citep{dziugaite2017}.
    
    Making this approach computationally practical, "PAC-Bayes with Backprop" derives differentiable objectives that allow probabilistic neural networks to be trained via backpropagation through the bound \citep{rivasplata2019pacbayes}. More recently, surrogate PAC-Bayes learning has introduced iterative objectives that preserve theoretical guarantees while significantly reducing optimization costs \citep{picardweibel2024learning}. This line  of research motivates our strategy for OMPB: utilizing the PAC-Bayes certificate not merely for passive online monitoring, but as an active surrogate regularizer for the deployment-time calibration of the Bayesian head.
    
    Because PAC-Bayes bounds hold uniformly over all posteriors $\rho$ with high probability, optimizing such a certificate constitutes a valid learning rule defined on bounded proxy quantities \citep{mcallester1999pac,catoni2007}. Leveraging this perspective, OMPB employs the certificate as a surrogate regularizer for online calibration to ensure theoretical stability.
    
    \section{Preliminaries}\label{sec:prelim}
    
    \subsection{Distribution Shift Categories}
    Time series OOD generalization is primarily characterized by two distinct forms of distributional instability \citep{sugiyama2012covariate,gama2014survey,kuznetsov2015learning,wu2025tsoog}:
    
    \begin{enumerate}[leftmargin=*]
        \item \textbf{Covariate shift:} The input marginal distribution $P(X)$ evolves while the conditional predictive distribution $P(Y\mid X)$ remains invariant. In time series data applications, this often manifests through sensor replacement, recalibration, seasonal fluctuations, or changes in operating conditions that alter feature magnitudes without changing the underlying data generating mechanism.
        \item \textbf{Concept shift:} The conditional distribution $P(Y\mid X)$ changes while the input marginal $P(X)$ remains stable. This corresponds to a fundamental drift in the data-generating mechanism, caused by factors such as external interventions, regime changes, or evolving system dynamics.
    \end{enumerate}
    
    \subsection{Martingale PAC-Bayes Domain Adaptation}\label{sec:martingale-pbda}
    
    The i.i.d. PAC-Bayes and PAC-Bayes domain adaptation results in Theorems~\ref{th:PAC_Bayes} and~\ref{th:PAC_DA} premised upon
    independent sampling to control the deviation between empirical and expected source risk
    \citep{mcallester1999pac,seeger2002,germain2016pac}. For time series data, the training sequence exhibits temporal dependence and may be nonstationary, so i.i.d. concentration is no longer valid and can yield overly optimistic certificates when applied directly \citep{ralaivola2010chromatic,kuznetsov2015learning}. We resolve this discrepancy by rigorously modeling the labeled source stream as a filtered stochastic process $(Z_t,\mathcal{F}_t)_{t=1}^{m}$. By substituting i.i.d. concentration with martingale concentration mechanisms, we derive dependence-aware PAC-Bayes domain adaptation bounds that remain valid under serial correlation \citep{seldin2012pac,freedman1975tail}.
    
    Let $(Z_t,\mathcal{F}_t)_{t=1}^m$ be a filtered source process generating the training sequence $S$.
    
    \begin{theorem}[Bounded Martingale PAC-Bayes Domain Adaptation]
    \label{th:bound-martingale}
    \footnote{Proof is included in Appendix~\ref{sec:proof_bound_martingale}. Derived from foundational results in \citet{germain2013pac,germain2016pac,freedman1975tail,seldin2012pac}.}

    Assume the loss is bounded such that $\ell(h,Z_t)\in[0,1]$. For any hypothesis class $\mathcal{H}$, any prior distribution $\pi$ over $\mathcal{H}$, and any confidence level $\delta\in(0,1]$, with probability at least $1-\delta$ over the sample $S$, the following inequality holds for every posterior distribution $\rho$ on $\mathcal{H}$:
    \begin{equation}
    \begin{aligned}
    R_{\mathcal{T}}(h_\rho)
    \ \le\
    \widehat{R}_{S}(h_\rho)
    \ +\
    \Gamma_m^{\mathrm{Freed}}(\rho,\pi,\delta)
    \\\ +\
    \frac{1}{2}\,\mathrm{dis}_{\rho}(\mathcal{S},\mathcal{T})
    \ +\
    \lambda_{\rho}.
    \end{aligned}
    \label{eq:martingale-da-bound}
    \end{equation}
    The complexity term $\Gamma_m^{\mathrm{Freed}}$, governed by Freedman's inequality, is defined as:
    \begin{equation}
    \begin{aligned}
    \Gamma_m^{\mathrm{Freed}}(\rho,\pi,\delta)
    &=
    \sqrt{
    \frac{
    2\,\mathbb{E}_{h\sim\rho}\!\left[V_m(h)\right]\,
    \left( \mathrm{KL}(\rho \Vert \pi) + \ln \tfrac{2\sqrt{m}}{\delta} \right)
    }{m^2}
    }
    \\
    &\quad
    +\frac{B}{3m}\left( \mathrm{KL}(\rho \Vert \pi) + \ln \tfrac{2\sqrt{m}}{\delta} \right).
    \end{aligned}
    \label{eq:gamma-freed}
    \end{equation}
    where $V_m(h)$ represents the cumulative conditional variance of the loss process:
    \begin{equation}
    \begin{aligned}
    V_m(h)
    \ :=\
    \sum_{t=1}^{m}\mathrm{Var}\left(\ell(h,Z_t)\mid\mathcal{F}_{t-1}\right).
    \end{aligned}
    \label{eq:Vm-freed}
    \end{equation}
    Here, $B$ is an upper bound on the martingale difference increments; for losses $\ell\in[0,1]$, we set $B=1$.
    \end{theorem}
    
    Crucially, the bounded-increment assumption in Theorem~\ref{th:bound-martingale} is frequently violated in time series forecasting, where standard regression losses are inherently unbounded. While loss clipping is a common workaround, it introduces estimation bias and yields effectively vacuous certificates in the presence of spikes or heavy-tailed noise.
    To resolve this, we replace the bounded constraint with a \textit{conditional sub-gamma} assumption. This relaxation controls the conditional moment generating function via a predictable variance proxy and a scale parameter, enabling a sharper Bernstein-type concentration. The resulting guarantee, formalized in Theorem~\ref{th:bound-subgamma}, preserves the tractable PAC-Bayes structure while remaining rigorous for unbounded forecasting objectives.
    
    \begin{theorem}[Sub-Gamma Martingale PAC-Bayes Domain Adaptation \citep{germain2013pac,germain2016pac,seldin2012pac,bercu2008}]\footnote{Proof is included in Appendix~\ref{sec:proof_bound_subgamma}.}
    \label{th:bound-subgamma}
    For each $h\in\mathcal H$, define $X_t(h):=\ell(h,Z_t)-\mathbb E[\ell(h,Z_t)\mid\mathcal F_{t-1}]$.
    Assume there exist predictable $v_t(h)\ge 0$ and $c(h)\ge 0$ such that
    \begin{equation}
    \begin{aligned}
    \mathbb E\!\left[\exp\{\lambda X_t(h)\}\mid\mathcal F_{t-1}\right]
    &\le 
    \exp\!\left(\frac{\lambda^2 v_t(h)}{2(1-c(h)\lambda)}\right),
    \qquad \\
    &\forall \lambda\in\big(0,1/c(h)\big).
    \end{aligned}
    \label{eq:subgamma-condition}
    \end{equation}
    Let $V_m(h):=\sum_{t=1}^{m} v_t(h)$ and $\bar c:=\sup_{h\in\mathcal H} c(h)$. Then for any $\mathcal{H}$, prior $\pi$,
    and $\delta\in(0,1]$, with probability at least $1-\delta$ over $S$, for every posterior $\rho$ on $\mathcal H$,
    \begin{equation}
    \begin{aligned}
    R_{\mathcal{T}}(h_\rho)
    &\le
    \widehat{R}_{S}(h_\rho)
    +
    \Gamma_m^{\mathrm{sub}\,\gamma}(\rho,\pi,\delta)
    +\tfrac{1}{2}\,\mathrm{dis}_{\rho}(\mathcal{S},\mathcal{T})
    +\lambda_{\rho}.
    \end{aligned}
    \label{eq:subgamma-da-bound}
    \end{equation}
    Moreover,
    \begin{equation}
    \begin{aligned}
    \Gamma_m^{\mathrm{sub}\,\gamma}(\rho,\pi,\delta)
    &=
    \sqrt{
    \frac{
    2\,\mathbb{E}_{h\sim\rho}\!\big[V_m(h)\big]\,
    \pbterm
    }{m^2}
    }\\
    &+
    \frac{\bar c}{m}\,\pbterm.
    \end{aligned}
    \label{eq:gamma-subgamma}
    \end{equation}
    \end{theorem}

    \section{Online Martingale PAC-Bayes Framework (OMPB)}\label{sec:method}
    \begin{figure}
        \centering
        \includegraphics[width=0.85\linewidth]{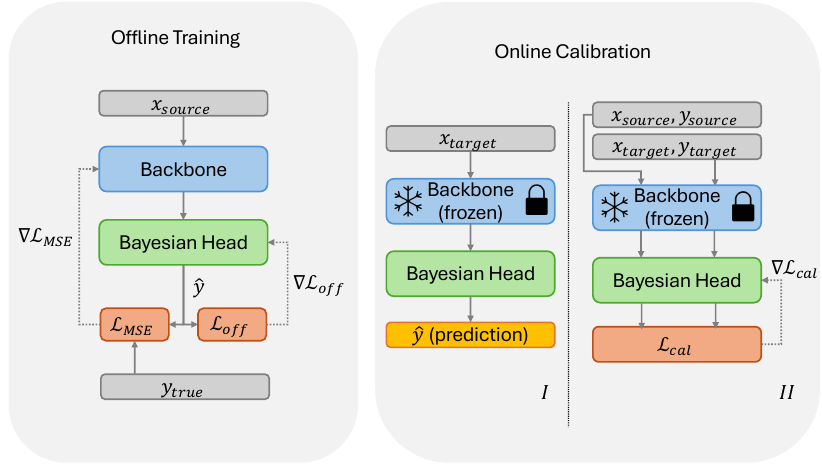}
        \caption{The left plot illustrates the offline training stage, while the right plot depicts the online calibration stage, in which predictions are generated first and the Bayesian head is updated afterward, ensuring causal validity.}
        \label{fig:pipeline}
    \end{figure}
    The OMPB framework operates via a two-stage procedure. In the offline stage, we pre-train a forecasting backbone on labeled source data and initialize a lightweight gated residual Bayesian head to yield a near-zero residual, ensuring a safe fallback to the backbone. In the online stage, the backbone remains frozen while the Bayesian head is updated recursively under a ``predict-then-update" protocol, strictly adhering to temporal causality to prevent look-ahead bias. Figure~\ref{fig:pipeline} illustrates this workflow.
    
    The remainder of this section first formalizes the gated residual Bayesian head and its offline initialization. We then detail the online adaptation process, specifically describing how OMPB estimates source-target mismatch, computes the time-indexed certificate, and updates the head within these causal constraints.
    
    \subsection{Gated Residual Bayesian Head}\label{sec:bayes-head}
    To ensure the tractability of the complexity term $\mathrm{KL}(\rho\Vert\pi)$, we treat the high-dimensional backbone as deterministic and place the posterior exclusively on the parameters of a lightweight head. This avoids the vacuous certificates often associated with applying a full Bayesian treatment to deep feature extractors.

    Let backbone map an input window $w$ to a base $H$-step forecast $\mathbf z(w)\in\mathbb{R}^{H}$. 
    The final prediction is a gated residual correction of the backbone forecast
    \begin{equation}
    \begin{aligned}
    \widehat{\mathbf y}(w)
    &=
    \mathbf z(w)
    +
    s\cdot\big(\mathbf z(w)\,\Delta W^\top + \Delta \mathbf b\big),
    \qquad \\
    & s:=\sigma(\alpha)\in(0,1),
    \end{aligned}
    \label{eq:gated-residual-head}
    \end{equation}
    where $\Delta W\in\mathbb{R}^{H\times H}$ and $\Delta \mathbf b\in\mathbb{R}^{H}$ are the core learnable correction
    parameters of the Bayesian head and $\sigma(\cdot)$ is the logistic function. This makes the backbone a stable default:
    when $s$ is small and $(\Delta W,\Delta\mathbf b)$ are near zero, $\widehat{\mathbf y}(w)\approx \mathbf z(w)$.
    
    We place an isotropic Gaussian prior on the vectorized correction parameters $\pi(\theta)=
    \mathcal N(0,\sigma_0^2 I)$ with $\theta=\mathrm{vec}(\Delta W,\Delta \mathbf b)$ and use a diagonal Gaussian posterior $\rho(\theta)=\mathcal N(\mu,\mathrm{diag}(\sigma^2))$, yielding the closed-form
    complexity term
    \begin{equation}
    \begin{aligned}
    \mathrm{KL}\!\left(\rho\;\Vert\;\pi\right)
    &=
    \frac{1}{2}\sum_{i=1}^{P}
    \left(
    \frac{\sigma_i^{2}+\mu_i^{2}}{\sigma_0^{2}}
    - 1
    - \ln\frac{\sigma_i^{2}}{\sigma_0^{2}}
    \right),
    \end{aligned}
    \label{eq:kl-diag-gaussian}
    \end{equation}
    where $P$ is the number of head parameters in $\theta$. The gate parameter $\alpha$ is optimized jointly with the head
    and is initialized and regularized to encourage a no-harm start.
    
    \subsection{Offline Training}\label{sec:offline}
    We first acquire the backbone parameters via standard empirical risk minimization on the labeled source sequence
    $(W_S,Y_S)=\{(w_t,y_t)\}_{t=1}^{m}$. Subsequently, we optimize the Bayesian head on these source windows by minimizing a composite objective consisting of the supervised forecasting loss and a KL regularization term relative to the prior distribution.
    \begin{equation}
    \begin{aligned}
    \mathcal{L}_{\mathrm{off}}
    &=
    \widehat{R}_S^{\mathrm{sup}}(\rho)
    +\lambda_{\mathrm{prior}}\,\mathrm{KL}(\rho\Vert\pi),
    \end{aligned}
    \label{eq:offline-objective}
    \end{equation}
    where $\widehat{R}_S^{\mathrm{sup}}(\rho)$ is the empirical source forecast loss of the head evaluated with the posterior
    mean predictor. In implementation, we normalize the KL term by the number of head-training windows $N$, so
    $\lambda_{\mathrm{prior}}=1/N$. We use a fixed head-training size $N$ across all datasets to keep the
    regularization scale consistent. We also compute and cache source-dependent scale statistics used to evaluate the
    martingale correction term $\Gamma_m^{\mathrm{sub}\,\gamma}(\rho,\pi,\delta)$ during online certification and calibration.
    
    \subsection{Online Calibration}\label{sec:online}
    At deployment time $t$, we maintain three distinct data pools: a sliding window of recent unlabeled target inputs $W_{T,t}$ for real-time shift sensing; a fixed replay set of labeled source windows serving as a stability anchor; and the current target instance $(w'_t,y'_t)$, where the ground-truth outcome $y'_t$ is revealed strictly after the forecast is issued.
    
    \paragraph{Bounded prediction disagreement.}
    To obtain an input-side mismatch signal compatible with the bounded proxy certificate, we define the posterior
    prediction disagreement on a target window \(w\) as
    \begin{equation}
    \begin{aligned}
    \widetilde d_{\tau_d}(h,h';w)
    :=
    \min\!\left\{
    1,\,
    \frac{\|h(w)-h'(w)\|_2^2}{HC\,\tau_d^2}
    \right\}
    \in[0,1],
    \end{aligned}
    \label{eq:d_reg_code}
    \end{equation}
    where \(H\) is the prediction horizon and \(C\) is the number of target variables. The normalization by \(HC\) makes the
    score an average posterior forecast dispersion rather than a raw dimension-dependent norm. \footnote{Because this disagreement estimate depends on posterior samples, its informativeness requires the posterior to retain
    meaningful epistemic diversity; we discuss posterior-miscalibration effects in
    Appendix~\ref{sec:posterior_miscalibration}.}

    The scale \(\tau_d\) is selected once from the source replay buffer before target adaptation:
    \begin{equation}
    \begin{aligned}
    \tau_{\mathrm{auto}}
    =
    \sqrt{
    \mathrm{Quantile}_{q}
    \left(
    \frac{\|h^{(k)}(w)-h^{(k')}(w)\|_2^2}{HC}
    \right)
    },
    \end{aligned}
    \label{eq:tau-auto}
    \end{equation}
    where the quantile is computed over posterior-sample pairs and source probe windows. We use \(q=0.5\), so
    \(\tau_{\mathrm{auto}}\) is the median source-side RMS posterior disagreement.
    
    Using posterior samples \(\{h^{(k)}\}_{k=1}^{K}\sim\rho\), we estimate per-window disagreement by unordered posterior
    pairs:
    \begin{equation}
    \widehat d_{\rho}(w)
    :=
    \frac{2}{K(K-1)}
    \sum_{1\le k<k'\le K}
    \widetilde d_{\tau_d}\!\big(h^{(k)},h^{(k')};w\big).
    \label{eq:perwindow-disagree}
    \end{equation}
    The online source--target mismatch is
    \begin{equation}
    \begin{aligned}
    \widehat{\mathrm{dis}}_\rho(t)
    =
    \left|
    \frac{1}{m}\sum_{t'=1}^{m}\widehat d_\rho(w_{t'})
    -
    \frac{1}{|W_{T,t}|}\sum_{w'\in W_{T,t}}\widehat d_\rho(w')
    \right|.
    \end{aligned}
    \label{eq:disagreement-online}
    \end{equation}
    With the default \(K=5\), OMPB uses \(\binom{5}{2}=10\) posterior-pair comparisons per online batch.
    \begin{table*}[t]
    \centering
    
    \caption{ETTh forecasting results (MAE | MSE) under In-Distribution (ID) and Out-of-Distribution (OOD) settings. Lower score values indicate better performance, and the best score is in \textbf{bold}. Additional results for more prediction horizon (H) are included in Appendix~\ref{sec:extended_ett_result}.}
    \begin{adjustbox}{max width=\textwidth}
    \label{tab:etth_results}
    \setlength{\tabcolsep}{9pt}
    \renewcommand{\arraystretch}{0.4}
    \begin{tabular}{cccccccc}
    \toprule
    \multirow{2}{*}{Backbone} &
    \multirow{2}{*}{Method} &
    \multicolumn{2}{c}{H=24} &
    \multicolumn{2}{c}{H=96} &
    \multicolumn{2}{c}{H=336} \\
    \cmidrule(lr){3-4}\cmidrule(lr){5-6}\cmidrule(lr){7-8}
    &
    & ID (ETTh1) & OOD (ETTh2)
    & ID (ETTh1) & OOD (ETTh2)
    & ID (ETTh1) & OOD (ETTh2) \\
    \midrule
    \multirow{5}{*}{TCN}
    & Original
    & 0.4588 | 0.4521 & 2.5145 | 8.9095
    & 0.5330 | 0.5715 & 2.5635 | 8.9076
    & 0.6193 | 0.7304 & 3.1380 | 13.7387 \\
    & SOLID
    & 0.4741 | 0.4333 & 0.7553 | 1.0175
    & 0.5363 | 0.5389 & 0.8886 | 1.4424
    & 0.7110 | 0.9710 & 0.9488 | 1.7629 \\
    & OneNet
    & 0.5838 | 0.6739 & 0.8234 | 1.3034
    & 0.6606 | 0.8488 & 1.0073 | 1.9592
    & 0.7082 | 0.9191 & 1.2228 | 3.3373 \\
    & PROCEED
    & 0.4917 | 0.5161 & 1.1002 | 0.7768
    & 0.5889 | 0.7086 & 0.9608 | 1.8119
    & 0.6490 | 0.8225 & 1.1293 | 2.5799 \\
    
    & OMPB
    & \textbf{0.3635} | \textbf{0.2724} & \textbf{0.5292} | \textbf{0.5659}
    & \textbf{0.3813} | \textbf{0.2800} & \textbf{0.5570} | \textbf{0.5912}
    & \textbf{0.4484} | \textbf{0.3674} & \textbf{0.5981} | \textbf{0.6782} \\
    \midrule
    \multirow{5}{*}{Autoformer}
    & Original
    & 0.5147 | 0.5605 & 1.1207 | 1.9961
    & 0.5325 | 0.5651 & 1.3298 | 2.6778
    & 0.6830 | 0.8846 & 1.1962 | 2.3453 \\
    & SOLID
    & 0.4899 | 0.4813 & 0.8410 | 1.2997
    & 0.5446 | 0.5629 & 0.9501 | 1.6828
    & 0.5937 | 0.6374 & 1.0386 | 2.0287 \\
    & OneNet
    & 0.4786 | 0.4815 & 0.7311 | 0.9580
    & 0.5571 | 0.6026 & 0.8395 | 1.2940
    & 0.6140 | 0.6861 & 0.9812 | 1.8274 \\
    & PROCEED
    & 0.5119 | 0.5967 & 0.7580 | 1.1151
    & 0.6827 | 1.0923 & 1.0509 | 2.7848
    & 0.8767 | 1.8344 & 1.5051 | 4.6450 \\
    
    & OMPB
    & \textbf{0.3814} | \textbf{0.3018} & \textbf{0.6785} | \textbf{0.8535}
    & \textbf{0.3845} | \textbf{0.2858} & \textbf{0.6033} | \textbf{0.6882}
    & \textbf{0.4373} | \textbf{0.3481} & \textbf{0.6461} | \textbf{0.7764} \\
    \midrule
    \multirow{5}{*}{GPT4TS}
    & Original
    & 0.3794 | 0.3217 & 0.5485 | 0.7119
    & 0.4473 | 0.4167 & 0.6738 | 1.0241
    & 0.5320 | 0.5588 & 0.8547 | 1.5203 \\
    & SOLID
    & 0.3763 | 0.3218 & 0.6172 | 0.7745
    & 0.4451 | 0.4160 & 0.7499 | 1.1224
    & 0.5333 | 0.5625 & 0.9131 | 1.6823 \\
    & OneNet
    & 0.3740 | 0.3184 & 0.6115 | 0.7512
    & 0.4466 | 0.4146 & 0.7489 | 1.1222
    & 0.5364 | 0.5572 & 0.8975 | 1.6239 \\
    & PROCEED
    & 0.4843 | 0.5082 & 0.7210 | 1.0268
    & 0.6237 | 0.8802 & 0.9364 | 1.8907
    & 0.6911 | 1.0597 & 1.1482 | 2.8742 \\
    
    & OMPB
    & \textbf{0.3156} | \textbf{0.2165} & \textbf{0.5242} | \textbf{0.6616}
    & \textbf{0.3535} | \textbf{0.2457} & \textbf{0.5708} | \textbf{0.7377}
    & \textbf{0.4095} | \textbf{0.3158} & \textbf{0.5719} | \textbf{0.7304} \\
    \bottomrule
    \end{tabular}
    \end{adjustbox}
    \end{table*}
    \paragraph{Online certificate:}
    OMPB forms a time-indexed certificate
    \begin{equation}
    \begin{aligned}
    \widehat{\mathrm{PB}}_{\gamma}(t)
    &=
    \widehat{R}_{S,t}(h_\rho)
    +\Gamma_m^{\mathrm{sub}\,\gamma}(\rho,\pi,\delta)
    +\tfrac{1}{2}\widehat{\mathrm{dis}}_\rho(t),
    \end{aligned}
    \label{eq:pb-online}
    \end{equation}
    where $\widehat{R}_{S,t}(h_\rho)$ is re-estimated online from the labeled source replay set to keep the certificate
    coupled to the current posterior. Equation~\eqref{eq:pb-online} is the computable, input-identifiable part of the
    PAC-Bayes domain-adaptation certificate in Theorem~\ref{th:bound-subgamma}. Specifically, the population source risk is
    replaced by its empirical source-replay estimate, the martingale correction is computed from the Bayesian-head posterior
    and source-estimated predictable variance proxy, and the source--target disagreement term is estimated from source
    windows and currently observed unlabeled target windows. We omit $\lambda_\rho$ in
    Equation~\eqref{eq:subgamma-da-bound} because it is a residual labeling-mismatch term that depends on target labels and
    is unavailable at prediction time under delayed and partial target feedback. Thus,
    $\widehat{\mathrm{PB}}_{\gamma}(t)$ is not an independent heuristic; it is the online surrogate induced by the
    input-identifiable terms of the certificate.
    
    In forecasting, this certificate is defined for bounded proxy quantities rather than the raw MAE and MSE reported in
    experiments. We therefore do not claim that Equation~\eqref{eq:iid-pbda} directly upper-bounds raw forecasting error.
    Instead, the mismatch term uses the bounded proxy discrepancy
    $\widetilde{d}(\cdot,\cdot;w)\in[0,1]$ in Equation~\eqref{eq:d_reg_code}, while MAE/MSE are evaluated separately as
    task metrics. The proxy remains informative under shift because $\widetilde{d}$ increases when posterior samples
    produce divergent forecasts on the same input. The sub-gamma correction in Theorem~\ref{th:bound-subgamma} is
    instantiated using predictable variance and scale statistics estimated from source residuals. Appendix~\ref{sec:reg_surrogate}
    formalizes the regression-proxy interpretation, and Appendix~\ref{sec:reg_proxy_diagnostics} reports post-hoc
    diagnostics comparing realized target proxy risk with full martingale and i.i.d. proxy certificates in
    Table~\ref{tab:reg_proxy_full_cert} and Figure~\ref{fig:reg_proxy_bound_curves}. These full diagnostics include the
    residual proxy mismatch term requiring target labels, so they are used only to assess numerical meaningfulness after the
    stream is observed, not for online calibration.
    \paragraph{Predict-then-update calibration:}
    At each time $t$, we first emit a forecast for $w'_t$ using $\rho_{t-1}$ and record the pre-update metrics. After the
    outcome $y'_t$ is observed, we update only the Bayesian head. The supervised feedback term is
    \begin{equation}
    \widehat{R}^{\mathrm{sup}}_{T,t}(\rho)
    :=
    \ell\!\big(h_{\rho}(w'_t),y'_t\big),
    \label{eq:target-sup-loss}
    \end{equation}
    and the online calibration objective is
    \begin{equation}
    \begin{aligned}
    \mathcal{L}_{\mathrm{cal}}(t;\rho)
    &=
    \widehat{\mathrm{PB}}_{\gamma}(t;\rho)
    +\widehat{R}^{\mathrm{sup}}_{T,t}(\rho).
    \end{aligned}
    \label{eq:cal-objective}
    \end{equation}

    \section{Performance Assessment}
    
    \subsection{Experiments Setup}
    
    \paragraph{Datasets Employed:}
    To investigate OOD generalization under complementary shift patterns, we employ three representative benchmarks: the Electricity Transformer Dataset (ETTh) \citep{ETTh} for hourly load forecasting; the U.S. Outpatient Influenza-like Illness Surveillance Network (ILI) \citep{ILI} for epidemiological tracking; and the Worldwide Weather Stations (WEATHER-5K) dataset \citep{han2024weather5k} for global climatological modeling. These datasets collectively span diverse temporal resolutions and shift types. Refer to Appendix~\ref{sec:dataset_detail} for complete usage details.
    
    \paragraph{Metrics:}
    We report MAE and MSE in the main text as the primary point-forecasting metrics for comparison with online adaptation
    baselines. Since OMPB uses a Bayesian head, Appendix~\ref{sec:probabilistic_results} further evaluates probabilistic
    forecasting quality using NLL, CRPS, 80\%/95\% interval coverage, interval width, and ECE, as summarized.
    
    \paragraph{Baselines:}
    We compare OMPB against three recent state-of-the-art online adaptation methods: SOLID \citep{chen2024solid}, which
    uses context retrieval and lightweight fine-tuning; OneNet \citep{zhang2023onenet}, which uses an online ensemble to
    track non-stationary drift; and PROCEED \citep{zhao2025proceed}, which proactively adapts through learned drift
    representations. These baselines represent leading approaches for time-series forecasting under distribution shift. All
    baseline comparisons and main experimental results use a one-step feedback delay. Additional experiments with longer
    feedback delays are reported in Appendix~\ref{sec:feedback_delay_analysis}.
    
    \paragraph{Backbone:}
    To demonstrate the model-agnostic property of OMPB, we evaluate it with three representative forecasting backbones, a
    convolutional forecaster, an attention-based forecaster, and a large language model (LLM)-based forecaster, instantiated by
    TCN~\citep{TCN}, Autoformer~\citep{autoformer}, and GPT4TS~\citep{GPT4TS}. We do not use an in-distribution (ID) validation set for early stopping 
    because we found that it can be suboptimal under distribution shift, and instead select the training epoch 
    via an epoch sensitivity study. Additional details of backbone training are  provided in the Appendix~\ref{sec:backbone_details}.
    
    \paragraph{Hyperparameter Selection.}
    All online calibration hyperparameters are fixed before target evaluation and are not tuned on target labels. The
    disagreement scale \(\tau_d\) is selected automatically from the source replay buffer using
    Equation~\eqref{eq:tau-auto} with \(q=0.5\). We use \(K=5\) posterior samples, the bounded disagreement normalization in
    Equation~\eqref{eq:d_reg_code}, and variance-proxy factor \(1.0\). Appendix~\ref{sec:sensitivity_analysis} reports
    one-at-a-time sensitivity analyses for \(\tau_d\), \(K\), disagreement normalization, runtime, and variance-proxy
    misspecification.
    \begin{table*}[t]
    \centering
    
    \caption{ILI forecasting results (MAE | MSE) under In-Distribution (ID) and Out-of-Distribution (OOD) settings. Lower score values indicate better performance, and the best score is in \textbf{bold}.}
    \label{tab:ili_results}
    \setlength{\tabcolsep}{7pt}
    \renewcommand{\arraystretch}{0.4}
    \begin{adjustbox}{max width=\textwidth}
    \begin{tabular}{cccccccc}
    \toprule
    \multirow{2}{*}{Backbone} &
    \multirow{2}{*}{Method} &
    \multicolumn{2}{c}{H=24} &
    \multicolumn{2}{c}{H=48} &
    \multicolumn{2}{c}{H=72} \\
    \cmidrule(lr){3-4}\cmidrule(lr){5-6}\cmidrule(lr){7-8}
    &
    & ID (Pre-COVID) & OOD (COVID)
    & ID (Pre-COVID) & OOD (COVID)
    & ID (Pre-COVID) & OOD (COVID) \\
    \midrule
    \multirow{5}{*}{TCN}
    & Original
    & 1.6734 | 5.8091 & 4.5372 | 38.1840
    & 1.2409 | 3.2224 & 2.6306 | 11.5707
    & 1.9092 | 8.1936 & 4.0072 | 32.0801 \\
    & SOLID
    & 1.2047 | 3.5700 & 1.7447 | 7.4170
    & 1.5927 | 5.3366 & 3.1461 | 32.3844
    & 1.5807 | 5.4009 & 2.2503 | 13.2387 \\
    & OneNet
    & 1.6995 | 6.9002 & 6.7105 | 81.2448
    & 1.8328 | 7.6513 & 7.3912 | 115.1614
    & 1.7588 | 6.7818 & 7.9340 | 143.7589 \\
    & PROCEED
    & 1.4870 | 4.6585 & 1.8515 | 10.9656
    & 1.2260 | 3.1388 & 1.7905 | 9.2732
    & 1.3271 | 3.7264 & 1.9005 | 10.6453 \\
    
    & OMPB
    & \textbf{0.7102} | \textbf{1.1751} & \textbf{1.2672} | \textbf{6.3903}
    & \textbf{0.5740} | \textbf{0.9482} & \textbf{0.7109} | \textbf{1.8331}
    & \textbf{0.5960} | \textbf{1.2583} & \textbf{0.7113} | \textbf{2.1360} \\
    \midrule
    \multirow{5}{*}{Autoformer}
    & Original
    & 1.2800 | 3.3924 & 2.7942 | 15.7648
    & 2.0135 | 8.1263 & 2.6485 | 15.4891
    & 2.0193 | 8.5962 & 2.6667 | 15.5638 \\
    & SOLID
    & 1.2591 | 3.5508 & 4.0567 | 64.7122
    & 1.2642 | 3.5153 & 2.3409 | 19.6094
    & 1.4241 | 4.4793 & 2.2399 | 14.5178 \\
    & OneNet
    & 1.5966 | 5.5776 & 2.4633 | 19.4883
    & 1.4029 | 4.6296 & 2.1933 | 16.2578
    & 1.4493 | 4.5256 & 2.2957 | 18.3962 \\
    & PROCEED
    & 1.7888 | 7.1304 & 1.9178 | 11.4278
    & 1.4210 | 4.1995 & 1.5856 | 7.5495
    & 1.3469 | 3.8370 & 2.2889 | 17.9593 \\
    
    & OMPB
    & \textbf{0.9916} | \textbf{2.3576} & \textbf{1.3498} | \textbf{3.9055}
    & \textbf{1.1288} | \textbf{3.1272} & \textbf{1.2460} | \textbf{4.2611}
    & \textbf{1.2825} | \textbf{3.9065} & \textbf{1.4256} | \textbf{4.8619} \\
    \midrule
    \multirow{5}{*}{GPT4TS}
    & Original
    & 0.9667 | 2.3908 & 1.3245 | 4.7387
    & 0.8972 | 1.9836 & 1.3013 | 4.2179
    & 1.0006 | 2.1929 & 1.4140 | 5.0265 \\
    & SOLID
    & 1.1803 | 3.5820 & 1.4930 | 6.8183
    & 1.1177 | 3.2670 & 1.8180 | 13.2240
    & 1.1957 | 3.6045 & 1.3675 | 6.7276 \\
    & OneNet
    & 1.2891 | 4.0408 & 1.9931 | 14.5083
    & 1.1923 | 3.5878 & 1.9970 | 14.2298
    & 1.3108 | 4.0760 & 2.2679 | 17.8664 \\
    & PROCEED
    & 1.6876 | 7.4868 & 2.2322 | 16.6780
    & 1.1823 | 3.3912 & 1.8764 | 10.6957
    & 1.2604 | 3.8947 & 2.1143 | 13.9342 \\
    
    & OMPB
    & \textbf{0.8543} | \textbf{1.9178} & \textbf{1.0906} | \textbf{3.3870}
    & \textbf{0.6248} | \textbf{1.0503} & \textbf{0.9263} | \textbf{2.6351}
    & \textbf{0.6909} | \textbf{1.3602} & \textbf{1.0130} | \textbf{2.9065} \\
    \bottomrule
    \end{tabular}
    \end{adjustbox}
    \end{table*}

    \begin{figure*}
        \centering
        \includegraphics[width=0.8\linewidth]{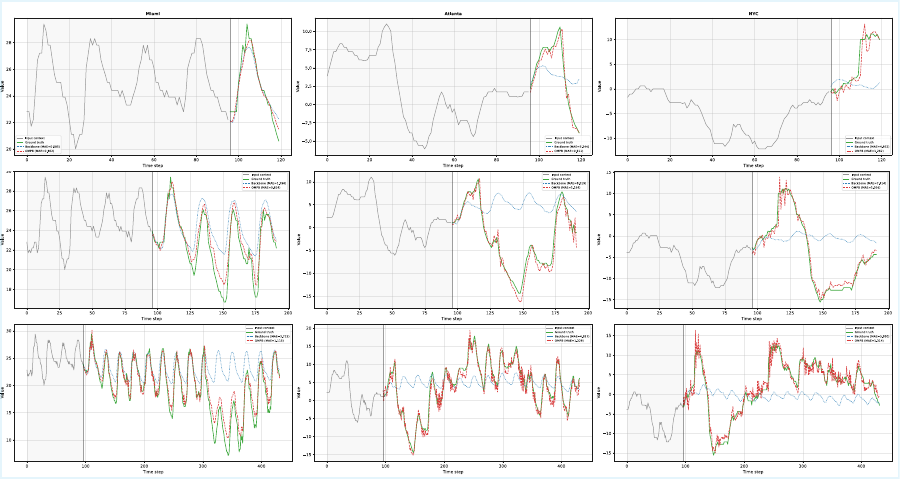}
        \caption{Visualization of WEATHER-5K forecasting with the GPT4TS backbone. Rows show horizons $H{=}24$, $H{=}96$, and $H{=}336$,
    and columns show Miami as ID, Atlanta as near OOD, and New York City (NYC) as far
    OOD. The gray curve shows the input window of length 96, the green curve shows the ground truth, the
    blue dotted curve shows the backbone prediction, and the red dashed curve shows the OMPB-calibrated prediction.}
        \label{fig:WEATHER_forecast}
    \end{figure*}
    
    \subsection{Experiments}
    
    \paragraph{A. Covariate Shift Experiment Result:}
    We use ETTh to study covariate shift because it contains hourly readings from two electricity transformers at two
    stations in different regions of the same province in China~\citep{ETTh}. This setting induces covariate shift through
    location and device-specific differences while preserving the forecasting task, analogous to sensor replacement or
    recalibration. We train on ETTh1 and evaluate on the ETTh1 test split and on ETTh2 as the ID and
    OOD test sets. Additional shift visualizations and quantification are provided in Appendix~\ref{sec:ETTh_shift}.
    
    We set the input sequence length to 96 for all methods on ETTh and evaluate horizons $H{=}24,96,336$ using MAE and MSE.
    Table~\ref{tab:etth_results} shows that covariate shift from ETTh1 to ETTh2 can severely degrade unadapted forecasters.
    For example, MAE of TCN increases from $0.4588$ to $2.5145$ at $H{=}24$ and remains large at longer horizons. While
    baselines reduce OOD error, they can sacrifice ID accuracy or provide inconsistent gains.
    In contrast, OMPB achieves the best results across all horizons and backbones, improving both ID and
    OOD metrics. The OOD MAE reductions are especially large for TCN and remain substantial
    for Autoformer. For GPT4TS, competing baselines often fail to improve and can degrade performance, whereas OMPB
    consistently improves GPT4TS across all horizons and both distributions, supporting the model-agnostic design of OMPB.
    
    \begin{table*}[t]
    \centering
    \caption{OMPB ablation results (MAE | MSE) with $H{=}24$. Lower values are better, and the best results are shown in \textbf{bold}.}
    \label{tab:ablation_result}
    
    \setlength{\tabcolsep}{6pt}
    \renewcommand{\arraystretch}{0.4}
    \begin{adjustbox}{max width=\textwidth}
    \begin{tabular}{ccccccccc}
    \toprule
    Ablation & Backbone &
    ETTh1 & ETTh2 &
    Pre-COVID & COVID &
    Miami & Atlanta & NYC \\
    \midrule
    \multirow{3}{*}{W/O gate}
    & TCN
    & 0.3853 | 0.3101
    & 0.5681 | 0.6651
    & \textbf{0.5918} | \textbf{0.8810}
    & \textbf{0.9886} | \textbf{4.4182}
    & 0.2721 | 0.2131
    & 0.3006 | 0.2359
    & 0.3499 | 0.3595 \\
    & Autoformer
    & 0.3937 | 0.3281
    & 0.7675 | 0.9474
    & 0.9944 | \textbf{2.0439}
    & 1.5084 | 5.3023
    & 0.3302 | 0.2784
    & \textbf{0.3236} | \textbf{0.2584}
    & \textbf{0.4775} | \textbf{0.6097} \\
    & GPT4TS
    & 0.3813 | 0.3263
    & \textbf{0.5214} | \textbf{0.6264}
    & \textbf{0.8488} | 2.0729
    & \textbf{0.9719} | \textbf{2.8031}
    & 0.2809 | 0.2312
    & 0.3072 | 0.2500
    & \textbf{0.3622} | 0.4022 \\
    \midrule
    \multirow{3}{*}{W/O PB}
    & TCN
    & 0.4085 | 0.3609
    & 0.6294 | 0.7812
    & 0.7406 | 1.3025
    & 1.2846 | 6.4667
    & 0.3239 | 0.2596
    & 0.4363 | 0.4468
    & 0.5555 | 0.7176 \\
    & Autoformer
    & 0.4359 | 0.4152
    & 0.6905 | 0.8798
    & 1.0997 | 2.4123
    & 1.6026 | 5.7349
    & 0.4189 | 0.3870
    & 0.3880 | 0.3396
    & 0.7216 | 1.1491 \\
    & GPT4TS
    & 0.3678 | 0.3113
    & 0.5295 | 0.6677
    & 0.9683 | 2.9854
    & 1.1392 | 3.7494
    & 0.3390 | 0.2909
    & 0.4541 | 0.4961
    & 0.5647 | 0.7872 \\
    \midrule
    \multirow{3}{*}{W/O Online}
    & TCN
    & 0.4599 | 0.4581
    & 2.5400 | 9.2468
    & 1.6623 | 5.8682
    & 4.4602 | 37.8893
    & 0.3331 | 0.2684
    & 0.5691 | 0.7435
    & 0.8092 | 1.4660 \\
    & Autoformer
    & 0.4911 | 0.5107
    & 0.9671 | 1.5330
    & 1.5403 | 4.5176
    & 2.9293 | 14.8753
    & 0.4037 | 0.3572
    & 0.6252 | 0.8325
    & 0.8183 | 1.4850 \\
    & GPT4TS
    & 0.3736 | 0.3155
    & 0.5674 | 0.7278
    & 1.1711 | 3.1064
    & 1.8671 | 7.9420
    & 0.3463 | 0.2984
    & 0.4936 | 0.5792
    & 0.6318 | 0.9540 \\
    \midrule
    \multirow{3}{*}{Original}
    & TCN
    & \textbf{0.3635} | \textbf{0.2724}
    & \textbf{0.5292} | \textbf{0.5659}
    & 0.7102 | 1.1751
    & 1.2672 | 6.3903
    & \textbf{0.2650} | \textbf{0.1965}
    & \textbf{0.2952} | \textbf{0.2274}
    & \textbf{0.3487} | \textbf{0.3388} \\
    & Autoformer
    & \textbf{0.3814} | \textbf{0.3018}
    & \textbf{0.6785} | \textbf{0.8535}
    & \textbf{0.9916} | 2.3576
    & \textbf{1.3498} | \textbf{3.9055}
    & \textbf{0.3284} | \textbf{0.2672}
    & 0.3298 | 0.2693
    & 0.5240 | 0.7036 \\
    & GPT4TS
    & \textbf{0.3156} | \textbf{0.2165}
    & 0.5242 | 0.6616
    & 0.8543 | \textbf{1.9178}
    & 1.0906 | 3.3870
    & \textbf{0.2759} | \textbf{0.2143}
    & \textbf{0.3025} | \textbf{0.2435}
    & 0.3639 | \textbf{0.3793} \\
    \bottomrule
    \end{tabular}
    \end{adjustbox}
    \end{table*}
    
    \paragraph{B. Concept Shift Experiment Result:}
    We evaluate concept shift on ILI by splitting the series into two regimes. We use 2009--2020 as the pre-COVID training
    period and 2021--2026 as the COVID test period. This setting reflects concept shift because the conditional mapping
    from past signals to future influenza-like illness levels changes due to the pandemic and related interventions, leading
    to a mismatch in $\mathbb{P}(Y\mid X)$. We train on the pre-COVID period and evaluate on both the pre-COVID test split and the
    COVID period. Additional shift visualizations and quantification are provided in Appendix~\ref{sec:ILI_shift}.
    
    Due to the limited length of the ILI time series, we set the input sequence length to 24 and evaluate horizons
    $H{=}24$ to $H{=}72$ using MAE and MSE. Table~\ref{tab:ili_results} shows that the shift
    from the pre-COVID period to the COVID period substantially increases forecasting difficulty, especially for the
    convolutional and attention-based backbones. Naive deployment exhibits large OOD errors, and OneNet
    becomes unstable under this regime change, with errors exploding across all horizons for TCN. SOLID and PROCEED mitigate
    the shift in some settings but provide inconsistent improvements and can be sensitive to the horizon and backbone. In
    contrast, OMPB achieves the best results across all backbones and horizons, consistently reducing both
    ID and OOD errors. Notably, under concept shift OMPB is also the only method that
    consistently improves GPT4TS across all horizons and both distributions, whereas competing baselines often degrade or
    fail to improve its performance. These results indicate that constraining adaptation to a gated residual Bayesian head
    and regularizing updates yields stable calibration under severe concept shift.

    \paragraph{C. Near vs. Far OOD Generalization:}

    The WEATHER-5K dataset contains 5{,}672 weather stations worldwide~\citep{han2024weather5k}. We evaluate covariate shift
    by training on Miami and testing on Atlanta as a nearby OOD location and New York City (NYC) as a farther OOD
    location. Although all three cities are on the U.S. East Coast, their climate regimes differ, inducing distribution
    shifts in variables such as temperature. We use GPT4TS as the backbone and visualize temperature forecasts in
    Figure~\ref{fig:WEATHER_forecast}.  Additional shift
    quantification and visualizations are provided in Appendix~\ref{sec:WEATHER_shift}.
    
    Figure~\ref{fig:WEATHER_forecast} shows that the forecasting error increases with the magnitude of covariate shift and
    that OMPB consistently reduces this error. Especially under the near OOD shift at Atlanta and the far OOD shift
    at NYC, the backbone becomes increasingly biased and overly smooth as the horizon grows, while OMPB produces forecasts
    that are consistently closer to the ground truth across $H{=}24,96,336$, with the largest gains in the far OOD setting.
    Additional numerical results for these locations and all backbones are reported in Appendix~\ref{sec:weather_result}.

    \paragraph{D. Ablation Study:}
    We perform ablations on OMPB by (i) removing the gate $s$ in Equation~\eqref{eq:gated-residual-head} (W/O gate) to assess
    the do-no-harm effect, (ii) removing the PAC-Bayes certificate $\widehat{\mathrm{PB}}_{\gamma}(t;\rho)$ from
    Equation~\eqref{eq:cal-objective} (W/O PB) to evaluate its contribution, and (iii) disabling the online calibration stage
    and using the offline-learned Bayesian head directly for prediction (W/O Online), which removes Stage II in
    Figure~\ref{fig:pipeline} to assess the necessity of deployment-time calibration.

    Table~\ref{tab:ablation_result} shows that removing the gate usually leads to modest degradation, supporting its
    do-no-harm role under mild shift, although it can be comparable under limited data (Pre-COVID and COVID) or stronger
    shift (NYC). In contrast, both W/O PB and W/O Online consistently degrade performance, indicating that the certificate
    is an important regularizer and that deployment-time online calibration is necessary under shift.
    \begin{figure}[ht]
        \begin{center}
        \includegraphics[width=\linewidth]{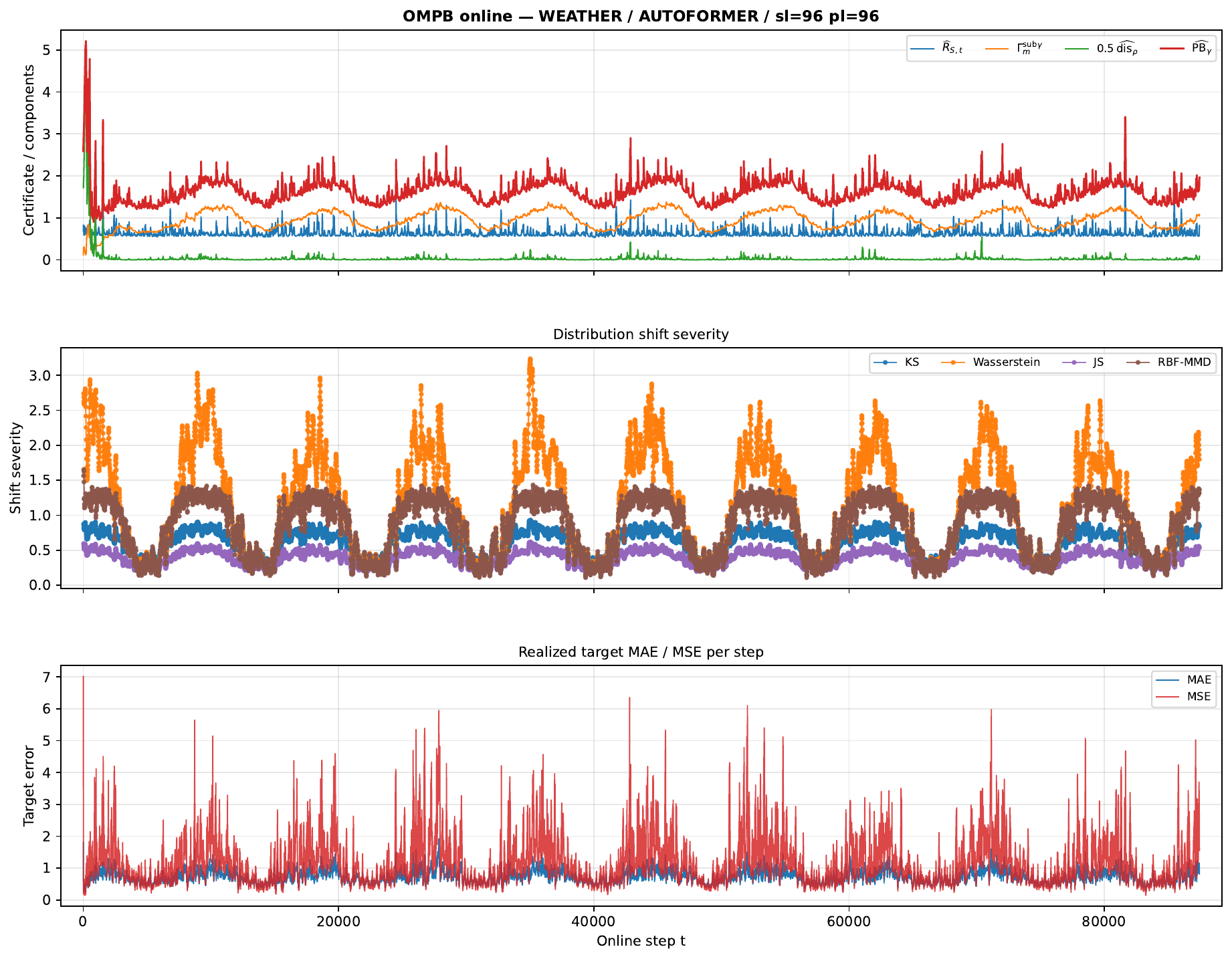}
        \caption{Deployment-time certificate diagnostics on the WEATHER-5K far-OOD stream. The figure compares the online certificate, posterior disagreement, external distribution-shift severity metrics, and realized pre-calibration target error. The certificate follows the recurring shift pattern and responds to local error spikes, supporting its use as a monitoring signal under distribution shift.}
        \label{fig:cert_shift_timeseries_main}
        \end{center}
    \end{figure}
    \section{Discussion}
    
    \paragraph{Certificate Behavior:}
    Figure~\ref{fig:cert_shift_timeseries_main} gives a deployment-time view of the WEATHER-5K far-OOD stream. The online
    certificate \(\widehat{\mathrm{PB}}_{\gamma}(t)\) combines the source-risk anchor, martingale correction, and input-side
    disagreement penalty, while the target-error curve reports the pre-calibration error observed after labels are released.
    The certificate follows the main distribution-shift pattern: it rises during periods of stronger source--target mismatch
    and recurring seasonal change, and it also responds to local error spikes. This supports its role as a monitoring signal
    rather than a direct MAE/MSE bound. Additional certificate diagnostics for ETTh and ILI are in
    Appendix~\ref{sec:timeseries_diagnostics}, and the corresponding certificate--error and disagreement--shift scatter
    analyses are in Appendix~\ref{sec:correlation_diagnostics}.
    
    \paragraph{Expressiveness of the Bayesian Head:}
    \label{sec:discussion_head_expressiveness}
    
    OMPB uses a lightweight gated residual Bayesian head instead of updating the full backbone to prioritize stable online
    calibration, lower cost, and a tractable PAC-Bayes complexity term under delayed feedback. Since delayed labels may
    reflect stale regimes, high-capacity backbone updates can overfit and destabilize adaptation; freezing the backbone
    localizes updates to the Bayesian head, while the gate preserves the backbone forecast as a safe fallback. As shown in
    Table~\ref{tab:partial_backbone_adaptation}, frozen-backbone OMPB outperforms last-layer and LoRA-style online updates
    on all three shifted streams. The partial-backbone variants add trainable parameters and degrade MAE/MSE, suggesting
    that extra online flexibility can hurt under delayed supervision, though it may still be useful when target feedback is
    denser or more reliable.
    
    \begin{table}[t]
    \centering
    
    \caption{Partial-backbone adaptation diagnostic for OMPB. ``Trainable'' reports additional backbone parameters updated online beyond the Bayesian head.}
    \begin{adjustbox}{max width=.5\textwidth}
    \label{tab:partial_backbone_adaptation}
    \renewcommand{\arraystretch}{0.4}
    \begin{tabular}{ccccc}
    \toprule
    Setting & Mode & Trainable & MAE & MSE \\
    \midrule
    \multirow{3}{*}{\begin{tabular}[c]{@{}l@{}}ETTh / TCN\\OOD ETTh2, \(H=96\)\end{tabular}}
    & Frozen     & 0        & \textbf{0.5570} & \textbf{0.5912} \\
    & Last layer & 22{,}176 & 0.7124 & 0.8834 \\
    & LoRA       & 2{,}816  & 0.7410 & 1.0474 \\
    \midrule
    \multirow{3}{*}{\begin{tabular}[c]{@{}l@{}}ILI / GPT4TS\\OOD COVID, \(H=48\)\end{tabular}}
    & Frozen     & 0         & \textbf{0.9263} & \textbf{2.6351} \\
    & Last layer & 110{,}640 & 1.4808 & 6.8875 \\
    & LoRA       & 9{,}408   & 1.6591 & 9.3536 \\
    \midrule
    \multirow{3}{*}{\begin{tabular}[c]{@{}l@{}}WEATHER-5K / Autoformer\\OOD far NYC, \(H=96\)\end{tabular}}
    & Frozen     & 0       & \textbf{0.4344} & \textbf{0.4720} \\
    & Last layer & 780     & 0.7401 & 0.9791 \\
    & LoRA       & 8{,}192 & 0.9681 & 1.5443 \\
    \bottomrule
    \end{tabular}
    \end{adjustbox}
    \end{table}
    
    
    \begin{acknowledgements} 
    The work of GM was partially supported by NSF grants ATD 2319552 and DMS 2348640.

    \end{acknowledgements}
    
    \bibliography{references}
    
    \newpage
    \onecolumn
    \title{Supplementary Material for \\ Model-Agnostic Online Certificate-Driven Calibration \\ for Time Series Forecasting Under Distribution Shift}
    \maketitle
    \thispagestyle{numbered} 

    \appendix
    
    \startcontents[appendix]
    
    \printappendixtoc

    \clearpage
    \section{Technical Appendix}
    
    \subsection{Proof of Theorem~\ref{th:PAC_Bayes}}
    \label{sec:proof_pacbayes}
    \begin{proof}
    Let $S=(z_1,\ldots,z_m)$ with $z_i=(x_i,y_i)\sim \mathcal D$ i.i.d. and assume $\ell(h,z)\in[0,1]$.
    For any $h\in\mathcal H$, define the empirical and population risks
    \begin{equation}
    \begin{aligned}
    \widehat R_S(h)
    :=
    \frac{1}{m}\sum_{i=1}^{m}\ell(h,z_i),
    \qquad
    R_{\mathcal D}(h)
    :=
    \mathbb E_{z\sim\mathcal D}\!\big[\ell(h,z)\big].
    \end{aligned}
    \label{eq:app_pb_emp_pop_risk}
    \end{equation}
    For any posterior $\rho$ on $\mathcal H$, the Gibbs risks are
    \begin{equation}
    \begin{aligned}
    \widehat R_S(h_\rho)
    :=
    \mathbb E_{h\sim\rho}\!\big[\widehat R_S(h)\big],
    \qquad
    R_{\mathcal D}(h_\rho)
    :=
    \mathbb E_{h\sim\rho}\!\big[R_{\mathcal D}(h)\big].
    \end{aligned}
    \label{eq:app_pb_gibbs_risk}
    \end{equation}
    
    Following the standard McAllester PAC-Bayes argument for bounded losses \citep{mcallester1999pac}, one starts from an
    exponential moment bound that holds for each fixed $h$:
    \begin{equation}
    \begin{aligned}
    \mathbb E_{S\sim\mathcal D^m}
    \Bigg[
    \exp\!\Big(
    2m\big(R_{\mathcal D}(h)-\widehat R_S(h)\big)^2
    \Big)
    \Bigg]
    \le
    2\sqrt{m}.
    \end{aligned}
    \label{eq:app_pb_moment_bound}
    \end{equation}
    Define the prior-averaged exponential moment
    \begin{equation}
    \begin{aligned}
    U(S)
    :=
    \mathbb E_{h\sim\pi}
    \Bigg[
    \exp\!\Big(
    2m\big(R_{\mathcal D}(h)-\widehat R_S(h)\big)^2
    \Big)
    \Bigg].
    \end{aligned}
    \label{eq:app_pb_U_def}
    \end{equation}
    Taking expectation over $S$ and using Equation~\eqref{eq:app_pb_moment_bound} yields
    \begin{equation}
    \begin{aligned}
    \mathbb E_{S\sim\mathcal D^m}\!\big[U(S)\big]
    \le
    2\sqrt{m}.
    \end{aligned}
    \label{eq:app_pb_U_expect}
    \end{equation}
    By Markov's inequality, with probability at least $1-\delta$ over $S$,
    \begin{equation}
    \begin{aligned}
    U(S)
    \le
    \frac{2\sqrt{m}}{\delta}.
    \end{aligned}
    \label{eq:app_pb_markov}
    \end{equation}
    
    Next, we apply the Donsker\textendash Varadhan change-of-measure inequality \citep{donsker1975asymptotic}, which states
    that for any measurable $\phi$ on $\mathcal H$,
    \begin{equation}
    \begin{aligned}
    \mathbb E_{h\sim\rho}\!\big[\phi(h)\big]
    \le
    \mathrm{KL}(\rho\Vert\pi)
    +
    \ln\mathbb E_{h\sim\pi}\!\big[\exp(\phi(h))\big].
    \end{aligned}
    \label{eq:app_pb_dv}
    \end{equation}
    On the event Equation~\eqref{eq:app_pb_markov}, apply Equation~\eqref{eq:app_pb_dv} with
    \begin{equation}
    \begin{aligned}
    \phi_S(h)
    :=
    2m\big(R_{\mathcal D}(h)-\widehat R_S(h)\big)^2.
    \end{aligned}
    \label{eq:app_pb_phi}
    \end{equation}
    Then, for every posterior $\rho$,
    \begin{equation}
    \begin{aligned}
    \mathbb E_{h\sim\rho}\!\Big[
    2m\big(R_{\mathcal D}(h)-\widehat R_S(h)\big)^2
    \Big]
    &\le
    \mathrm{KL}(\rho\Vert\pi)
    +
    \ln U(S)\\
    &\le
    \mathrm{KL}(\rho\Vert\pi)
    +
    \ln\frac{2\sqrt{m}}{\delta}.
    \end{aligned}
    \label{eq:app_pb_second_moment}
    \end{equation}
    Dividing by $2m$ gives
    \begin{equation}
    \begin{aligned}
    \mathbb E_{h\sim\rho}\!\Big[
    \big(R_{\mathcal D}(h)-\widehat R_S(h)\big)^2
    \Big]
    \le
    \frac{\mathrm{KL}(\rho\Vert\pi)+\ln\frac{2\sqrt{m}}{\delta}}{2m}.
    \end{aligned}
    \label{eq:app_pb_quad}
    \end{equation}
    Finally, Jensen's inequality implies
    \begin{equation}
    \begin{aligned}
    R_{\mathcal D}(h_\rho)-\widehat R_S(h_\rho)
    &=
    \mathbb E_{h\sim\rho}\!\big[R_{\mathcal D}(h)-\widehat R_S(h)\big]\\
    &\le
    \sqrt{
    \mathbb E_{h\sim\rho}\!\Big[
    \big(R_{\mathcal D}(h)-\widehat R_S(h)\big)^2
    \Big]
    }.
    \end{aligned}
    \label{eq:app_pb_jensen}
    \end{equation}
    Combining Equation~\eqref{eq:app_pb_quad} and Equation~\eqref{eq:app_pb_jensen} yields
    \begin{equation}
    \begin{aligned}
    R_{\mathcal D}(h_{\rho})
    \le
    \widehat{R}_{S}(h_{\rho})
    +
    \sqrt{
    \frac{\mathrm{KL}(\rho \Vert \pi) + \ln \frac{2\sqrt{m}}{\delta}}{2m}
    },
    \end{aligned}
    \label{eq:app_pb_final}
    \end{equation}
    which is Equation~\eqref{eq:pacbayes-iid}, and $\mathrm{KL}(\rho \Vert \pi)
    \ :=\ 
    \sum_{h \in \mathcal{H}} \rho(h)\,
    \ln \frac{\rho(h)}{\pi(h)}.$
    \end{proof}

    \subsection{Proof of Theorem~\ref{th:PAC_DA}}
    \label{sec:proof_pac_da}
    \begin{proof}
    We combine a PAC-Bayes control of the source estimation error \citep{mcallester1999pac} with the Gibbs domain-adaptation
    decomposition of \citet{germain2013pac,germain2016pac}.

    For any domain $\mathcal{D}$ and any pair $h,h'\in\mathcal{H}$, define the pointwise error indicators
    \begin{equation}
    \begin{aligned}
    A(w,y)
    :=
    \mathbf{1}\{h(w)\neq y\},
    \qquad
    A'(w,y)
    :=
    \mathbf{1}\{h'(w)\neq y\},
    \end{aligned}
    \label{eq:app_err_indicators}
    \end{equation}
    and the pointwise disagreement indicator
    \begin{equation}
    \begin{aligned}
    D(w)
    :=
    \mathbf{1}\{h(w)\neq h'(w)\}.
    \end{aligned}
    \label{eq:app_dis_ind}
    \end{equation}
    For binary classification, the following identity holds for every $(w,y)$
    \begin{equation}
    \begin{aligned}
    A(w,y)+A'(w,y)
    =
    D(w)
    +
    2\,A(w,y)\,A'(w,y),
    \end{aligned}
    \label{eq:app_key_identity}
    \end{equation}
    because if $h(w)\neq h'(w)$ then exactly one of $\{h,h'\}$ is wrong, while if $h(w)=h'(w)$ then either both are correct
    or both are wrong.
    Taking expectation with respect to $(w,y)\sim\mathcal{D}$ gives
    \begin{equation}
    \begin{aligned}
    R_{\mathcal{D}}(h)+R_{\mathcal{D}}(h')
    =
    R_{\mathcal{D}}(h,h')
    +
    2\,e_{\mathcal{D}}(h,h'),
    \end{aligned}
    \label{eq:app_risk_dis_joint}
    \end{equation}

    where $R_{\mathcal{D}}(h,h')
    :=
    \mathbb{E}_{w\sim \mathcal{D}_{\mathcal{X}}}\!\big[d(h,h';w)\big]$ with $d(h,h';w)=\mathbf{1}\{h(w)\neq h'(w)\}$,
    and $e_{\mathcal{D}}(h,h')
    :=
    \mathbb{E}_{(w,y)\sim \mathcal{D}}
    \Big[
    \mathbf{1}\{h(w)\neq y\}\,\mathbf{1}\{h'(w)\neq y\}
    \Big]$.
    Now take expectation over $h,h'\sim\rho$ and use symmetry
    $\mathbb{E}_{h,h'\sim\rho}[R_{\mathcal{D}}(h)]
    =
    \mathbb{E}_{h,h'\sim\rho}[R_{\mathcal{D}}(h')]
    =
    R_{\mathcal{D}}(h_\rho)$ to obtain the Gibbs decomposition
    \begin{equation}
    \begin{aligned}
    R_{\mathcal{D}}(h_\rho)
    =
    \frac{1}{2}\,
    \mathbb{E}_{h,h'\sim\rho}\,R_{\mathcal{D}}(h,h')
    +
    \mathbb{E}_{h,h'\sim\rho}\,e_{\mathcal{D}}(h,h').
    \end{aligned}
    \label{eq:app_gibbs_decomp}
    \end{equation}
    Equation~\eqref{eq:app_gibbs_decomp} is the origin of the factor $\tfrac{1}{2}$ in front of the disagreement term.
    
    Apply Equation~\eqref{eq:app_gibbs_decomp} to $\mathcal{T}$ and $\mathcal{S}$ and subtract:
    \begin{equation}
    \begin{aligned}
    R_{\mathcal{T}}(h_\rho)
    -
    R_{\mathcal{S}}(h_\rho)
    &=
    \frac{1}{2}\,
    \mathbb{E}_{h,h'\sim\rho}
    \Big[
    R_{\mathcal{T}}(h,h')-R_{\mathcal{S}}(h,h')
    \Big]\\
    &\quad
    +
    \mathbb{E}_{h,h'\sim\rho}
    \Big[
    e_{\mathcal{T}}(h,h')-e_{\mathcal{S}}(h,h')
    \Big].
    \end{aligned}
    \label{eq:app_diff_ST}
    \end{equation}
    Taking absolute values and using the definitions of $R_{\mathcal{D}}(h,h')$ and let $\lambda_{\rho}
    :=
    \left|
    \mathbb{E}_{h,h' \sim \rho}\, e_{\mathcal{T}}(h,h')
    -
    \mathbb{E}_{h,h' \sim \rho}\, e_{\mathcal{S}}(h,h')
    \right|$, which yields
    \begin{equation}
    \begin{aligned}
    R_{\mathcal{T}}(h_\rho)
    \le
    R_{\mathcal{S}}(h_\rho)
    +
    \frac{1}{2}\,\mathrm{dis}_{\rho}(\mathcal{S},\mathcal{T})
    +
    \lambda_{\rho}.
    \end{aligned}
    \label{eq:app_da_decomp_bound}
    \end{equation}
    Finally, substitute the source PAC-Bayes bound Equation~\eqref{eq:app_pb_final} into Equation~\eqref{eq:app_da_decomp_bound} to get
    Equation~\eqref{eq:iid-pbda}.
    \end{proof}
    
    \subsection{Proof of Theorem~\ref{th:bound-martingale}}
    \label{sec:proof_bound_martingale}
    \begin{proof}
    We reuse the domain-adaptation decomposition from \citet{germain2013pac,germain2016pac} that was used in the proof of
    Theorem~\ref{th:PAC_DA}. 
    Thus, it suffices to upper-bound $R_{\mathcal{S}}(h_\rho)$ by $\widehat{R}_{S}(h_\rho)$ plus a dependence-aware
    complexity term.
    
    Define the (predictable) source risk and the empirical source risk for a fixed $h$ by
    \begin{equation}
    \begin{aligned}
    R_{\mathcal{S}}(h)
    :=
    \frac{1}{m}\sum_{t=1}^{m}\mathbb{E}\!\big[\ell(h,Z_t)\mid\mathcal{F}_{t-1}\big],
    \qquad
    \widehat{R}_{S}(h)
    :=
    \frac{1}{m}\sum_{t=1}^{m}\ell(h,Z_t).
    \end{aligned}
    \label{eq:app_source_risks_def}
    \end{equation}
    Let
    \begin{equation}
    \begin{aligned}
    X_t(h)
    :=
    \mathbb{E}\!\big[\ell(h,Z_t)\mid\mathcal{F}_{t-1}\big]
    -
    \ell(h,Z_t),
    \qquad
    M_m(h)
    :=
    \sum_{t=1}^{m}X_t(h).
    \end{aligned}
    \label{eq:app_mds_def_freed}
    \end{equation}
    Then $(M_t(h),\mathcal{F}_t)_{t=1}^m$ is a martingale and
    \begin{equation}
    \begin{aligned}
    R_{\mathcal{S}}(h)-\widehat{R}_{S}(h)
    =
    \frac{1}{m}\,M_m(h).
    \end{aligned}
    \label{eq:app_gap_mart_relation}
    \end{equation}
    Since $\ell(h,Z_t)\in[0,1]$, we have $|X_t(h)|\le B$ with $B=1$. A Freedman-type exponential supermartingale
    construction \citep{freedman1975tail,bercu2008} combined with the PAC-Bayes change-of-measure inequality
    \citep{donsker1975asymptotic,catoni2007} yields the following standard PAC-Bayes-Freedman bound for martingales
    \citep{seldin2012pac,seldin2011pac}: with probability at least $1-\delta$ over the draw of $S$, for all posteriors
    $\rho$ simultaneously,
    \begin{equation}
    \begin{aligned}
    R_{\mathcal{S}}(h_\rho)
    \le
    \widehat{R}_{S}(h_\rho)
    +
    \sqrt{
    \frac{
    2\,\mathbb{E}_{h\sim\rho}\!\big[V_m(h)\big]\,
    \pbterm
    }{m^2}
    }
    +
    \frac{B}{3m}\,\pbterm,
    \end{aligned}
    \label{eq:app_pac_bayes_freed_source}
    \end{equation}
    where $V_m(h)
    :=
    \sum_{t=1}^{m}\mathrm{Var}\!\big(\ell(h,Z_t)\mid\mathcal{F}_{t-1}\big)$ and $\pbterm=\mathrm{KL}(\rho\Vert\pi)+\ln\tfrac{C_0\sqrt{m}}{\delta}$.
    The additional $\ln(C_0\sqrt{m})$ term is the usual discretization or union-bound overhead appearing in time-uniform or
    self-normalized PAC-Bayes martingale bounds \citep{seldin2012pac}.
    
    Finally, combining Equation~\eqref{eq:app_da_decomp_bound} with Equation~\eqref{eq:app_pac_bayes_freed_source} gives
    Equation~\eqref{eq:martingale-da-bound} with $\Gamma_m^{\mathrm{Freed}}(\rho,\pi,\delta)$ as in Equation~\eqref{eq:gamma-freed}.
    \end{proof}
    
    \subsection{Proof of Theorem~\ref{th:bound-subgamma}}
    \label{sec:proof_bound_subgamma}

    \begin{proof}
    As in the proof of Theorem~\ref{th:bound-martingale}, we start from the PAC-Bayes domain-adaptation decomposition of
    \citet{germain2013pac,germain2016pac}. 
    It remains to upper-bound $R_{\mathcal{S}}(h_\rho)$ by $\widehat{R}_{S}(h_\rho)$ plus a martingale PAC-Bayes term under
    the sub-gamma condition:
    \begin{equation}
    \begin{aligned}
    \mathbb E\!\left[\exp\{\lambda X_t(h)\}\mid\mathcal F_{t-1}\right]
    &\le 
    \exp\!\left(\frac{\lambda^2 v_t(h)}{2(1-c(h)\lambda)}\right),
    \qquad \\
    &\forall \lambda\in\big(0,1/c(h)\big).
    \end{aligned}
    \label{eq:subgamma-condition2}
    \end{equation}

    Define the predictable source risk and its empirical estimate by
    \begin{equation}
    \begin{aligned}
    R_{\mathcal{S}}(h)
    :=
    \frac{1}{m}\sum_{t=1}^{m}\mathbb{E}\!\big[\ell(h,Z_t)\mid\mathcal{F}_{t-1}\big],
    \qquad
    \widehat{R}_{S}(h)
    :=
    \frac{1}{m}\sum_{t=1}^{m}\ell(h,Z_t).
    \end{aligned}
    \label{eq:app_source_risks_def_subgamma}
    \end{equation}
    Then
    \begin{equation}
    \begin{aligned}
    \widehat{R}_{S}(h)-R_{\mathcal{S}}(h)
    =
    \frac{1}{m}\sum_{t=1}^{m}X_t(h).
    \end{aligned}
    \label{eq:app_gap_subgamma}
    \end{equation}
    To bound $R_{\mathcal{S}}(h)-\widehat{R}_{S}(h)$, set
    \begin{equation}
    \begin{aligned}
    \widetilde{X}_t(h)
    :=
    -\,
    X_t(h)
    =
    \mathbb{E}\!\big[\ell(h,Z_t)\mid\mathcal{F}_{t-1}\big]
    -
    \ell(h,Z_t).
    \end{aligned}
    \label{eq:app_xt_tilde_def}
    \end{equation}
    Under the standard conditional sub-gamma assumption, $\widetilde{X}_t(h)$ is also conditionally sub-gamma with the same
    $(v_t(h),c(h))$ parameters \citep{bercu2008}. Iterating Equation~\eqref{eq:subgamma-condition2} yields an exponential supermartingale
    argument for each fixed $h$ \citep{bercu2008,howard2021time}: for any $\lambda\in(0,1/c(h))$,
    \begin{equation}
    \begin{aligned}
    \mathbb{E}\!\left[
    \exp\!\left(
    \lambda\sum_{t=1}^{m}\widetilde{X}_t(h)
    -
    \frac{\lambda^2}{2(1-c(h)\lambda)}\,V_m(h)
    \right)
    \right]
    \le 1.
    \end{aligned}
    \label{eq:app_subgamma_exp_mg}
    \end{equation}
    
    Combining Equation~\eqref{eq:app_subgamma_exp_mg} with the PAC-Bayes change-of-measure inequality
    \citep{donsker1975asymptotic,catoni2007} and the PAC-Bayes martingale machinery of \citet{seldin2012pac} gives the
    following uniform high-probability bound: with probability at least $1-\delta$ over $S$, for all posteriors $\rho$ and
    all $\lambda\in(0,1/\bar c)$,
    \begin{equation}
    \begin{aligned}
    \mathbb{E}_{h\sim\rho}\!\left[\sum_{t=1}^{m}\widetilde{X}_t(h)\right]
    &\le
    \frac{\lambda}{2\big(1-\bar c\,\lambda\big)}\,
    \mathbb{E}_{h\sim\rho}\!\big[V_m(h)\big]
    +
    \frac{1}{\lambda}\,
    \pbterm.
    \end{aligned}
    \label{eq:app_pac_bayes_subgamma_lambda}
    \end{equation}
    Optimizing the right-hand side over $\lambda\in(0,1/\bar c)$ yields the standard sub-gamma (Bernstein-type) form
    \citep{bercu2008,seldin2012pac}:
    \begin{equation}
    \begin{aligned}
    \mathbb{E}_{h\sim\rho}\!\left[\sum_{t=1}^{m}\widetilde{X}_t(h)\right]
    &\le
    \sqrt{
    2\,\mathbb{E}_{h\sim\rho}\!\big[V_m(h)\big]\,
    \pbterm
    }
    +
    \bar c\,\pbterm.
    \end{aligned}
    \label{eq:app_pac_bayes_subgamma_opt}
    \end{equation}
    Dividing by $m$ and using $R_{\mathcal{S}}(h)-\widehat{R}_{S}(h)=\frac{1}{m}\sum_{t=1}^m \widetilde{X}_t(h)$ gives
    \begin{equation}
    \begin{aligned}
    R_{\mathcal{S}}(h_\rho)
    &\le
    \widehat{R}_{S}(h_\rho)
    +
    \sqrt{
    \frac{
    2\,\mathbb{E}_{h\sim\rho}\!\big[V_m(h)\big]\,
    \pbterm
    }{m^2}
    }
    +
    \frac{\bar c}{m}\,\pbterm.
    \end{aligned}
    \label{eq:app_source_subgamma_bound}
    \end{equation}
    Finally, combining Equation~\eqref{eq:app_da_decomp_bound} with Equation~\eqref{eq:app_source_subgamma_bound} yields
    Equation~\eqref{eq:subgamma-da-bound} with $\Gamma_m^{\mathrm{sub}\,\gamma}(\rho,\pi,\delta)$ as in Equation~\eqref{eq:gamma-subgamma}.
    \end{proof}

    \subsection{Regression Instantiation and Surrogate Interpretation}\label{sec:reg_surrogate}
    
    \begin{theorem}[Regression-Proxy Martingale PAC-Bayes Domain Adaptation]\label{th:reg-proxy-da}
    Let $(Z_t,\mathcal F_t)_{t=1}^{m}$ be a filtered source process with $Z_t=(w_t,y_t)$, let $\mathcal T$ be a target
    distribution over $(w,y)$, let $\pi$ be a prior over $\mathcal H$, and let $\rho$ be any posterior over $\mathcal H$.
    For thresholds $\tau_y,\tau_d>0$, define the bounded forecasting proxy loss and bounded prediction-discrepancy
    surrogate as
    \begin{equation}
    \ell_{\tau_y}(h;(w,y))
    :=
    \min\!\left\{1,\frac{\|y-h(w)\|_2^2}{\tau_y^2}\right\},
    \qquad
    \widetilde d_{\tau_d}(h,h';w)
    :=
    \min\!\left\{1,\frac{\|h(w)-h'(w)\|_2^2}{\tau_d^2}\right\}.
    \label{eq:reg_proxy_loss_dis}
    \end{equation}
    Let $R_{\mathcal T,\tau_y}(h_\rho)$ be the target risk under $\ell_{\tau_y}$, and let
    $\widehat R_{S,\tau_y}(h_\rho)$ be its empirical source estimate. Define the posterior-weighted proxy discrepancy
    \begin{equation}
    \widetilde{\mathrm{dis}}_{\rho}(\mathcal S,\mathcal T)
    :=
    \left|
    \mathbb E_{h,h'\sim\rho}\widetilde R_{\mathcal S}(h,h')
    -
    \mathbb E_{h,h'\sim\rho}\widetilde R_{\mathcal T}(h,h')
    \right|,
    \label{eq:reg_proxy_dis}
    \end{equation}
    where
    \begin{equation}
    \widetilde R_{\mathcal S}(h,h')
    :=
    \frac{1}{m}\sum_{t=1}^{m}
    \mathbb E\!\left[\widetilde d_{\tau_d}(h,h';w_t)\mid\mathcal F_{t-1}\right],
    \qquad
    \widetilde R_{\mathcal T}(h,h')
    :=
    \mathbb E_{w\sim\mathcal T_{\mathcal X}}\!\left[\widetilde d_{\tau_d}(h,h';w)\right].
    \label{eq:reg_proxy_source_target_dis}
    \end{equation}
    Define the residual proxy mismatch
    \begin{equation}
    \begin{aligned}
    \widetilde\lambda_{\rho,\tau_y,\tau_d}
    :=
    \Bigg|
    &\left(
    R_{\mathcal T,\tau_y}(h_\rho)
    -\frac{1}{2}\mathbb E_{h,h'\sim\rho}\widetilde R_{\mathcal T}(h,h')
    \right)
    -
    \left(
    R_{\mathcal S,\tau_y}(h_\rho)
    -\frac{1}{2}\mathbb E_{h,h'\sim\rho}\widetilde R_{\mathcal S}(h,h')
    \right)
    \Bigg|.
    \end{aligned}
    \label{eq:reg_proxy_lambda}
    \end{equation}
    Then, with probability at least $1-\delta$ over the source stream, for every posterior $\rho$,
    \begin{equation}
    \begin{aligned}
    R_{\mathcal T,\tau_y}(h_\rho)
    \le
    \widehat R_{S,\tau_y}(h_\rho)
    +
    \Gamma_m^{\mathrm{Freed}}(\rho,\pi,\delta)
    +
    \frac{1}{2}\widetilde{\mathrm{dis}}_{\rho}(\mathcal S,\mathcal T)
    +
    \widetilde\lambda_{\rho,\tau_y,\tau_d}.
    \end{aligned}
    \label{eq:reg_proxy_bound}
    \end{equation}
    Here, $\Gamma_m^{\mathrm{Freed}}(\rho,\pi,\delta)$ is the Freedman-style PAC-Bayes correction applied to the bounded
    proxy loss $\ell_{\tau_y}$.
    \end{theorem}
    
    \begin{proof}
    Since $\ell_{\tau_y}\in[0,1]$, the centered source increments are bounded martingale differences. Applying the
    PAC-Bayes--Freedman source-risk bound to $\ell_{\tau_y}$ gives, uniformly over posteriors,
    \begin{equation}
    R_{\mathcal S,\tau_y}(h_\rho)
    \le
    \widehat R_{S,\tau_y}(h_\rho)
    +
    \Gamma_m^{\mathrm{Freed}}(\rho,\pi,\delta).
    \label{eq:reg_proxy_source_bound}
    \end{equation}
    For $\mathcal D\in\{\mathcal S,\mathcal T\}$, define
    \[
    B_{\mathcal D}(\rho)
    :=
    \frac{1}{2}\mathbb E_{h,h'\sim\rho}\widetilde R_{\mathcal D}(h,h'),
    \qquad
    \Delta_{\mathcal D}(\rho)
    :=
    R_{\mathcal D,\tau_y}(h_\rho)-B_{\mathcal D}(\rho).
    \]
    Then
    \begin{equation}
    \begin{aligned}
    R_{\mathcal T,\tau_y}(h_\rho)
    &=
    R_{\mathcal S,\tau_y}(h_\rho)
    +
    \big(B_{\mathcal T}(\rho)-B_{\mathcal S}(\rho)\big)
    +
    \big(\Delta_{\mathcal T}(\rho)-\Delta_{\mathcal S}(\rho)\big) \\
    &\le
    R_{\mathcal S,\tau_y}(h_\rho)
    +
    \frac{1}{2}\widetilde{\mathrm{dis}}_{\rho}(\mathcal S,\mathcal T)
    +
    \widetilde\lambda_{\rho,\tau_y,\tau_d}.
    \end{aligned}
    \label{eq:reg_proxy_decomp}
    \end{equation}
    Combining Equations~\eqref{eq:reg_proxy_source_bound} and~\eqref{eq:reg_proxy_decomp} proves the result.
    \end{proof}
    
    Theorem~\ref{th:reg-proxy-da} certifies the bounded proxy risk $R_{\mathcal T,\tau_y}(h_\rho)$, not raw MAE or MSE.
    Since $\widetilde\lambda_{\rho,\tau_y,\tau_d}$ depends on target labels and is unavailable before prediction, OMPB
    optimizes the input-identifiable part of the bound:
    \[
    \widehat R_{S,\tau_y}(h_\rho)
    +
    \Gamma_m^{\mathrm{Freed}}(\rho,\pi,\delta)
    +
    \frac{1}{2}\widetilde{\mathrm{dis}}_{\rho}(\mathcal S,\mathcal T).
    \]
    This is the population analogue of the online certificate in Equation~\eqref{eq:pb-online}; the realized MAE and MSE
    improvements reported in the experiments are empirical validation of the calibration procedure rather than a direct
    consequence of a raw-regression-loss upper bound.
    
    \begin{corollary}[Tail-Risk Certificate for Forecasting]\label{cor:tail_risk}
    Fix a threshold $\tau>0$ and define the bounded truncated-square loss
    \begin{equation}
    \ell_\tau\!\big(h;(w,y)\big)
    :=
    \min\!\left\{1,\ \frac{\|y-h(w)\|_2^2}{\tau^2}\right\}
    \in[0,1].
    \label{eq:trunc_sq_loss}
    \end{equation}
    Then, for any predictor $h$ and any target distribution $\mathcal{T}$ over $(w,y)$,
    \begin{equation}
    \mathbf{1}\!\left\{\|y-h(w)\|_2\ge \tau\right\}
    \ \le\
    \ell_\tau\!\big(h;(w,y)\big),
    \label{eq:indicator_upper_trunc}
    \end{equation}
    which implies
    \begin{equation}
    \mathbb{P}_{(w,y)\sim\mathcal{T}}\!\left(\|y-h(w)\|_2\ge \tau\right)
    \ \le\
    R_{\mathcal{T},\tau}(h),
    \qquad
    R_{\mathcal{T},\tau}(h):=\mathbb{E}_{(w,y)\sim\mathcal{T}}\!\left[\ell_\tau\!\big(h;(w,y)\big)\right].
    \label{eq:tail_prob_upper}
    \end{equation}
    Consequently, any finite-sample upper bound or online certificate on the bounded proxy risk
    $R_{\mathcal{T},\tau}(h_\rho)$ yields an upper bound on the catastrophic-error probability
    $\mathbb{P}_{\mathcal{T}}(\|y-h_\rho(w)\|_2\ge \tau)$.
    
    Moreover, since $\ell_\tau\in[0,1]$, the centered increments
    \begin{equation}
    \widetilde X_t(h)
    :=
    \mathbb{E}\!\left[\ell_\tau(h,Z_t)\mid\mathcal{F}_{t-1}\right]-\ell_\tau(h,Z_t)
    \in[-1,1]
    \label{eq:centered_increments_bounded}
    \end{equation}
    admit a Bernstein/Hoeffding-type moment generating function control, providing a concrete instantiation of the
    conditional sub-gamma assumption used in Theorem~\ref{th:bound-subgamma}.
    \end{corollary}

    \paragraph{What the certificate controls:}
    The online certificate $\widehat{\mathrm{PB}}_{\gamma}(t)$ is derived for the bounded proxy quantities that appear in
    Equation~\eqref{eq:pb-online}, rather than for the raw regression metrics used for reporting. In particular, the mismatch
    component depends on the bounded disagreement score $d(\cdot,\cdot;w)\in[0,1]$ in Equation~\eqref{eq:d_reg_code}. As a
    result, we interpret $\widehat{\mathrm{PB}}_{\gamma}(t)$ as a PAC-Bayes inspired certificate for monitoring and as a
    surrogate regularizer for head calibration, while forecasting accuracy is evaluated separately using MAE and MSE.
    Corollary~\ref{cor:tail_risk} further shows that when the proxy loss is chosen as the bounded truncated-square loss
    $\ell_\tau$ in Equation~\eqref{eq:trunc_sq_loss}, any certificate on the corresponding proxy risk also upper-bounds the
    probability of threshold exceedance as in Equation~\eqref{eq:tail_prob_upper}.
    
    \paragraph{Why the proxy is informative for regression:}
    A lightweight tail interpretation is given by Corollary~\ref{cor:tail_risk}, which connects a certificate on the bounded
    proxy risk $R_{\mathcal{T},\tau}$ to an upper bound on the probability of large forecast errors in
    Equation~\eqref{eq:tail_prob_upper}. We use this relation as motivation and do not claim a tight equivalence between the
    proxy certificate and MAE/MSE.
    
    The proxy is less informative when $\tau$ is so small that most errors saturate at one, or so large that moderate errors
    barely affect the loss. Clipping also loses distinctions among errors far beyond the threshold. In addition, if the
    posterior collapses or is badly miscalibrated, posterior disagreement may no longer reflect epistemic uncertainty; and
    if the target label mechanism changes strongly while inducing little input-side disagreement, the residual proxy
    mismatch in Theorem~\ref{th:reg-proxy-da} can dominate but cannot be identified from unlabeled target windows.
    
    \paragraph{Why the sub-gamma assumption is reasonable for forecasting losses:}
    The sub-gamma condition in Theorem~\ref{th:bound-subgamma} is a conditional tail assumption on the centered loss
    increments
    \begin{equation}
    X_t(h)\ :=\ \ell(h,Z_t)\ -\ \mathbb{E}\!\big[\ell(h,Z_t)\mid\mathcal{F}_{t-1}\big].
    \label{eq:centered_increment_app}
    \end{equation}
    In forecasting, common losses such as squared error, absolute error, and negative log-likelihood are unbounded. Under
    many practical noise models, residuals become approximately light-tailed after normalization, which makes such centered
    loss increments amenable to sub-gamma type control. For instance, if the conditional residual $r_t=y_t-\widehat{y}_t$ is
    approximately conditionally sub-Gaussian, then functions such as $r_t^2$ are
    sub-exponential and their centered versions admit Bernstein-type moment generating function bounds, which can be
    expressed in the sub-gamma form used in Equation~\eqref{eq:subgamma-condition2}. This yields a variance-adaptive correction term
    that is robust to heteroskedasticity and occasional spikes.
    
    In implementation, we instantiate the predictable variance proxy and scale parameters using source-dependent
    statistics. Concretely, we maintain a predictable scale estimate and use it to construct $v_t(h)$ and $c(h)$ that enter $\Gamma_m^{\mathrm{sub}\,\gamma}$, which makes
    the certificate sensitive to the observed loss scale and more stable than clipping-based alternatives. This design does
    not require explicit estimation of mixing coefficients and is compatible with streaming deployment where the target
    distribution may drift over time.

    \subsection{Dataset Detail}
    \label{sec:dataset_detail}

    \begin{table}[htbp]
    \centering
    \caption{Summary of the datasets used in our experiments. Train and test sizes report the number of sliding-window samples.}
    \label{tab:datasets}
    \renewcommand{\arraystretch}{0.8}
    \begin{tabular}{cccccc}
    \toprule
    Dataset & Features & Frequency & Input Len & Train Samples & Test Samples \\
    \midrule
    ETTh    & 7  & Hourly & 96 & 13{,}745 & 3{,}293 \\
    ILI     & 11 & Weekly & 24 & 443      & 108      \\
    WEATHER-5K & 6  & Hourly & 96 & 78{,}764 & 8{,}646$\times 2$ \\
    \bottomrule
    \end{tabular}
    \end{table}

    \paragraph{ETTh~\citep{ETTh}:}
    We use 80\% of ETTh1 as the training set and the remaining 20\% of ETTh1 as the ID test set. To evaluate
    covariate shift, we sample the same number of windows from ETTh2 as the OOD test set.
    
    The dataset is available at: \url{https://github.com/zhouhaoyi/ETDataset}.
    
    \paragraph{ILI~\citep{ILI}:}
    Due to the limited amount of available ILI data, we use 80\% of the pre-COVID period for training and the remaining
    20\% of the pre-COVID period as the ID test set. We then sample the same number of windows from the COVID
    period as the OOD test set. This split preserves sufficient evaluation coverage for medium-term
    forecasting up to $H{=}72$.
    
    The dataset is available at: \url{https://gis.cdc.gov/grasp/fluview/fluportaldashboard.html}
    
    \paragraph{WEATHER-5K~\citep{han2024weather5k}:}
    Since WEATHER-5K is larger, we use 90\% of the Miami subset for training and the remaining 10\% as the
    ID test set. For OOD evaluation, we sample the same number of windows from Atlanta
    and from New York City. This choice provides ample training data for the backbones and avoids the severe degradation
    that can arise from insufficient source training when transferring to a new location.
    
    The dataset is available at: \url{https://github.com/taohan10200/WEATHER-5K}. We selected the following stations:
    \begin{itemize}[leftmargin=*]
        \item Miami, Florida, USA (72202012839.csv)
        \item Atlanta, Georgia, USA (72219013874.csv)
        \item New York City, New York, USA (72503014732.csv)
    \end{itemize}
    
    Table~\ref{tab:datasets} summarizes each dataset, including the number of features, recording frequency, input sequence
    length, and the numbers of training and testing samples.
    
    \subsection{Distribution Shift Verification}
    \subsubsection{ETTh}
    \label{sec:ETTh_shift}
    Table~\ref{tab:covariate_shift_ett} quantifies the covariate shift between ETTh1 and ETTh2 and confirms that the two
    stations exhibit a substantial distribution mismatch. The mean Kolmogorov--Smirnov (KS) statistic of 0.751 indicates large
    differences in marginal feature distributions, while the mean Wasserstein distance of 15.20 reflects a pronounced shift
    in feature values. The mean Jensen--Shannon (JS) divergence of 0.455 and the radial basis function maximum mean discrepancy (MMD)
    of 0.867 further indicate strong global distributional separation between ETTh1 and ETTh2, supporting the use of this
    split as a covariate-shift benchmark. Figure~\ref{fig:ETTh_shift} has the density visualization for each feature and PCA between ETTh1 and ETTh2.
    \begin{figure}
        \centering
        \includegraphics[width=\linewidth]{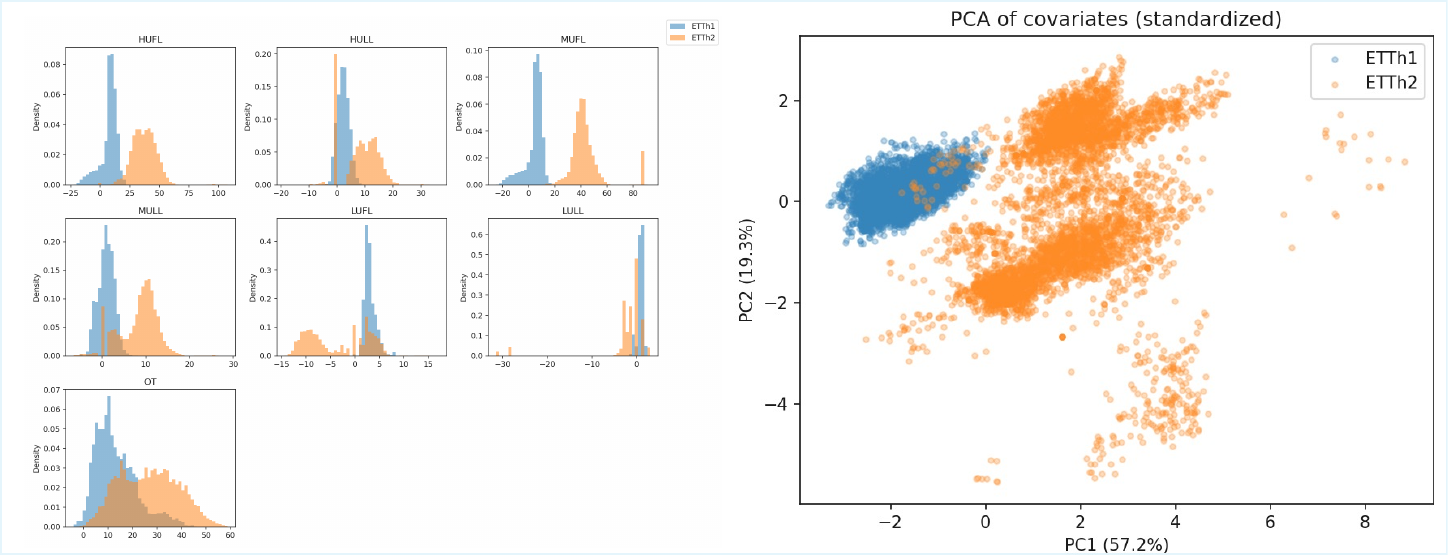}
        \caption{The left plot visualizes the marginal density distributions of each feature for ETTh1 and ETTh2, while the right plot shows a PCA projection of ETTh1 and ETTh2, illustrating the distributional separation between the two datasets.}
        \label{fig:ETTh_shift}
    \end{figure}
    
    \begin{table}[ht]
    \centering
    \caption{Pairwise covariate-shift metrics (ETTh dataset).}
    \renewcommand{\arraystretch}{0.8}
    \label{tab:covariate_shift_ett}
    \begin{tabular}{ccccc}
    \toprule
    Pair & KS (mean) & Wasserstein (mean) & JS (mean) & MMD (RBF) \\
    \midrule
    ETTh1 vs ETTh2 & 0.751 & 15.20 & 0.455 & 0.867 \\
    \bottomrule
    \end{tabular}
    \end{table}

    \subsubsection{ILI}
    \label{sec:ILI_shift}
    Table~\ref{tab:covariate_shift_ili} quantifies the distribution shift between the Pre-COVID training split and the COVID
    test split in the ILI dataset and shows a substantial mismatch. The mean Kolmogorov--Smirnov (KS) statistic of 0.644
    indicates large differences in marginal feature distributions, while the mean Wasserstein distance of 180{,}774
    suggests a pronounced shift in feature scale. The mean Jensen–Shannon (JS) divergence of 0.352 and the radial basis
    function maximum mean discrepancy (MMD) of 1.315 further confirm strong global separation between the two periods, supporting
    the use of this split as a severe distribution-shift benchmark. Figure~\ref{fig:ILI_shift} has the density visualization for each feature and PCA between Pre-COVID and COVID.
    \begin{figure}
        \centering
        \includegraphics[width=\linewidth]{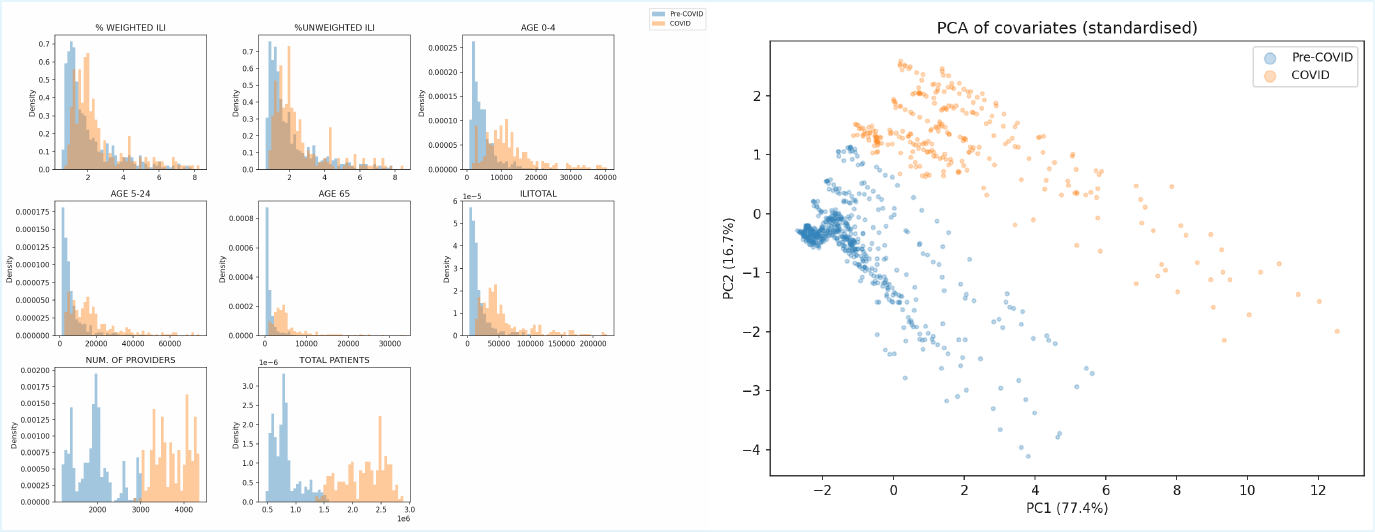}
        \caption{The left plot visualizes the marginal density distributions of each feature for pre-COVID and COVID in ILI dataset, while the right plot shows a PCA projection of pre-COVID and COVID, illustrating the distributional separation between the two datasets.}
        \label{fig:ILI_shift}
    \end{figure}
    \begin{table}[ht]
    \centering
    \caption{Pairwise covariate-shift metrics (ILI dataset).}
    \renewcommand{\arraystretch}{0.8}
    \label{tab:covariate_shift_ili}
    \begin{tabular}{ccccc}
    \toprule
    Pair & KS (mean) & Wasserstein (mean) & JS (mean) & MMD (RBF) \\
    \midrule
    Pre-COVID vs COVID & 0.644 & 180{,}774 & 0.352 & 1.315 \\
    \bottomrule
    \end{tabular}
    \end{table}

    \subsubsection{WEATHER-5K}
    \label{sec:WEATHER_shift}
    
    \begin{figure}[h]
        \centering
        \includegraphics[width=\linewidth]{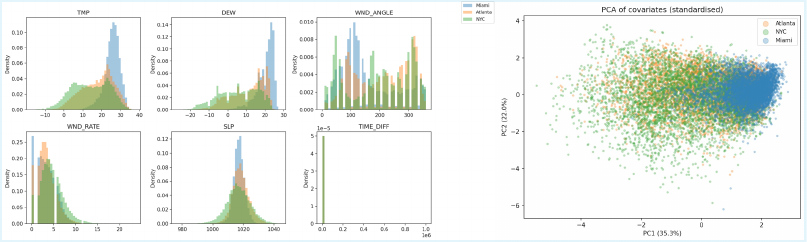}
        \caption{The left plot visualizes the marginal density distributions of each feature for the WEATHER-5K dataset, comparing the
    training location in Miami with the out-of-distribution test locations in Atlanta and New York City (NYC).
    The right plot shows a PCA projection of the three locations.}
        \label{fig:WEATHER_shift}
    \end{figure}
    \begin{table}[h]
    \centering
    \caption{Pairwise covariate-shift metrics (WEATHER-5K dataset).}
    \renewcommand{\arraystretch}{0.8}
    \label{tab:covariate_shift_weather_locations}
    \begin{tabular}{ccccc}
    \toprule
    Pair & KS (mean) & Wasserstein (mean) & JS (mean) & MMD (RBF) \\
    \midrule
    Miami vs Atlanta    & 0.363 & 845.15  & 0.080 & 0.088 \\
    Miami vs NYC        & 0.422 & 845.93  & 0.139 & 0.082 \\
    Atlanta vs NYC      & 0.255 & 4.86    & 0.116 & 0.008 \\
    \bottomrule
    \end{tabular}
    \end{table}
    
    Table~\ref{tab:covariate_shift_weather_locations} quantifies covariate shift between weather stations and shows clear distributional differences across locations. Miami versus Atlanta exhibits a moderate mismatch, with mean Kolmogorov--Smirnov (KS) statistic 0.363, mean Wasserstein distance 845.15, mean Jensen–Shannon (JS) divergence 0.080, and radial basis function  maximum mean discrepancy (MMD) 0.088. Miami versus New York City (NYC) shows the strongest shift across all metrics,
    with mean KS statistic 0.422, mean JS divergence 0.139, and radial basis function MMD 0.082, together with a large mean Wasserstein distance of 845.93. In contrast, Atlanta versus NYC is comparatively closer in this setup, with mean KS statistic 0.255, mean Wasserstein distance
    4.86, mean JS divergence 0.116, and radial basis function MMD 0.008. These results support that the selected stations induce measurable distribution shifts, which we further visualize in Figure~\ref{fig:WEATHER_shift} using feature density plots and a PCA projection.
    
    \subsection{Backbone Training Detail}
    \label{sec:backbone_details}
    
    \begin{table}[htbp]
    \centering
    
    \caption{Training hyperparameters for the three backbone models.}
    \begin{adjustbox}{max width=\textwidth}
    \renewcommand{\arraystretch}{0.8}
    \label{tab:backbone-params}
    \begin{tabular}{ccccc}
    
    \toprule
    & Parameter & TCN & Autoformer & GPT4TS \\
    \midrule
    \multirow{3}{*}{\rotatebox{90}{\textit{Training}}}
    & Optimizer          & Adam     & Adam      & Adam \\
    & Learning rate      & $10^{-3}$ & $10^{-3}$ & $10^{-4}$ \\
    & Loss function      & MSE      & MSE       & MSE \\
    \midrule
    \multirow{15}{*}{\rotatebox{90}{\textit{Architecture}}}
    & Hidden widths        & [32, 32, 32] & --   & -- \\
    & Kernel size          & 3            & --   & -- \\
    & $d_{\text{model}}$   & --           & 64   & 768 \\
    & $d_{\text{ff}}$      & --           & 256  & -- \\
    & Attention heads      & --           & 4    & -- \\
    & Encoder layers       & --           & 2    & -- \\
    & Decoder layers       & --           & 1    & -- \\
    & Moving avg.\ window  & --           & 25   & -- \\
    & Auto-corr.\ factor   & --           & 3.0  & -- \\
    & GPT-2 layers         & --           & --   & 2 \\
    & Patch size           & --           & --   & 16 \\
    & Stride               & --           & --   & 8 \\
    & Dropout              & 0.1          & 0.1  & -- \\
    & Pretrained weights   & --           & --   & GPT-2 \\
    & Fine-tuned params    & All          & All  & LN + Pos.\ Emb. \\
    \bottomrule
    \end{tabular}
    \end{adjustbox}
    \end{table}
    \paragraph{TCN:}
    We implement the Temporal Convolutional Network following \citet{TCN}, using residual blocks with causal, dilated
    one-dimensional convolutions. The backbone has three layers with 32 hidden channels, kernel size 3, and dilation factors
    1, 2, and 4, yielding a receptive field that covers the full input window. Each block uses weight normalization, ReLU
    activations, and dropout 0.1. A linear projection maps the final representation to the prediction horizon. We train all
    parameters from scratch with Adam, learning rate $10^{-3}$, and MSE.
    
    \paragraph{Autoformer:}
    We use the series-decomposition Transformer of \citet{autoformer}. The model has two encoder layers and one decoder
    layer, each combining auto-correlation attention with progressive trend and seasonal decomposition. We set the model
    dimension to 64 with four attention heads and feed-forward dimension 256. The moving-average kernel size for
    decomposition is 25, the auto-correlation top-$k$ factor is 3.0, and dropout is 0.1. The model is trained
    from scratch with Adam, learning rate $10^{-3}$, and MSE.
    
    \paragraph{GPT4TS:}
    We follow \citet{GPT4TS} and adapt a pretrained GPT-2 model for time-series forecasting via patching and lightweight
    fine-tuning. The input is tokenized into patches of size 16 with stride 8, projected to the GPT-2 embedding dimension
    768, and processed by the two GPT-2 layers. To preserve pretrained representations, we fine-tune only the
    LayerNorm and positional embedding parameters and keep the remaining GPT-2 weights frozen. A linear output layer maps
    the patch representations to the forecast horizon. We train with Adam, learning rate $10^{-4}$, and MSE.
    
    The above details are also summarized in Table~\ref{tab:backbone-params}.
    
    \paragraph{Epoch Sensitivity:}
    Since we evaluate models on distributions that differ from the training distribution, early stopping based only on an
    ID validation set may be suboptimal. To select a robust training epoch for each backbone, we run an epoch sensitivity
    study on each dataset with three training budgets, 5, 20, and 100 epochs. The results are reported in
    Tables~\ref{tab:epoch_sensitivity_etth}, \ref{tab:epoch_sensitivity_ili}, and~\ref{tab:epoch_sensitivity_weather}.
    Table~\ref{tab:epoch_selection_summary} summarizes how often each epoch attains the best result across horizons and
    splits, and marks the final selected epoch for each backbone in bold. We select the epoch with the highest count for
    each backbone and dataset, breaking ties by the overall total count. For example, for GPT4TS on ETTh, epochs 5 and 20
    each achieve six best results, so we choose epoch 20 because it has the larger total count.
    
    Table~\ref{tab:val_vs_sens_all} compares early stopping based on an ID validation set with the epoch selected by our
    sensitivity study. For early stopping, we hold out 10\% of the training data as a validation set, use patience 10, and
    set the maximum epoch to 200. For each backbone, the upper row reports the validation-selected epoch and its
    performance, while the lower row reports the sensitivity-selected epoch.

    \begin{table*}[h]
    \centering
    
    \caption{Epoch sensitivity on ETTh (MAE | MSE) for In-Distribution and Out-of-Distribution evaluation. Lower score values indicate better performance, and for each model and prediction horizon the best result is shown in \textbf{bold}.}
    \begin{adjustbox}{max width=\textwidth}
    \label{tab:epoch_sensitivity_etth}
    \renewcommand{\arraystretch}{0.7}
    \setlength{\tabcolsep}{6pt}
    \begin{tabular}{cccc|cc|cccc}
    \toprule
    \multirow{2}{*}{H} & \multirow{2}{*}{Epoch}
    & \multicolumn{2}{c|}{TCN}
    & \multicolumn{2}{c|}{Autoformer}
    & \multicolumn{2}{c}{GPT4TS} \\
    \cmidrule(lr){3-4}\cmidrule(lr){5-6}\cmidrule(lr){7-8}
    & & ID (ETTh1) & OOD (ETTh2)
      & ID (ETTh1) & OOD (ETTh2)
      & ID (ETTh1) & OOD (ETTh2) \\
    \midrule
    \multirow{3}{*}{$H{=}24$}
    & 5
    & 0.4761 | 0.4548 & 2.5387 | \textbf{8.7398}
    & 0.5147 | 0.5605 & \textbf{1.1207} | 1.9961
    & 0.3796 | 0.3275 & 0.5573 | 0.7407 \\
    & 20
    & \textbf{0.4588} | \textbf{0.4521} & \textbf{2.5145} | 8.9095
    & \textbf{0.5116} | \textbf{0.5345} & 1.1220 | \textbf{1.9614}
    & \textbf{0.3794} | \textbf{0.3217} & \textbf{0.5485} | \textbf{0.7119} \\
    & 100
    & 0.4638 | 0.4730 & 2.6279 | 10.5223
    & 0.6190 | 0.7968 & 1.3421 | 3.0684
    & 0.3851 | 0.3317 & 0.5576 | 0.7234 \\
    \midrule
    \multirow{3}{*}{$H{=}96$}
    & 5
    & 0.5489 | 0.5973 & 2.8486 | 11.1977
    & 0.5325 | 0.5651 & 1.3298 | 2.6778
    & 0.4479 | 0.4225 & \textbf{0.6700} | \textbf{1.0212} \\
    & 20
    & \textbf{0.5330} | \textbf{0.5715} & \textbf{2.5635} | \textbf{8.9076}
    & 0.5719 | 0.6328 & 1.3791 | 2.9361
    & 0.4473 | 0.4167 & 0.6738 | 1.0241 \\
    & 100
    & 0.5428 | 0.5823 & 2.7946 | 10.8517
    & \textbf{0.5296} | \textbf{0.5603} &\textbf{1.1485} | \textbf{2.0664}
    & 0.4612 | 0.4375 & 0.6926 | 1.0762 \\
    \midrule
    \multirow{3}{*}{$H{=}336$}
    & 5
    & \textbf{0.6169} | \textbf{0.7059} & 3.2290 | 14.2250
    & \textbf{0.5147} | \textbf{0.5605} & \textbf{1.1207} | \textbf{1.9961}
    & \textbf{0.5234} | \textbf{0.5447} & \textbf{0.8312} | \textbf{1.4613} \\
    & 20
    & 0.6193 | 0.7304 & \textbf{3.1380} | \textbf{13.7387}
    & 0.6704 | 0.8652 & 1.2568 | 2.5576
    & 0.5320 | 0.5588 & 0.8547 | 1.5203 \\
    & 100
    & 0.6210 | 0.7248 & 3.1463 | 14.5143
    & 0.6890 | 0.9008 & 1.2900 | 2.6016
    & 0.5453 | 0.5785 & 0.8593 | 1.5134 \\
    \bottomrule
    \end{tabular}
    \end{adjustbox}
    \end{table*}

    \begin{table*}[h]
    \centering
    
    \caption{Epoch sensitivity on ILI (MAE | MSE) for In-Distribution and Out-of-Distribution evaluation. Lower score values indicate better performance, and for each model and prediction horizon the best result is shown in \textbf{bold}.}
    \begin{adjustbox}{max width=\textwidth}
    \label{tab:epoch_sensitivity_ili}
    \renewcommand{\arraystretch}{0.5}
    \setlength{\tabcolsep}{6pt}
    \begin{tabular}{cccc|cc|cc}
    \toprule
    \multirow{2}{*}{H} & \multirow{2}{*}{Epoch}
    & \multicolumn{2}{c|}{TCN}
    & \multicolumn{2}{c|}{Autoformer}
    & \multicolumn{2}{c}{GPT4TS} \\
    \cmidrule(lr){3-4}\cmidrule(lr){5-6}\cmidrule(lr){7-8}
    & & ID (Pre-COVID) & OOD (COVID)
      & ID (Pre-COVID) & OOD (COVID)
      & ID (Pre-COVID) & OOD (COVID) \\
    \midrule
    \multirow{3}{*}{$H{=}24$}
    & 5
    & 1.8138 | 6.2710 & \textbf{3.1916} | \textbf{17.7921}
    & 1.7129 | 5.4064 & 3.4545 | 20.1964
    & 1.2870 | 3.9209 & 1.2507 | 4.1161 \\
    & 20
    & \textbf{1.4627} | \textbf{3.9348} & 3.7786 | 22.6670
    & 1.7020 | 5.4825 & 3.6535 | 27.2704
    & 1.0951 | 2.9652 & 1.3875 | 5.0460 \\
    & 100
    & 1.6734 | 5.8091 & 4.5372 | 38.1840
    & \textbf{1.2800} | \textbf{3.3924} & \textbf{2.7942} | \textbf{15.7648}
    & \textbf{0.9667} | \textbf{2.3908} & \textbf{1.3245} | \textbf{4.7387} \\
    \midrule
    \multirow{3}{*}{$H{=}48$}
    & 5
    & 2.2747 | 9.2657 & 3.5927 | 22.9180
    & \textbf{1.7946} | \textbf{6.3452} & \textbf{2.3338} | \textbf{12.8700}
    & 1.1305 | 3.1574 & \textbf{1.2570} | \textbf{3.9344} \\
    & 20
    & 1.7178 | 5.3849 & 2.8466 | 14.8282
    & 1.9458 | 7.5849 & 2.4399 | 14.2704
    & 1.0353 | 2.5726 & 1.2934 | 4.2084 \\
    & 100
    & \textbf{1.2409} | \textbf{3.2224} & \textbf{2.6306} | \textbf{11.5707}
    & 2.0135 | 8.1263 & 2.6485 | 15.4891
    & \textbf{0.8972} | \textbf{1.9836} & 1.3013 | 4.2179 \\
    \midrule
    \multirow{3}{*}{$H{=}72$}
    & 5
    & 2.0803 | \textbf{7.4882} & \textbf{3.3792} | \textbf{21.4741}
    & 2.1000 | \textbf{8.5060} & 2.9516 | 17.5582
    & 1.1591 | 2.7503 & 1.3926 | 4.8496 \\
    & 20
    & 2.3788 | 9.1391 & 4.3711 | 31.0456
    & 2.2536 | 9.7026 & 3.6342 | 27.2995
    & 1.0691 | 2.4867 & \textbf{1.3798} | \textbf{4.8574} \\
    & 100
    & \textbf{1.9092} | 8.1936 & 4.0072 | 32.0801
    & \textbf{2.0193} | 8.5962 & \textbf{2.6667} | \textbf{15.5638}
    & \textbf{1.0006} | \textbf{2.1929} & 1.4140 | 5.0265 \\
    \bottomrule
    \end{tabular}
    \end{adjustbox}
    \end{table*}
    
    \begin{table*}[h]
    \centering
    
    \caption{Epoch sensitivity on WEATHER-5K (MAE | MSE) across In-Distribution and Out-of-Distribution locations. Lower score values indicate better performance, and for each model and prediction horizon the best result is shown in \textbf{bold}.}
    \begin{adjustbox}{max width=\textwidth}
    \label{tab:epoch_sensitivity_weather}
    \renewcommand{\arraystretch}{0.5}
    \setlength{\tabcolsep}{3pt}
    \resizebox{\textwidth}{!}{%
    \begin{tabular}{ccccc| ccc| ccc}
    \toprule
    \multirow{2}{*}{H} & \multirow{2}{*}{Epoch}
    & \multicolumn{3}{c|}{TCN}
    & \multicolumn{3}{c|}{Autoformer}
    & \multicolumn{3}{c}{GPT4TS} \\
    \cmidrule(lr){3-5}\cmidrule(lr){6-8}\cmidrule(lr){9-11}
    & & ID & Near & Far & ID & Near & Far & ID & Near & Far \\
    \midrule
    \multirow{3}{*}{$H{=}24$}
    & 5
    & 0.3353 | 0.2700 & 0.3675 | 0.3212 & \textbf{0.8012} | \textbf{1.4215}
    & \textbf{0.4309} | \textbf{0.3837} & \textbf{0.4657} | \textbf{0.4550} & 0.8779 | 1.7081
    & 0.3603 | 0.3301 & 0.3938 | 0.3900 & 0.6730 | 1.1136 \\
    & 20
    & 0.3289 | \textbf{0.2619} & 0.3593 | 0.3106 & 0.8515 | 1.7190
    & 0.4929 | 0.4929 & 0.5391 | 0.5936 & 0.9340 | 1.8922
    & 0.3511 | 0.3168 & 0.3839 | 0.3768 & 0.6498 | 1.0380 \\
    & 100
    & \textbf{0.3269} | 0.2657 & \textbf{0.3535} | \textbf{0.3071} & 0.9110 | 2.0531
    & 0.4516 | 0.4274 & 0.4930 | 0.5116 & \textbf{0.8713} | \textbf{1.6349}
    & \textbf{0.3488} | \textbf{0.3101} & \textbf{0.3813} | \textbf{0.3688} & \textbf{0.6424} | \textbf{1.0096} \\
    \midrule
    \multirow{3}{*}{$H{=}96$}
    & 5
    & \textbf{0.4522} | \textbf{0.4378} & 0.4968 | 0.5285 & \textbf{1.1416} | \textbf{2.8791}
    & 0.5181 | 0.5293 & 0.5776 | 0.6596 & 0.9777 | 2.0166
    & 0.4819 | 0.5215 & 0.5430 | 0.6534 & \textbf{0.9179} | \textbf{1.9213} \\
    & 20
    & 0.4523 | 0.4405 & \textbf{0.4870} | \textbf{0.5139} & 1.2380 | 3.5844
    & \textbf{0.5027} | \textbf{0.5164} & \textbf{0.5598} | \textbf{0.6366} & 0.9791 | 2.0369
    & 0.4755 | 0.5064 & 0.5369 | 0.6411 & 0.9240 | 1.9608 \\
    & 100
    & 0.4659 | 0.4670 & 0.4895 | 0.5157 & 1.2955 | 4.0131
    & 0.5292 | 0.5700 & 0.5878 | 0.7006 & \textbf{0.9743} | \textbf{2.0091}
    & \textbf{0.4732} | \textbf{0.4995} & \textbf{0.5316} | \textbf{0.6296} & 0.9390 | 2.0504 \\
    \midrule
    \multirow{3}{*}{$H{=}336$}
    & 5
    & 0.5093 | 0.5240 & 0.5586 | 0.6339 & \textbf{1.3262} | \textbf{3.9330}
    & \textbf{0.5527} | \textbf{0.6137} & 0.6213 | 0.7597 & 1.0620 | 2.4130
    & 0.5374 | 0.6171 & 0.6086 | 0.7837 & \textbf{0.9994} | \textbf{2.2002} \\
    & 20
    & \textbf{0.5066} | \textbf{0.5215} & 0.5477 | \textbf{0.6215} & 1.4085 | 4.5465
    & 0.5532 | 0.6172 & \textbf{0.6186} | \textbf{0.7603} & 1.0474 | 2.3230
    & 0.5333 | 0.6088 & 0.6059 | 0.7824 & 1.0272 | 2.3732 \\
    & 100
    & 0.5083 | 0.5344 & \textbf{0.5464} | 0.6223 & 1.3647 | 4.1765
    & 0.5567 | 0.6227 & 0.6244 | 0.7716 & \textbf{1.0467} | \textbf{2.3192}
    & \textbf{0.5299} | \textbf{0.5968} & \textbf{0.6004} | \textbf{0.7658} & 1.0442 | 2.4666 \\
    \bottomrule
    \end{tabular}%
    }
    \end{adjustbox}
    \end{table*}

    \begin{table}[h]
    
    \centering

    \caption{Epoch selection summary. Count columns report the number of best results across horizons and splits. Bold entries mark the selected epoch for each backbone after tie-breaking by total count.}
    \begin{adjustbox}{max width=\textwidth}
    \label{tab:epoch_selection_summary}
    
    \renewcommand{\arraystretch}{0.5}
    
    \setlength{\tabcolsep}{5pt}
    
    \begin{tabular}{ccccccc}
    
    \toprule
    
    Dataset & Epoch & TCN & Autoformer & GPT4TS & Total & Selected backbone(s) \\
    
    \midrule
    
    \multirow{3}{*}{ETTh}
    
    & 5   & 3 & \textbf{5} & 6 & 14 & Autoformer \\
    
    & 20  & \textbf{9} & 3 & \textbf{6} & 18 & TCN, GPT4TS \\
    
    & 100 & 0 & 4 & 0 & 4 & -- \\
    
    \midrule
    
    \multirow{3}{*}{ILI}
    
    & 5   & 5 & 3 & 2 & 10 & -- \\
    
    & 20  & 2 & 0 & 2 & 4 & -- \\
    
    & 100 & \textbf{5} & \textbf{7} & \textbf{8} & 20 & TCN, Autoformer, GPT4TS \\
    
    \midrule
    
    \multirow{3}{*}{WEATHER-5K}
    
    & 5   & \textbf{8} & 6 & 4 & 18 & TCN \\
    
    & 20  & 6 & 6 & 0 & 12 & -- \\
    
    & 100 & 4 & \textbf{6} & \textbf{14} & 24 & Autoformer, GPT4TS \\
    
    \bottomrule
    
    \end{tabular}
    \end{adjustbox}
    \end{table}

    \begin{table*}[h]
    \centering
    
    \caption{Performance in MAE|MSE with $H{=}24$ using validation-set early stopping versus epoch sensitivity selection. For each backbone, the
    upper row reports the validation-selected epoch results and the lower row reports the sensitivity-selected epoch results. Lower values
    are better, and the better scores are shown in \textbf{bold}.}
    \begin{adjustbox}{max width=\textwidth}
    \label{tab:val_vs_sens_all}
    \renewcommand{\arraystretch}{0.5}
    \setlength{\tabcolsep}{7pt}
    \begin{tabular}{cccc|cccc}
    \toprule
    \multicolumn{4}{c|}{ETTh} & \multicolumn{4}{c}{ILI} \\
    \midrule
    Backbone & Epoch & ETTh1 & ETTh2 & Backbone & Epoch & Pre-COVID & COVID \\
    \midrule
    \multirow{2}{*}{TCN}
    & 31 & 0.4689 | 0.4690 & 2.7690 | 10.3066
    & \multirow{2}{*}{TCN}
    & 42  & 1.7914 | 7.4170 & \textbf{3.8867 | 30.3466} \\
    & 20 & \textbf{0.4588 | 0.4521} & \textbf{2.5145 | 8.9095}
    & & 100 & \textbf{1.6734 | 5.8091} & 4.5372 | 38.1840 \\
    \midrule
    \multirow{2}{*}{Autoformer}
    & 14 & 0.5270 | 0.5751 & 1.2336 | \textbf{1.9241}
    & \multirow{2}{*}{Autoformer}
    & 23  & 1.9455 | 8.0297 & 3.1915 | 21.2431 \\
    & 5  & \textbf{0.5147 | 0.5605} & \textbf{1.1207} | 1.9961
    & & 100 & \textbf{1.2800 | 3.3924} & \textbf{2.7942 | 15.7648} \\
    \midrule
    \multirow{2}{*}{GPT4TS}
    & 19 & 0.3868 | 0.3335 & 0.5614 | 0.7849
    & \multirow{2}{*}{GPT4TS}
    & 138 & 1.3128 | 4.5128 & 1.7399 | 8.2247 \\
    & 20 & \textbf{0.3796 | 0.3275} & \textbf{0.5573 | 0.7407}
    & & 100 & \textbf{0.9667 | 2.3908} & \textbf{1.3245 | 4.7387} \\
    \midrule
    \multicolumn{8}{c}{WEATHER-5K} \\
    \midrule
    Backbone & Epoch & \multicolumn{2}{c}{Miami} & \multicolumn{2}{c}{Atlanta} & \multicolumn{2}{c}{NYC} \\
    \midrule
    \multirow{2}{*}{TCN}
    & 56  & \multicolumn{2}{c}{\textbf{0.3263 | 0.2610}} & \multicolumn{2}{c}{0.5687 | 0.7530} & \multicolumn{2}{c}{0.8134 | 1.5342} \\
    & 5   & \multicolumn{2}{c}{0.3353 | 0.2700} & \multicolumn{2}{c}{\textbf{0.5652 | 0.7230}} & \multicolumn{2}{c}{\textbf{0.8012 | 1.4215}} \\
    \midrule
    \multirow{2}{*}{Autoformer}
    & 12  & \multicolumn{2}{c}{0.4442 | 0.3901} & \multicolumn{2}{c}{0.6753 | 0.9807} & \multicolumn{2}{c}{0.8837 | 1.8407} \\
    & 5   & \multicolumn{2}{c}{\textbf{0.4309 | 0.3837}} & \multicolumn{2}{c}{\textbf{0.6640 | 0.9252}} & \multicolumn{2}{c}{\textbf{0.8779 | 1.7081}} \\
    \midrule
    \multirow{2}{*}{GPT4TS}
    & 52  & \multicolumn{2}{c}{0.5327 | 0.6062} & \multicolumn{2}{c}{0.8155 | 1.3858} & \multicolumn{2}{c}{1.0283 | 2.3649} \\
    & 100 & \multicolumn{2}{c}{\textbf{0.3488 | 0.3101}} & \multicolumn{2}{c}{\textbf{0.4994 | 0.6067}} & \multicolumn{2}{c}{\textbf{0.6424 | 1.0096}} \\
    \bottomrule
    \end{tabular}
    \end{adjustbox}
    \end{table*}

    \FloatBarrier
    \section{Additional Experiment Result}
    
    \subsection{Probabilistic Forecasting Metrics}\label{sec:probabilistic_results}
    
    Tables~\ref{tab:probabilistic_forecasting_full} reports probabilistic
    forecasting metrics on ETTh. OMPB obtains the best
    NLL in all six ID/OOD settings, the best CRPS in five of six settings, and the lowest ECE in all six settings. The OOD
    results show why interval width alone is insufficient: original models often keep narrow intervals but become severely
    under-covered, while some adaptive baselines recover coverage only with very wide intervals.

    \begin{table*}[h]
    \centering
    
    \caption{Probabilistic forecasting under distribution shift on ETTh. Lower NLL, CRPS, ECE, and interval width are better, while coverage is best when closer to the nominal level. Width should therefore be interpreted jointly with coverage.}
    \begin{adjustbox}{max width=\textwidth}
    \label{tab:probabilistic_forecasting_full}
    \setlength{\tabcolsep}{3.2pt}
    \renewcommand{\arraystretch}{0.7}
    \resizebox{\textwidth}{!}{%
    \begin{tabular}{cccccccccc}
    \toprule
    Backbone & Split & Method & NLL & CRPS & Cov@80 & Cov@95 & Width@80 & Width@95 & ECE \\
    \midrule
    \multirow{10}{*}{Autoformer}
    & \multirow{5}{*}{ID ETTh1}
    & Original & 1.7718 & 0.4287 & 0.6445 & 0.8028 & 1.1327 & 1.7323 & 0.1421 \\
    & & SOLID    & 1.4954 & 0.5462 & 0.9388 & 0.9141 & 1.4870 & 2.2741 & 0.1422 \\
    & & OneNet   & 5.5852 & 0.6910 & 0.7567 & 0.8863 & 1.9091 & 2.9197 & 0.0942 \\
    & & PROCEED  & 2.1991 & 0.7426 & 0.7517 & 0.8762 & 2.3809 & 3.6413 & 0.0924 \\
    & & OMPB     & \textbf{1.0226} & \textbf{0.3703} & \textbf{0.7751} & \textbf{0.9791} & 4.2665 & 6.5250 & \textbf{0.0541} \\
    \cmidrule(lr){2-10}
    & \multirow{5}{*}{OOD ETTh2}
    & Original & 20.2276 & 1.5403 & 0.2581 & 0.3883 & 1.1327 & 1.7323 & 0.5096 \\
    & & SOLID    & 2.1651 & 1.2138 & 0.8657 & 0.9350 & 9.7381 & 14.8932 & 0.1126 \\
    & & OneNet   & 10.2486 & 0.5717 & \textbf{0.8591} & 0.9143 & 1.9091 & 2.9197 & 0.1245 \\
    & & PROCEED  & 17.3017 & 1.1846 & 0.6904 & 0.8102 & 2.3809 & 3.6413 & 0.1356 \\
    & & OMPB     & \textbf{0.7362} & \textbf{0.2420} & 0.8792 & \textbf{0.9501} & 1.4870 & 2.2741 & \textbf{0.1023} \\
    \midrule
    \multirow{10}{*}{GPT4TS}
    & \multirow{5}{*}{ID ETTh1}
    & Original & 1.3541 & \textbf{0.3433} & 0.6941 & 0.8387 & 1.0097 & 1.5442 & 0.0979 \\
    & & SOLID    & 1.1159 & 0.3886 & 0.8595 & 0.9101 & 2.3241 & 3.5543 & 0.0969 \\
    & & OneNet   & 3.4019 & 0.5798 & 0.7644 & 0.8995 & 1.8982 & 2.9031 & 0.0809 \\
    & & PROCEED  & 2.0357 & 0.7036 & 0.7544 & 0.8802 & 2.3034 & 3.5228 & 0.0902 \\
    & & OMPB     & \textbf{1.0146} & 0.3697 & \textbf{0.7670} & \textbf{0.9376} & 1.4492 & 2.2164 & \textbf{0.0593} \\
    \cmidrule(lr){2-10}
    & \multirow{5}{*}{OOD ETTh2}
    & Original & 12.3514 & 0.7294 & 0.4775 & 0.6392 & 1.0097 & 1.5442 & 0.2946 \\
    & & SOLID    & 2.5754 & 1.6636 & 0.8692 & 0.8485 & 13.3194 & 20.3703 & 0.1717 \\
    & & OneNet   & 5.5656 & 0.4346 & \textbf{0.8804} & 0.9391 & 1.8982 & 2.9031 & 0.1204 \\
    & & PROCEED  & 49.2202 & 1.2571 & 0.6883 & 0.8091 & 2.3034 & 3.5228 & 0.1355 \\
    & & OMPB     & \textbf{0.7736} & \textbf{0.2468} & 0.7454 & \textbf{0.9418} & 1.4492 & 2.2164 & \textbf{0.1001} \\
    \midrule
    \multirow{10}{*}{TCN}
    & \multirow{5}{*}{ID ETTh1}
    & Original & 1.4953 & 0.4037 & 0.6602 & 0.8192 & 1.1063 & 1.6920 & 0.1279 \\
    & & SOLID    & 1.3446 & 0.4757 & 0.9118 & 0.9060 & 3.3620 & 5.1418 & 0.1156 \\
    & & OneNet   & 1.2186 & 0.4686 & \textbf{0.7722} & 0.9127 & 1.8874 & 2.8866 & 0.0677 \\
    & & PROCEED  & 1.8723 & 0.6646 & 0.7571 & 0.8841 & 2.2259 & 3.4043 & 0.0880 \\
    & & OMPB     & \textbf{1.0067} & \textbf{0.3690} & 0.7589 & \textbf{0.9635} & 1.4114 & 2.1586 & \textbf{0.0645} \\
    \cmidrule(lr){2-10}
    & \multirow{5}{*}{OOD ETTh2}
    & Original & 54.2016 & 3.3362 & 0.0745 & 0.1149 & 1.1063 & 1.6920 & 0.7064 \\
    & & SOLID    & 2.9577 & 3.7641 & 0.9896 & 0.9955 & 37.0345 & 56.6394 & 0.1994 \\
    & & OneNet   & 0.8826 & 0.2974 & 0.9018 & \textbf{0.9639} & 1.8874 & 2.8866 & 0.1164 \\
    & & PROCEED  & 81.1387 & 1.3296 & 0.6863 & 0.8080 & 2.2259 & 3.4043 & 0.1353 \\
    & & OMPB     & \textbf{0.8110} & \textbf{0.2517} & \textbf{0.8593} & 0.9335 & 1.4114 & 2.1586 & \textbf{0.0978} \\
    \bottomrule
    \end{tabular}%
    }
    \end{adjustbox}
    \end{table*}
    \subsection{Regression-Proxy Diagnostics}
    \label{sec:reg_proxy_diagnostics}
    
    Table~\ref{tab:reg_proxy_full_cert} compares the realized target proxy risk with the full post-hoc martingale and
    i.i.d. proxy certificates. The full certificates include the residual proxy mismatch term and are therefore diagnostic
    only; OMPB uses the input-identifiable certificate in Equation~\eqref{eq:pb-online} for online calibration. The
    martingale certificate is tighter than the variance-blind i.i.d. baseline in all four representative settings. This does
    not imply that every martingale bound is tighter than every i.i.d. bound; here the gain comes from the Freedman
    variance-adaptive correction when the bounded proxy loss has low predictable variation. Figure~\ref{fig:reg_proxy_bound_curves}
    visualizes the corresponding online curves.
    
    These results should be read as certificate diagnostics rather than deployment-time objectives. The full martingale and
    i.i.d. columns include the post-hoc residual proxy mismatch term, which requires target labels and is unavailable at
    prediction time. Their purpose is to check whether the bounded proxy interpretation is numerically meaningful after the
    stream is observed. The online algorithm uses only the source proxy risk, the martingale correction, and the unlabeled
    input-side disagreement term.
    
    \begin{table}[!htbp]
    \centering
    \caption{Regression-proxy full-certificate diagnostics. Target proxy risk is the realized target proxy-risk diagnostic.}
    \label{tab:reg_proxy_full_cert}
    \begin{tabular}{ccccccc}
    \toprule
    Dataset & Backbone & $H$ & Split & Target Proxy Risk & Full Mart. Bound & Full IID Bound \\
    \midrule
    ETTh & TCN & 96 & ETTh2 & 0.9859 & 1.0853 & 1.1371 \\
    ILI & GPT4TS & 48 & COVID & 0.9462 & 1.0369 & 1.1008 \\
    WEATHER-5K & Autoformer & 96 & Close & 0.7228 & 0.7963 & 0.8132 \\
    WEATHER-5K & Autoformer & 96 & Far & 0.8439 & 0.9257 & 0.9498 \\
    \bottomrule
    \end{tabular}
    \end{table}
    
    \begin{figure}[!htbp]
    \centering
    \begin{minipage}{0.48\textwidth}
    \centering
    \includegraphics[width=\linewidth]{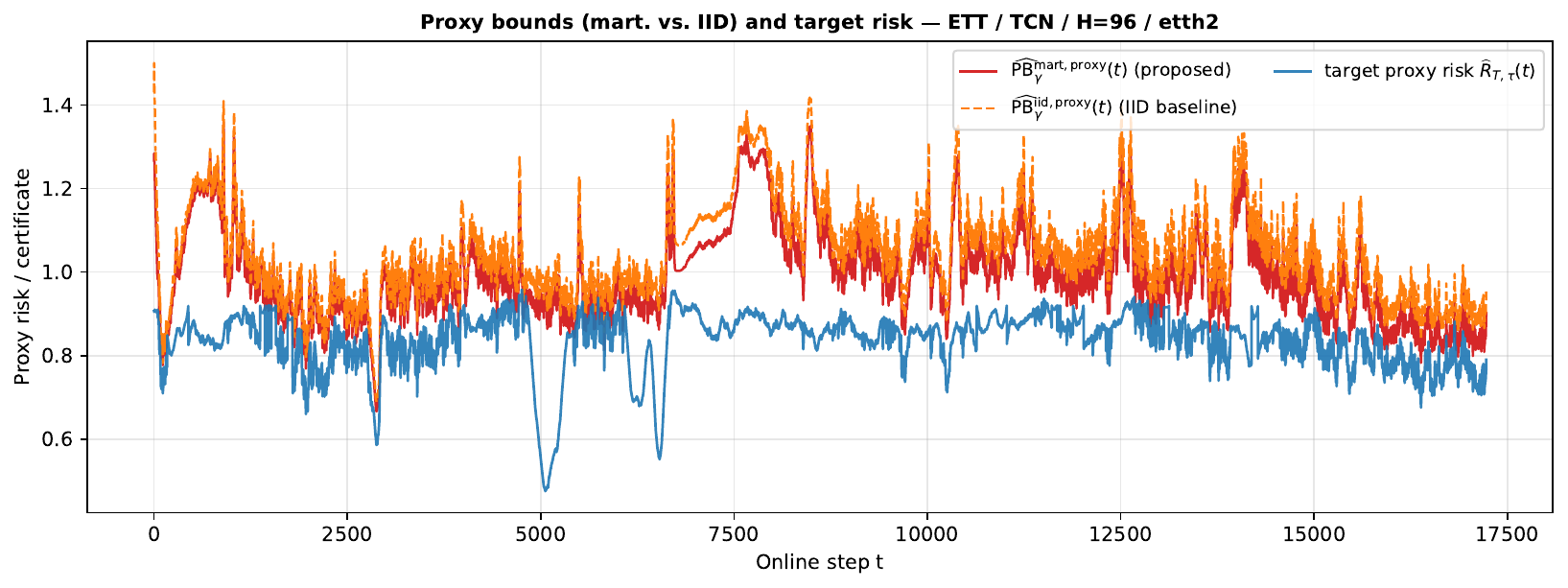}
    \end{minipage}
    \hfill
    \begin{minipage}{0.48\textwidth}
    \centering
    \includegraphics[width=\linewidth]{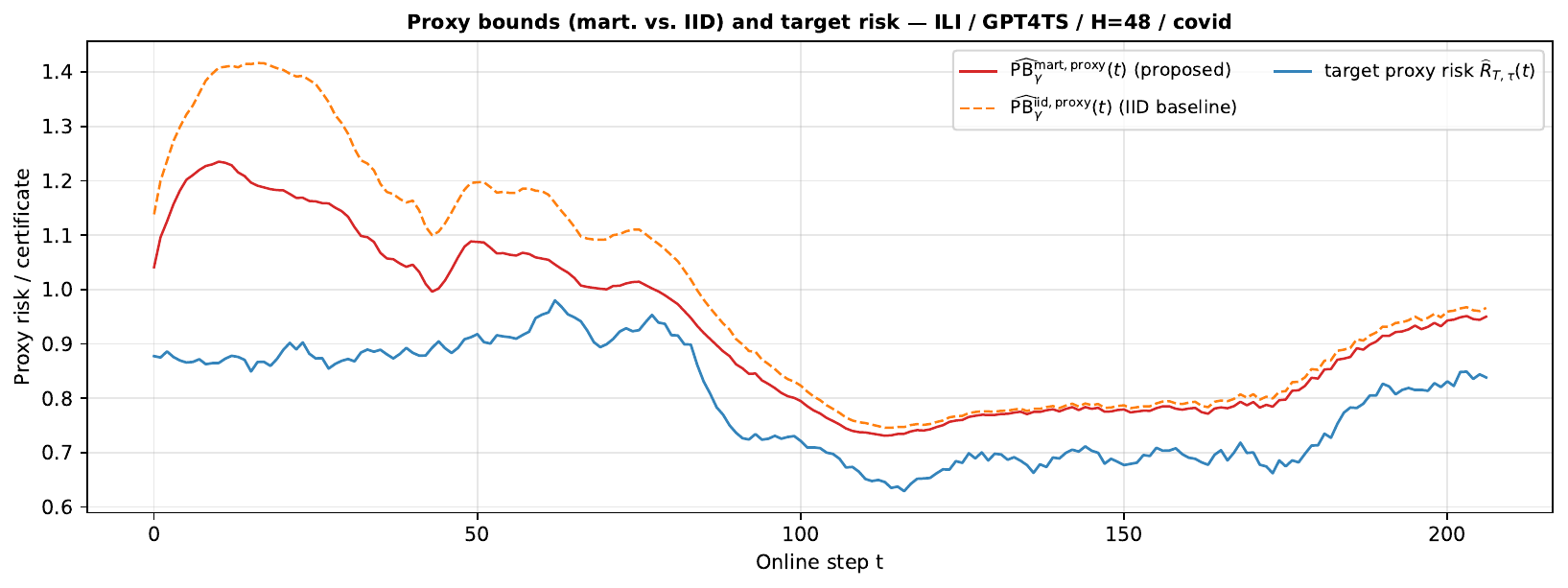}
    \end{minipage}

    \begin{minipage}{0.48\textwidth}
    \centering
    \includegraphics[width=\linewidth]{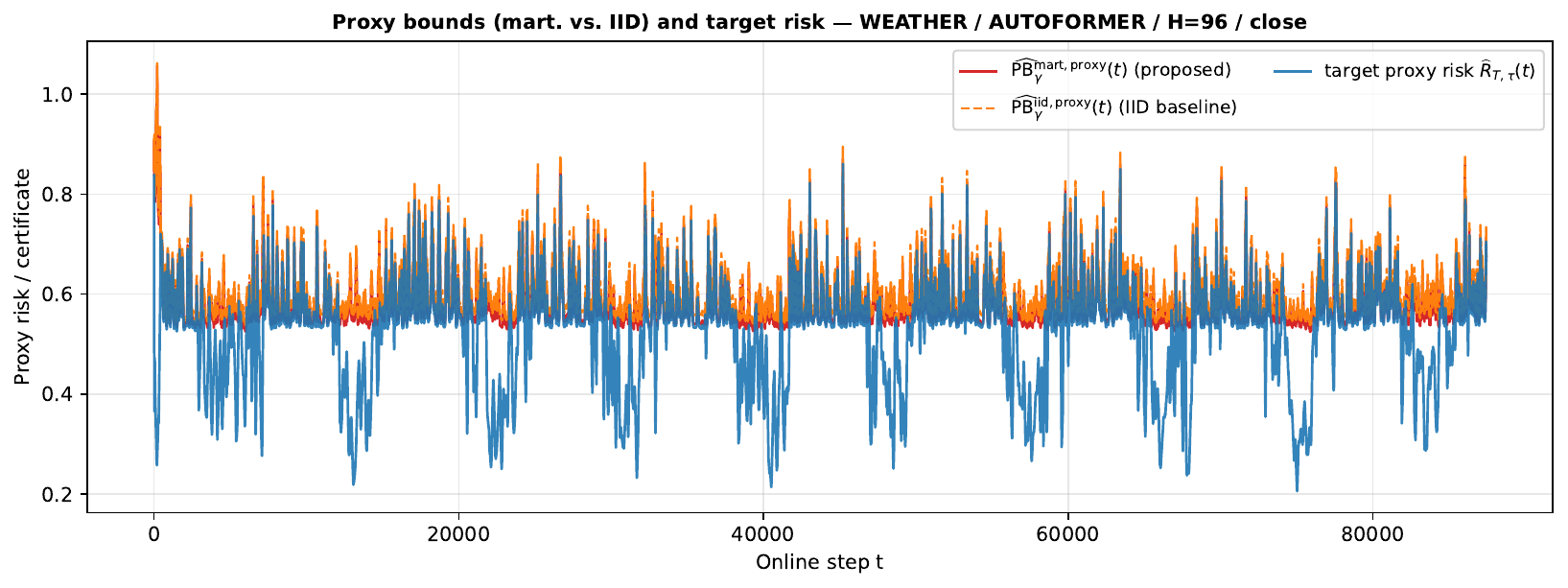}
    \end{minipage}
    \hfill
    \begin{minipage}{0.48\textwidth}
    \centering
    \includegraphics[width=\linewidth]{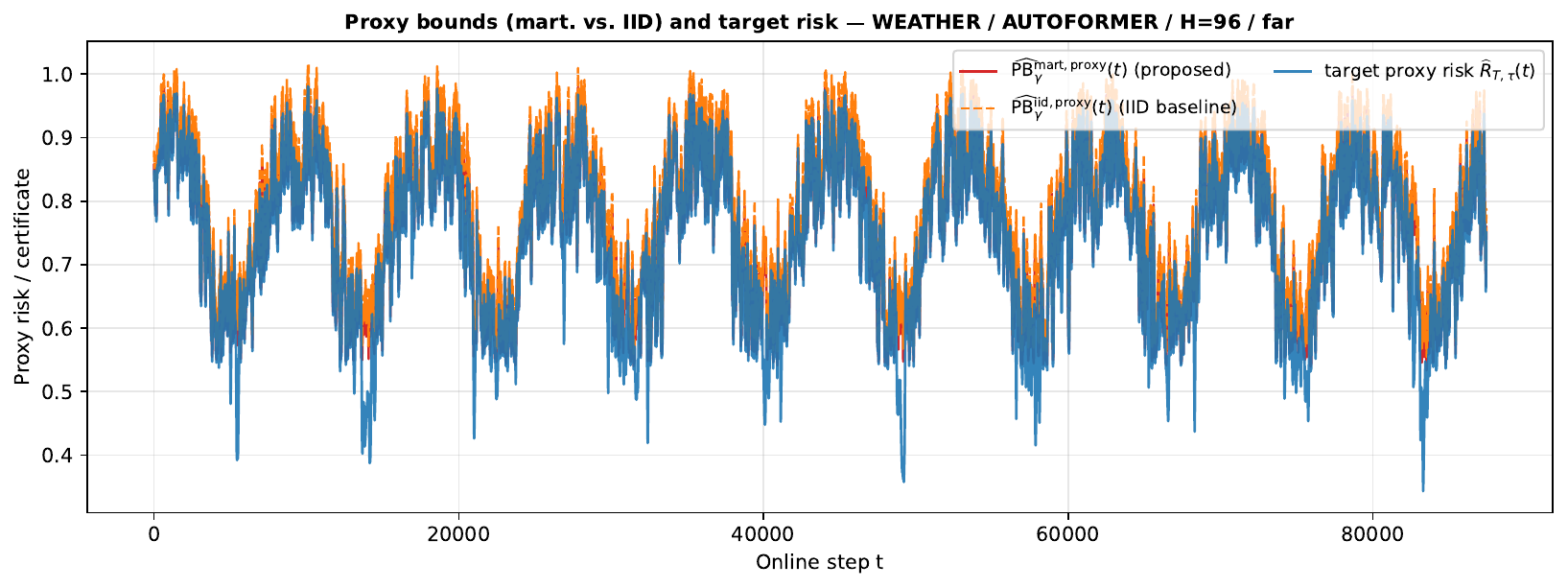}
    \end{minipage}
    \caption{Regression-proxy diagnostics comparing the realized target proxy risk, the full martingale proxy certificate, and the full i.i.d. proxy certificate on representative ETTh, ILI, and WEATHER-5K shifts.}
    \label{fig:reg_proxy_bound_curves}
    \end{figure}
    \FloatBarrier
    
    \subsection{Certificate Diagnostics}
    \label{sec:timeseries_diagnostics}
    
    The certificate diagnostics evaluate whether the certificate behaves as a meaningful deployment-time monitoring signal
    over the online stream. These plots compare the online certificate, posterior disagreement, external shift-severity
    metrics, and realized pre-calibration target error over online prediction steps. The shift-severity curves are diagnostic
    only and are not used as target labels or as part of the calibration objective.
    
    Figure~\ref{fig:cert_shift_timeseries_main} in the main text shows the WEATHER-5K far-OOD stream, where the certificate
    follows the recurring seasonal shift pattern and responds to local error spikes. Figure~\ref{fig:cert_shift_timeseries_appendix}
    shows the corresponding ETTh and ILI streams. On ETTh, the certificate exhibits an abrupt spike near a difficult shifted
    regime, consistent with a sudden increase in shift severity and target error. On ILI, the certificate increases more
    gradually through the COVID-period shift, matching the slower evolution of the target regime. Together, these trajectories
    show that the certificate reacts to different shift patterns, abrupt, gradual, and seasonal, while remaining a monitoring
    signal rather than a direct MAE/MSE bound.
    
    \begin{figure}[!htbp]
    \centering
    \begin{minipage}{0.48\textwidth}
    \centering
    \includegraphics[width=\linewidth]{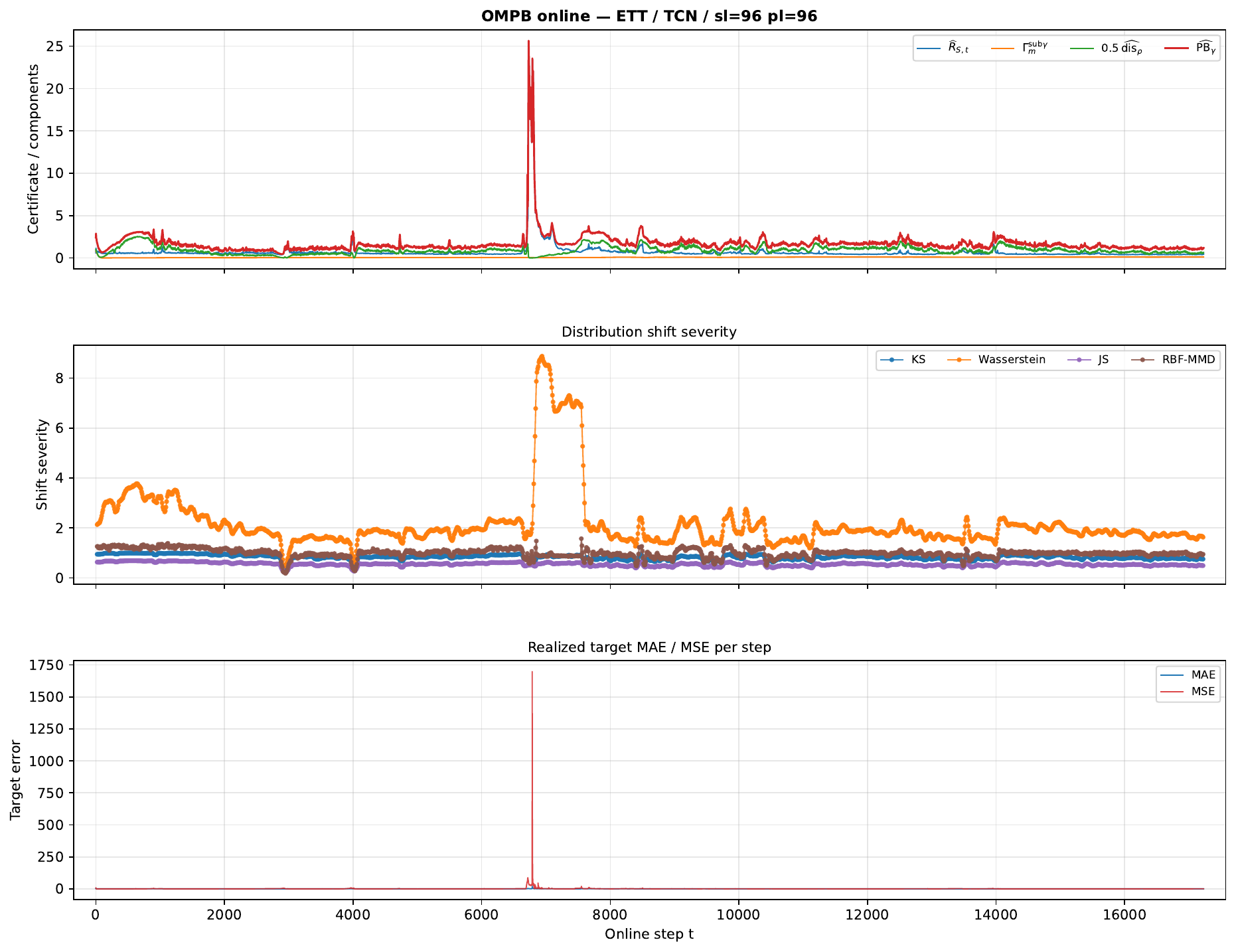}
    \end{minipage}
    \hfill
    \begin{minipage}{0.48\textwidth}
    \centering
    \includegraphics[width=\linewidth]{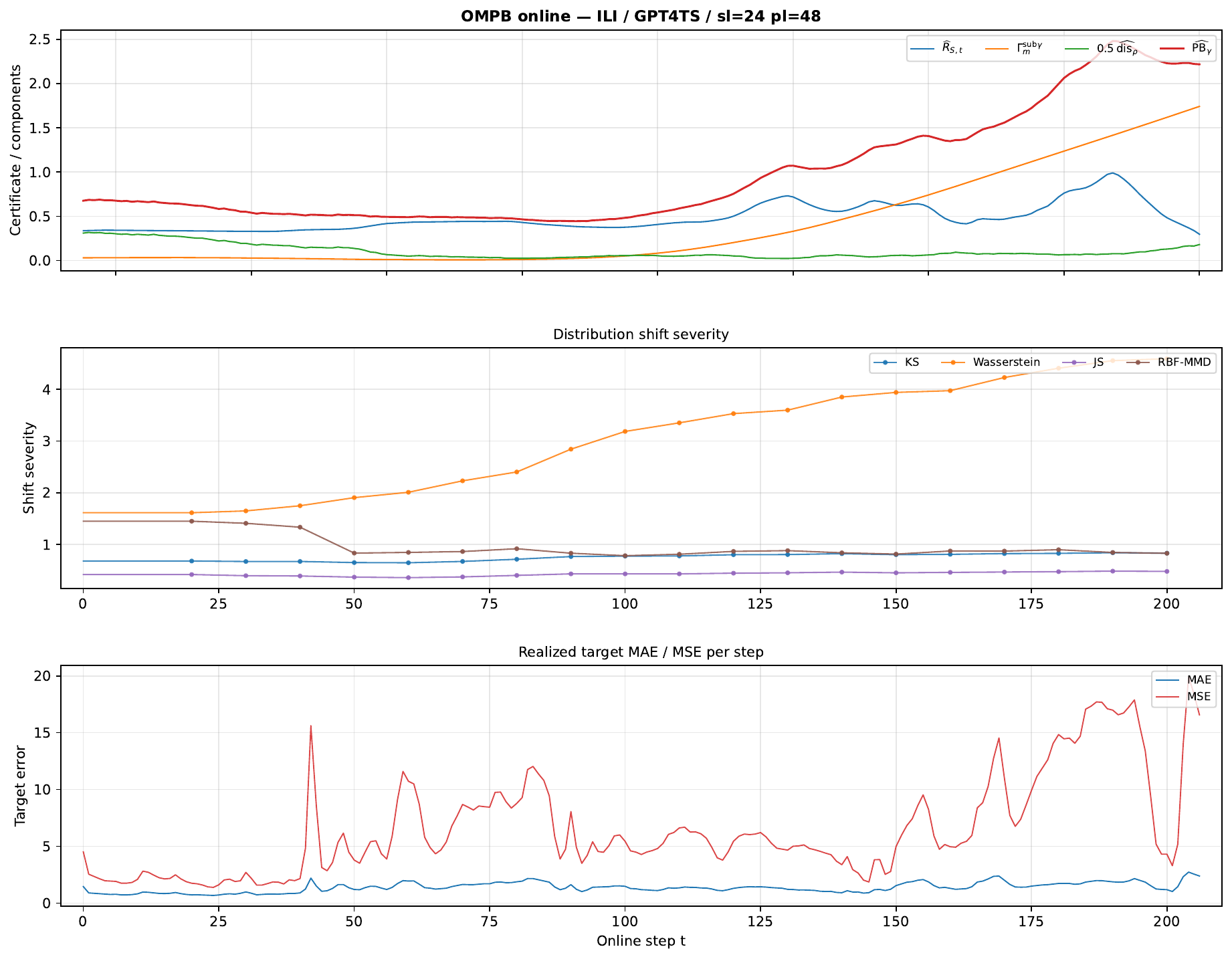}
    \end{minipage}
    \caption{Time-series certificate diagnostics for ETTh and ILI. ETTh shows an abrupt certificate spike near a difficult shifted regime, while ILI shows a gradual certificate increase through the COVID-period shift. These trajectories illustrate how the certificate responds to different temporal shift patterns.}
    \label{fig:cert_shift_timeseries_appendix}
    \end{figure}
    
    \FloatBarrier

    \subsection{Correlation and Scatter Diagnostics}
    \label{sec:correlation_diagnostics}
    
    The correlation diagnostics quantify whether the online certificate and disagreement term are associated with realized
    forecasting error and external distribution-shift severity. Tables~\ref{tab:cert_error_corr} and~\ref{tab:dis_shift_corr}
    report Pearson and Spearman correlations between the online certificate and target MAE/MSE, and between posterior
    disagreement and shift-severity metrics. Figure~\ref{fig:correlation_scatter_diagnostics} provides the corresponding
    scatter plots.
    
    The online certificate is positively correlated with realized target error across all datasets. ETTh and ILI show
    stronger Pearson correlations, while WEATHER-5K shows stronger Spearman correlations, suggesting a monotonic but
    nonlinear relationship under seasonal shift. The disagreement term is also positively associated with most external
    shift metrics, especially on ILI and in the rank correlations for WEATHER-5K. These results support the use of the
    certificate and disagreement as informative warning signals under distribution shift, without claiming that either is a
    tight bound on MAE or MSE.
    
    These diagnostics also clarify when disagreement is informative. It works best when posterior samples retain meaningful
    epistemic diversity and when the source-side disagreement scale is not saturated. If the posterior collapses, posterior
    samples become too similar and disagreement can understate uncertainty. If the posterior is excessively diffuse,
    disagreement can become less selective. Thus, the scatter diagnostics support disagreement as a deployment-time warning
    signal rather than as a standalone shift detector.
    
    \begin{table}[!htbp]
    \centering
    
    \caption{Correlation between the online certificate and realized target error. Each cell reports Pearson $r$ \(\mid\) Spearman $\rho$.}
    \begin{adjustbox}{max width=\textwidth}
    \label{tab:cert_error_corr}
    \begin{tabular}{ccc}
    \toprule
    Dataset & Target MAE & Target MSE \\
    \midrule
    ETTh & 0.592 \(\mid\) 0.278 & 0.339 \(\mid\) 0.383 \\
    ILI & 0.507 \(\mid\) 0.276 & 0.432 \(\mid\) 0.344 \\
    WEATHER-5K & 0.239 \(\mid\) 0.439 & 0.270 \(\mid\) 0.488 \\
    \bottomrule
    \end{tabular}
    \end{adjustbox}
    \end{table}
    
    \begin{table}[!htbp]
    \centering
    
    \caption{Correlation between disagreement $\widehat{\mathrm{dis}}_{\rho}(t)$ and distribution-shift severity metrics. Each cell reports Pearson $r$ \(\mid\) Spearman $\rho$.}
    \begin{adjustbox}{max width=\textwidth}
    \label{tab:dis_shift_corr}
    \setlength{\tabcolsep}{3pt}
    \begin{tabular}{ccccc}
    \toprule
    Dataset & KS & Wasserstein & JS & RBF-MMD \\
    \midrule
    ETTh & 0.247 \(\mid\) 0.180 & 0.044 \(\mid\) 0.333 & 0.278 \(\mid\) 0.181 & 0.320 \(\mid\) 0.319 \\
    ILI & 0.525 \(\mid\) 0.439 & 0.636 \(\mid\) 0.402 & 0.621 \(\mid\) 0.489 & -0.118 \(\mid\) 0.070 \\
    WEATHER-5K & 0.139 \(\mid\) 0.547 & 0.197 \(\mid\) 0.565 & 0.152 \(\mid\) 0.559 & 0.132 \(\mid\) 0.543 \\
    \bottomrule
    \end{tabular}
    \end{adjustbox}
    \end{table}
    
    \begin{figure}[!htbp]
    \centering
    \begin{minipage}{0.32\textwidth}
    \centering
    \includegraphics[width=\linewidth]{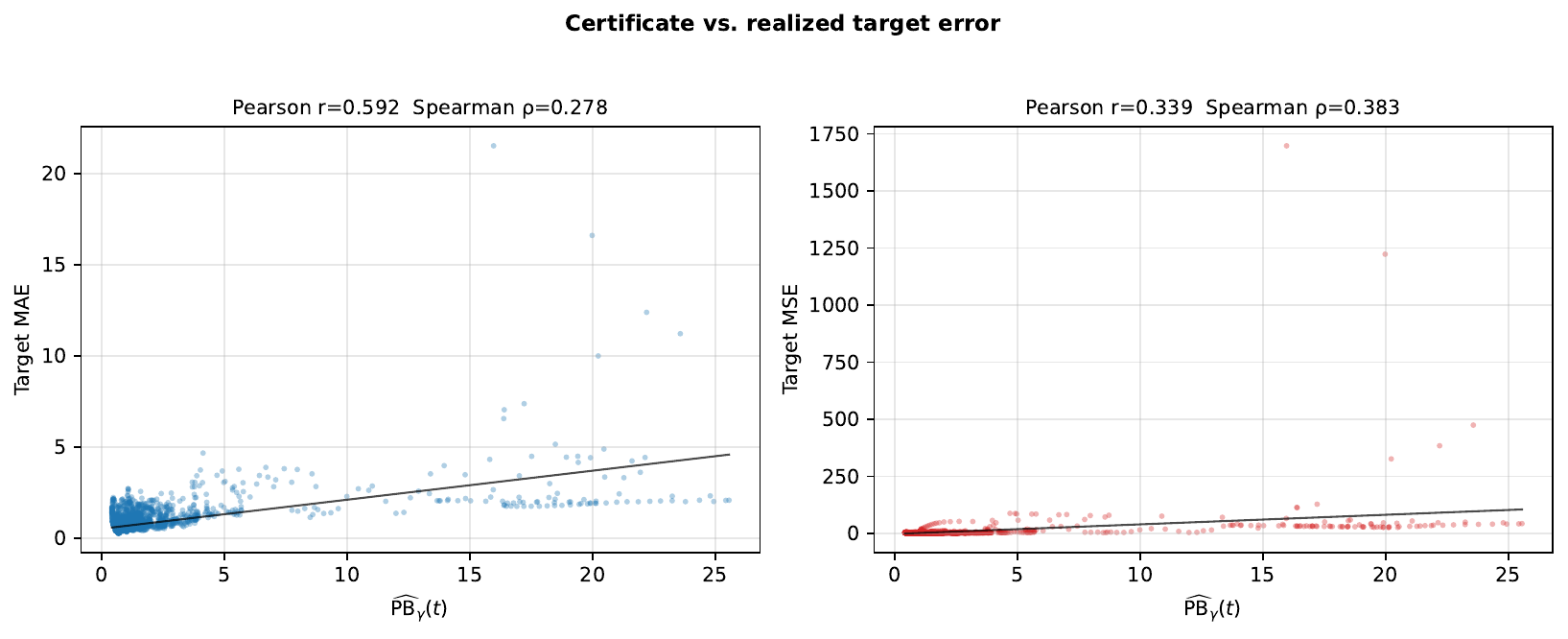}
    \end{minipage}
    \hfill
    \begin{minipage}{0.32\textwidth}
    \centering
    \includegraphics[width=\linewidth]{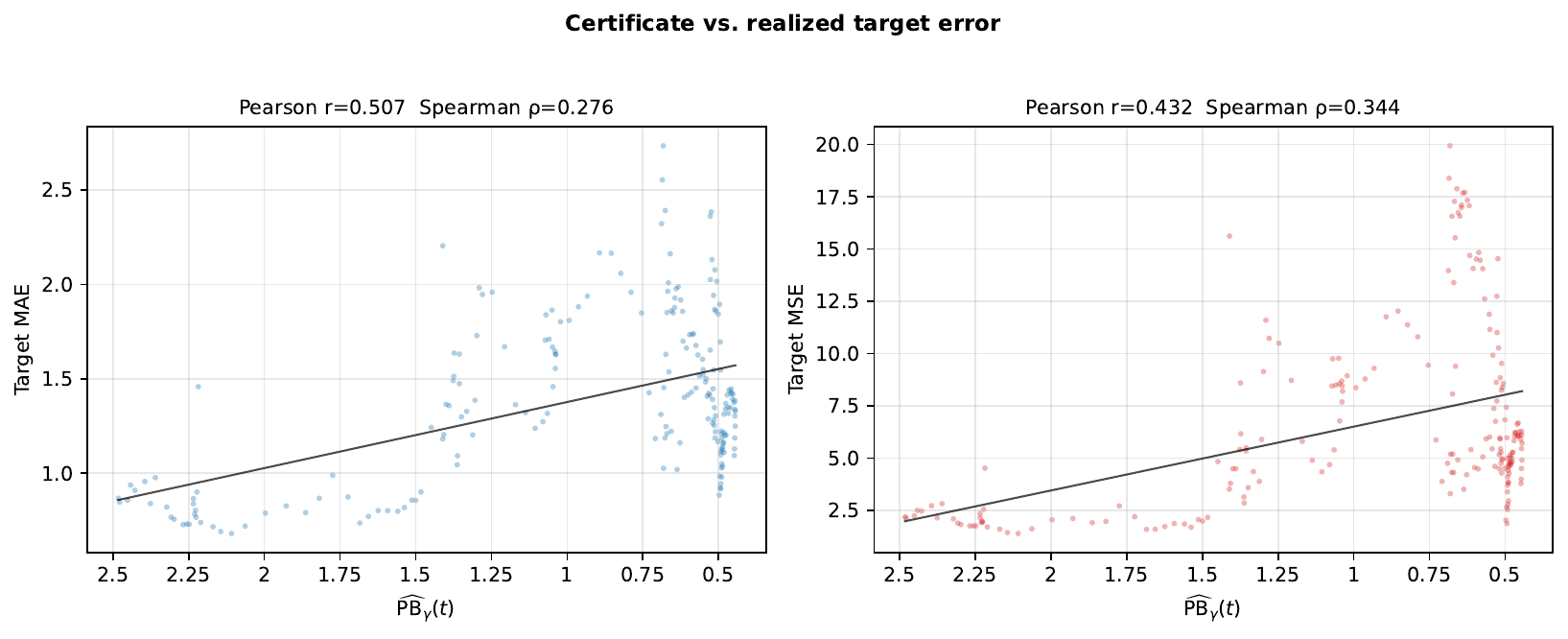}
    \end{minipage}
    \hfill
    \begin{minipage}{0.32\textwidth}
    \centering
    \includegraphics[width=\linewidth]{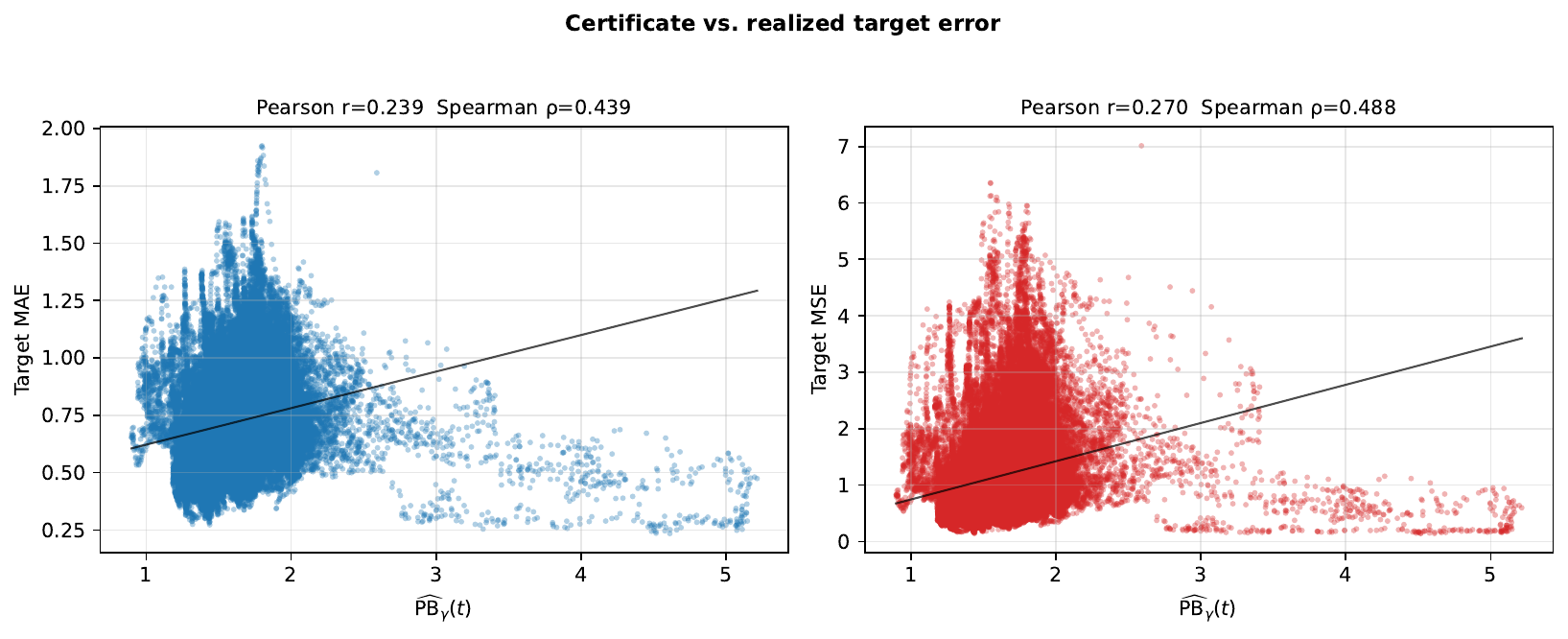}
    \end{minipage}

    \begin{minipage}{0.32\textwidth}
    \centering
    \includegraphics[width=\linewidth]{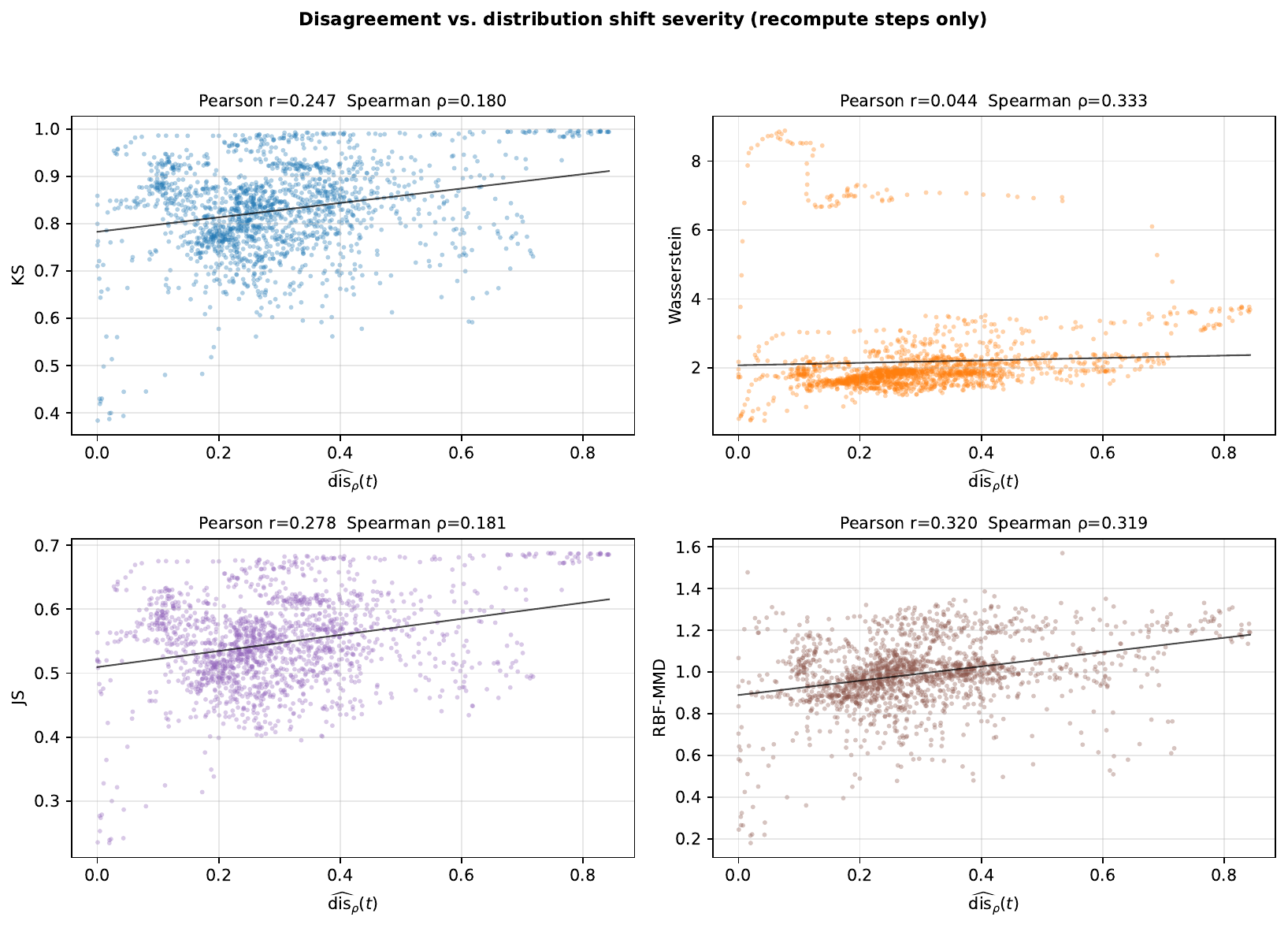}
    \end{minipage}
    \hfill
    \begin{minipage}{0.32\textwidth}
    \centering
    \includegraphics[width=\linewidth]{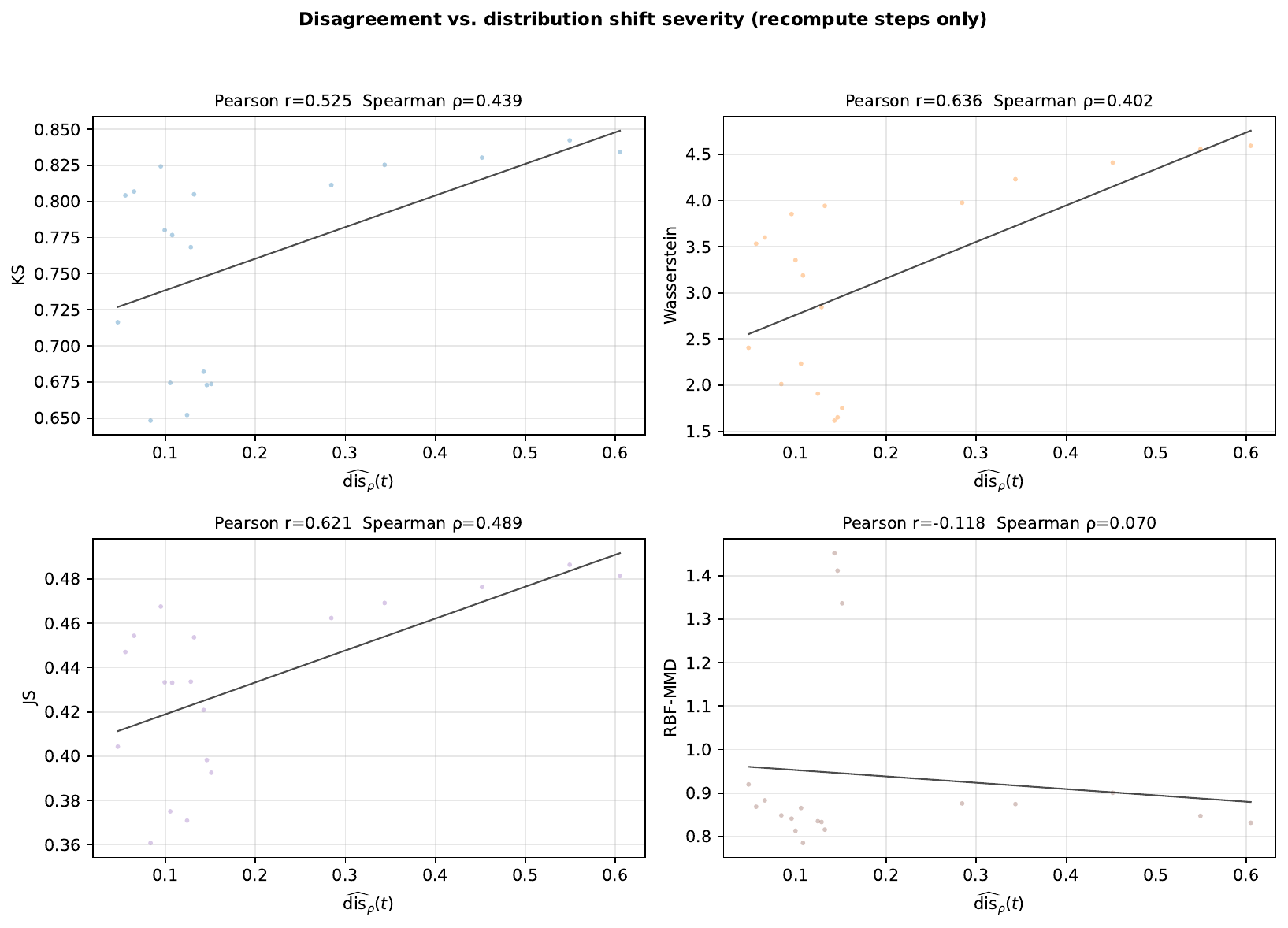}
    \end{minipage}
    \hfill
    \begin{minipage}{0.32\textwidth}
    \centering
    \includegraphics[width=\linewidth]{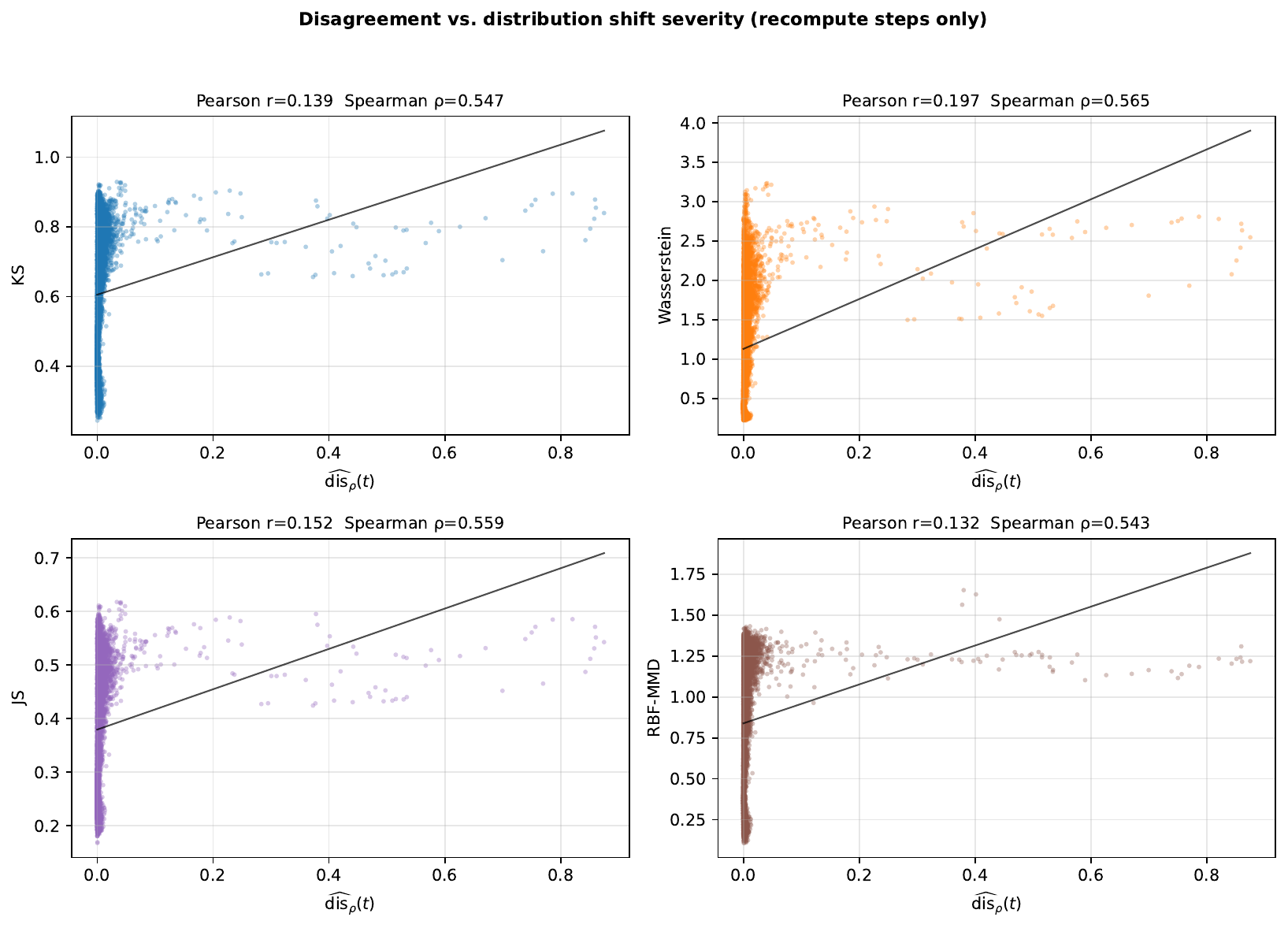}
    \end{minipage}
    \caption{Scatter diagnostics for the online monitoring signals. Top row: online certificate versus realized forecasting error for ETTh, ILI, and WEATHER-5K. Bottom row: posterior disagreement versus distribution-shift severity for the same datasets.}
    \label{fig:correlation_scatter_diagnostics}
    \end{figure}
    
    \FloatBarrier
    \subsection{Sensitivity Analysis}
    \label{sec:sensitivity_analysis}
    
    Table~\ref{tab:tau_sensitivity} shows that changing the disagreement scale mainly changes the certificate and
    disagreement scale rather than point prediction: MAE remains between $0.6728$ and $0.6737$ over factors from $0.25$ to
    $64$. Table~\ref{tab:k_sensitivity_runtime} shows that increasing $K$ from $2$ to $32$ changes MAE by only
    $4.27\times10^{-7}$ and MSE by $3.24\times10^{-6}$ in the representative ETTh/TCN/$H{=}96$ setting, while runtime
    increases at $K=32$ due to the $496$ pairwise comparisons. Figure~\ref{fig:sensitivity_curves} plots these sweeps and
    the predictable variance-proxy sweep.
    
    \begin{table}[!htbp]
    \centering
    
    \caption{Sensitivity to disagreement scale $\tau_d$ on ETTh with TCN at $H=96$ on OOD ETTh2.}
    \begin{adjustbox}{max width=\textwidth}
    \label{tab:tau_sensitivity}
    \begin{tabular}{rrrrr}
    \toprule
    $\tau_d$ factor & MAE & MSE & Cert. & Dis. \\
    \midrule
    0.25 & 0.6728 & 0.7694 & 1.4375 & 0.0000 \\
    0.5 & 0.6728 & 0.7694 & 1.4375 & 0.0000 \\
    1 & 0.6728 & 0.7694 & 1.4375 & 0.0000 \\
    2 & 0.6728 & 0.7698 & 1.4376 & 0.0005 \\
    4 & 0.6731 & 0.7793 & 1.4389 & 0.0139 \\
    8 & 0.6735 & 0.8703 & 1.4429 & 0.1555 \\
    16 & 0.6737 & 1.0314 & 1.4447 & 0.4091 \\
    32 & 0.6733 & 0.9383 & 1.4409 & 0.2571 \\
    64 & 0.6729 & 0.8488 & 1.4381 & 0.1381 \\
    \bottomrule
    \end{tabular}
    \end{adjustbox}
    \end{table}
    
    \begin{table}[!htbp]
    \centering
    
    \caption{Posterior-sample count and runtime on ETTh with TCN at $H=96$ on OOD ETTh2. Runtime is measured per online step.}
    \begin{adjustbox}{max width=\textwidth}
    \label{tab:k_sensitivity_runtime}
    \begin{tabular}{rrrrrrr}
    \toprule
    $K$ & Pairs & MAE & MSE & Cert. & Time (ms) & Steps/s \\
    \midrule
    2 & 1 & 0.672819251 & 0.769440815 & 1.437516313 & 102.15 & 9.79 \\
    4 & 6 & 0.672819032 & 0.769439362 & 1.437514619 & 103.08 & 9.70 \\
    8 & 28 & 0.672818925 & 0.769438442 & 1.437513823 & 103.83 & 9.63 \\
    16 & 120 & 0.672818874 & 0.769438162 & 1.437513388 & 96.76 & 10.33 \\
    32 & 496 & 0.672818824 & 0.769437596 & 1.437513077 & 239.59 & 4.17 \\
    \bottomrule
    \end{tabular}
    \end{adjustbox}
    \end{table}
    
    \begin{figure}[!htbp]
    \centering
    \begin{minipage}{0.32\textwidth}
    \centering
    \includegraphics[width=\linewidth]{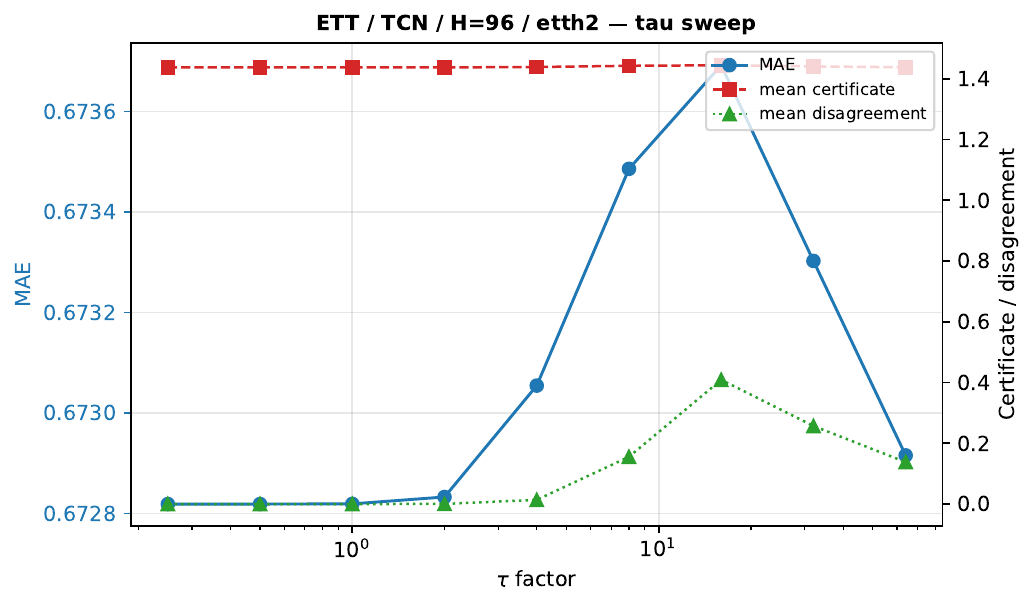}
    \end{minipage}
    \hfill
    \begin{minipage}{0.32\textwidth}
    \centering
    \includegraphics[width=\linewidth]{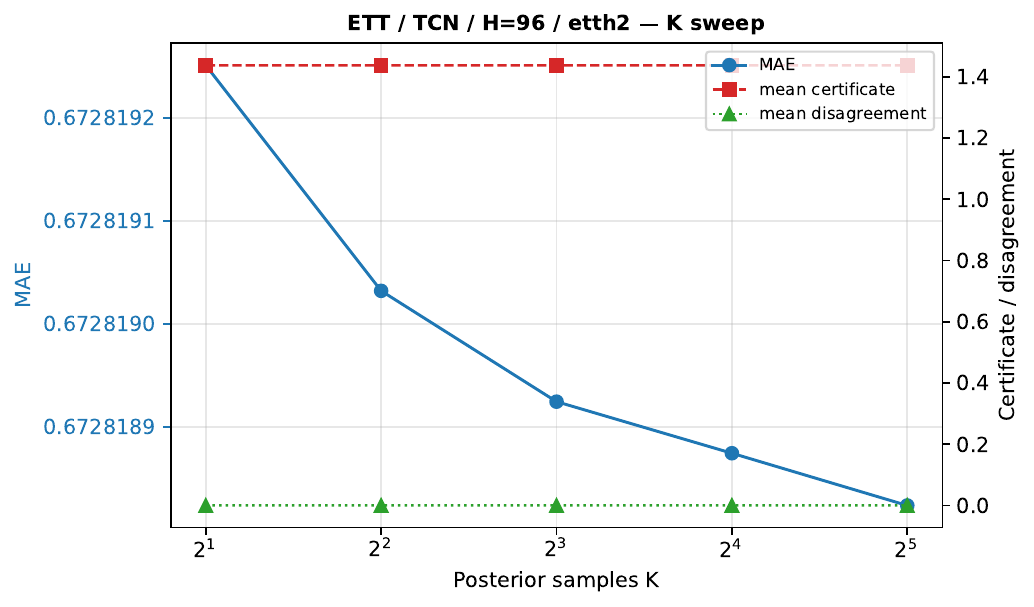}
    \end{minipage}
    \hfill
    \begin{minipage}{0.32\textwidth}
    \centering
    \includegraphics[width=\linewidth]{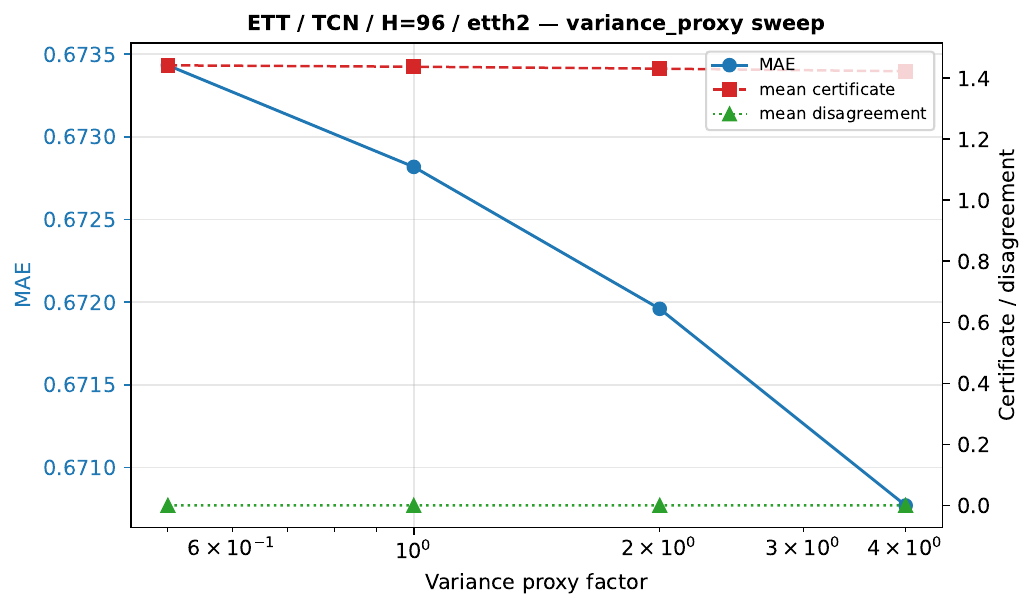}
    \end{minipage}
    \caption{Sensitivity diagnostics on ETTh with TCN at $H=96$ on OOD ETTh2. The panels summarize the disagreement-scale sweep, posterior-sample-count sweep, and predictable variance-proxy sweep.}
    \label{fig:sensitivity_curves}
    \end{figure}
    
    Table~\ref{tab:variance_proxy_sensitivity} evaluates predictable variance-proxy misspecification. Moderate
    under-estimation or over-estimation mainly changes certificate conservativeness rather than destabilizing calibrated
    forecasts in this shifted stream. Table~\ref{tab:disagreement_normalization} shows that clipped and normalized
    disagreement variants are stable, while an unclipped variant violates the bounded proxy requirement and becomes
    numerically unstable.
    
    Together, these sweeps show that OMPB is not relying on a fragile hyperparameter setting in the representative
    ETTh/TCN/$H{=}96$ stream. The disagreement scale $\tau_d$ primarily rescales monitoring; the posterior sample count $K$
    has little effect beyond computational cost once a small number of posterior pairs is available; and moderate
    misspecification of the predictable variance proxy changes the certificate level rather than the forecast trajectory.
    The boundedness of the disagreement score is the important invariant: removing clipping breaks the proxy-risk
    interpretation and produces the unstable values in Table~\ref{tab:disagreement_normalization}.
    
    \begin{table}[!htbp]
    \centering
    
    \caption{Variance-proxy misspecification on ETTh with TCN at $H=96$ on OOD ETTh2.}
    \begin{adjustbox}{max width=\textwidth}
    \label{tab:variance_proxy_sensitivity}
    \begin{tabular}{cccc}
    \toprule
    $V$ factor & MAE & MSE & Cert. \\
    \midrule
    0.5 & 0.6734 & 0.7659 & 1.4421 \\
    1 & 0.6728 & 0.7694 & 1.4375 \\
    2 & 0.6720 & 0.7743 & 1.4312 \\
    4 & 0.6708 & 0.7808 & 1.4226 \\
    \bottomrule
    \end{tabular}
    \end{adjustbox}
    \end{table}
    
    \begin{table}[!htbp]
    \centering
    
    \caption{Disagreement normalization sensitivity on ETTh with TCN at $H=96$ on OOD ETTh2. The unclipped variant is diagnostic only because it violates boundedness.}
    \begin{adjustbox}{max width=\textwidth}
    \label{tab:disagreement_normalization}
    \begin{tabular}{ccccc}
    \toprule
    Mode & MAE & MSE & Cert. & Dis. \\
    \midrule
    Clipped & 0.6728 & 0.7694 & 1.4375 & 0.0312 \\
    Normalized & 0.6728 & 0.7798 & 1.4377 & 0.0201 \\
    Unclipped & 100.6700 & 610.6500 & 601.7700 & 1216.6500 \\
    \bottomrule
    \end{tabular}
    \end{adjustbox}
    \end{table}
    \FloatBarrier
    
    \subsection{Feedback-Delay Analysis}
    \label{sec:feedback_delay_analysis}
    
    Table~\ref{tab:feedback_delay} isolates the effect of longer label delays while current target inputs remain available
    for shift sensing. Performance degrades as supervision becomes increasingly stale, supporting the predict-then-update
    design: unlabeled current inputs can be used immediately for certificate monitoring, while delayed labels are
    incorporated conservatively once available. Figure~\ref{fig:feedback_delay_curves} shows the corresponding rolling MAE
    trajectories.
    
    This analysis separates two information channels in the online protocol. Current target inputs are available before the
    forecast and can be used immediately to estimate source-target disagreement. Target labels, however, arrive only after a
    delay and may describe a regime that has already moved on. The results show that OMPB still benefits from immediate
    input-side shift sensing, but accuracy worsens as the supervised feedback delay grows from one step to $2H$. This is the
    setting where aggressive online tuning is most likely to overfit stale supervision, motivating the certificate-driven
    regularization and frozen-backbone design.
    
    \begin{table}[!htbp]
    \centering
    
    \caption{Feedback-delay analysis for OMPB. Labels are released after delay \(D\), expressed relative to prediction horizon \(H\). The one-step delay is the main baseline-comparison protocol.}
    \begin{adjustbox}{max width=\textwidth}
    \label{tab:feedback_delay}
    \begin{tabular}{cccc}
    \toprule
    Setting & Delay \(D\) & MAE & MSE \\
    \midrule
    \multirow{4}{*}{\begin{tabular}[c]{@{}l@{}}ETTh / TCN\\OOD ETTh2, \(H=96\)\end{tabular}}
    &Original & 2.5636 & 8.9076 \\
    & 1 step & 0.5570 & 0.5912 \\
    & \(H/2\) & 0.8867 & 0.8905 \\
    & \(H\)   & 1.0047 & 1.1443 \\
    & \(2H\) & 1.8608 & 2.7751 \\
    \midrule
    \multirow{4}{*}{\begin{tabular}[c]{@{}l@{}}ILI / GPT4TS\\OOD COVID, \(H=48\)\end{tabular}}
    & Original & 1.3013 & 4.2179 \\
    & 1 step & 0.9263 & 2.6351 \\
    & \(H/2\) & 1.3225 & 2.8780 \\
    & \(H\)   & 1.8200 & 4.9499 \\
    & \(2H\) & 2.0668 & 8.7554 \\
    \midrule
    \multirow{4}{*}{\begin{tabular}[c]{@{}l@{}}WEATHER-5K / Autoformer\\OOD far NYC, \(H=96\)\end{tabular}}
    &Original & 0.9743 & 2.0091 \\
    & 1 step & 0.4344 & 0.4720 \\
    & \(H/2\) & 0.5858 & 0.8004 \\
    & \(H\)   & 1.0746 & 1.0307 \\
    & \(2H\) & 1.6230 & 2.3945 \\
    \bottomrule
    \end{tabular}
    \end{adjustbox}
    \end{table}
    
    \begin{figure}[!htbp]
    \centering
    \begin{minipage}{0.32\textwidth}
    \centering
    \includegraphics[width=\linewidth]{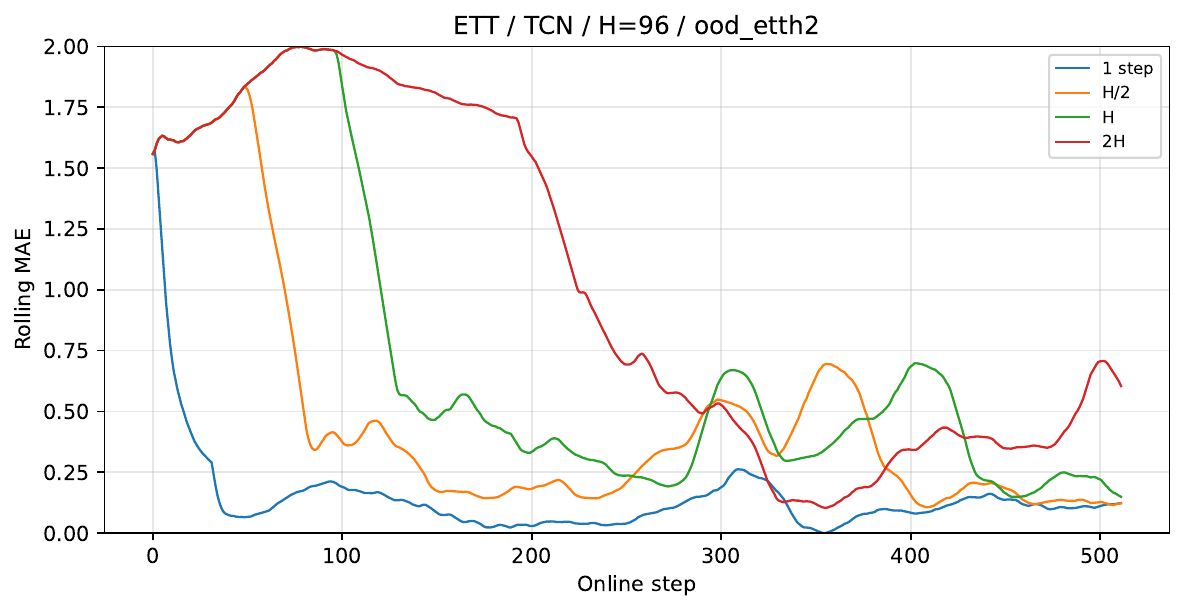}
    \end{minipage}
    \hfill
    \begin{minipage}{0.32\textwidth}
    \centering
    \includegraphics[width=\linewidth]{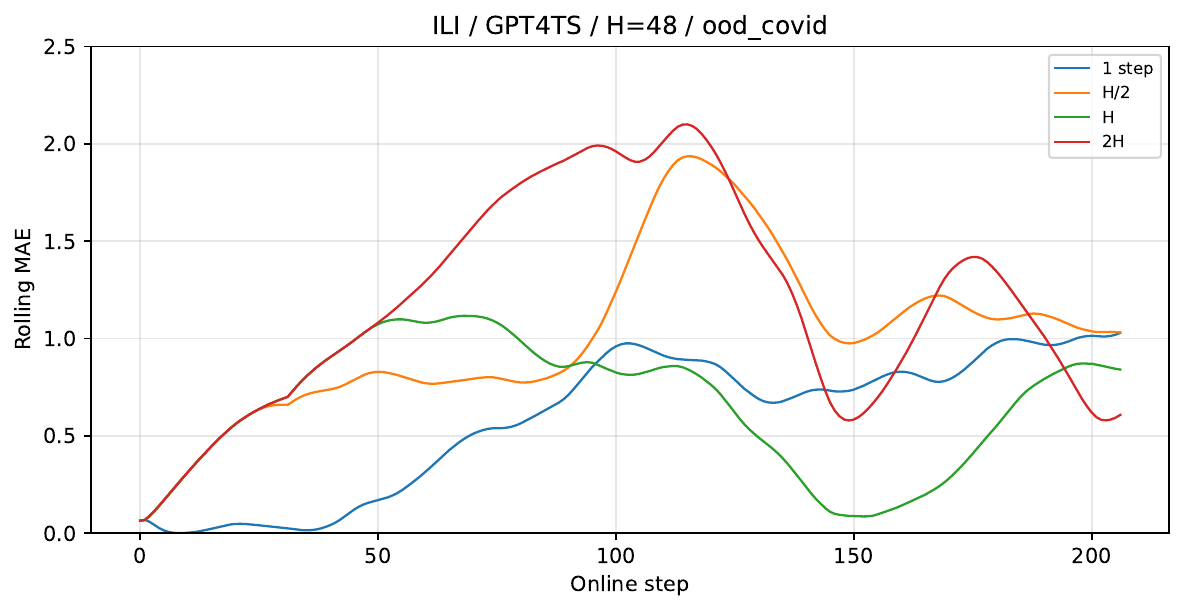}
    \end{minipage}
    \hfill
    \begin{minipage}{0.32\textwidth}
    \centering
    \includegraphics[width=\linewidth]{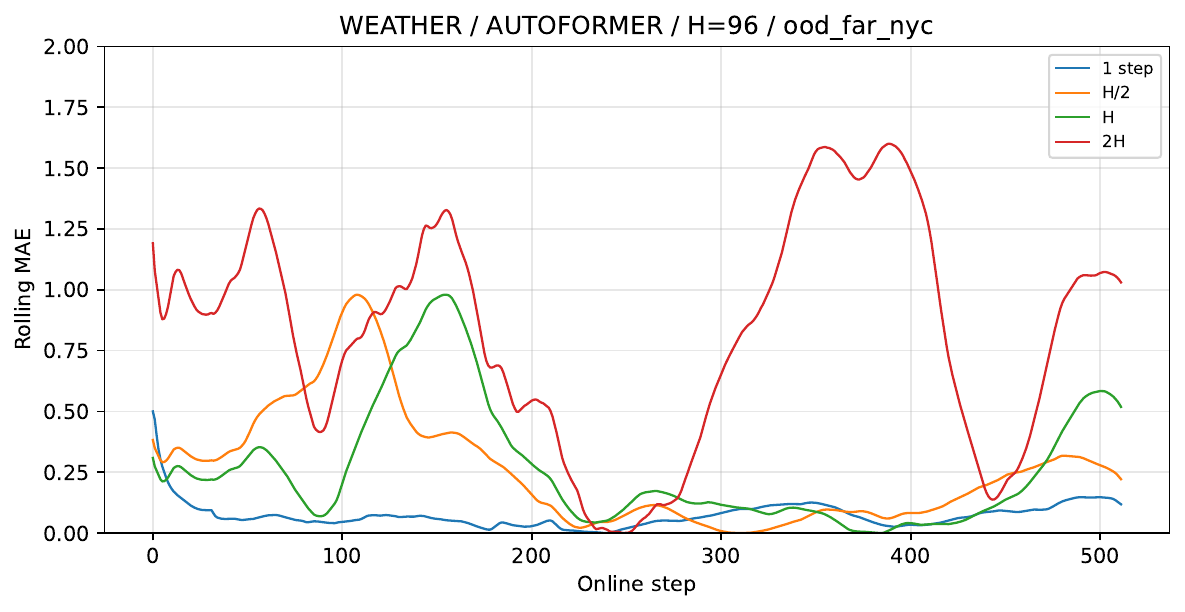}
    \end{minipage}
    \caption{Rolling MAE under artificial feedback delays for ETTh, ILI, and WEATHER-5K. Longer delays make supervision increasingly stale, while current target inputs remain available for shift sensing.}
    \label{fig:feedback_delay_curves}
    \end{figure}
    \FloatBarrier
    
    \subsection{Posterior Miscalibration}
    \label{sec:posterior_miscalibration}
    
    Posterior miscalibration changes the magnitude and selectivity of the certificate because disagreement is computed from
    posterior forecast samples. If the posterior variance is collapsed, samples can become too similar and disagreement may
    understate epistemic uncertainty. If the posterior variance is inflated, disagreement may become less selective and
    overstate uncertainty in benign periods. Figure~\ref{fig:posterior_miscalibration_diagnostics} visualizes this behavior
    by scaling the posterior variance while keeping the posterior mean fixed. The main temporal shift patterns remain
    visible, but the certificate scale changes, reinforcing that the certificate is a monitoring signal rather than a tight
    replacement for MAE/MSE.
    
    \begin{figure}[!htbp]
    \centering
    \begin{minipage}{0.32\textwidth}
    \centering
    \includegraphics[width=\linewidth]{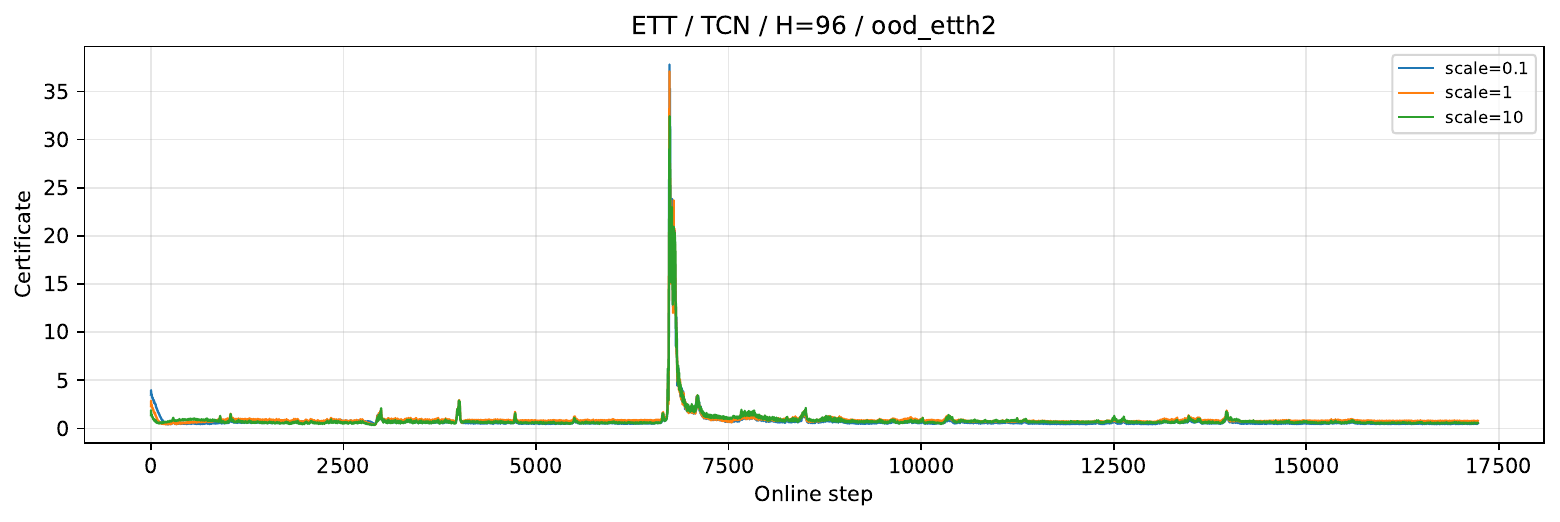}
    \end{minipage}
    \hfill
    \begin{minipage}{0.32\textwidth}
    \centering
    \includegraphics[width=\linewidth]{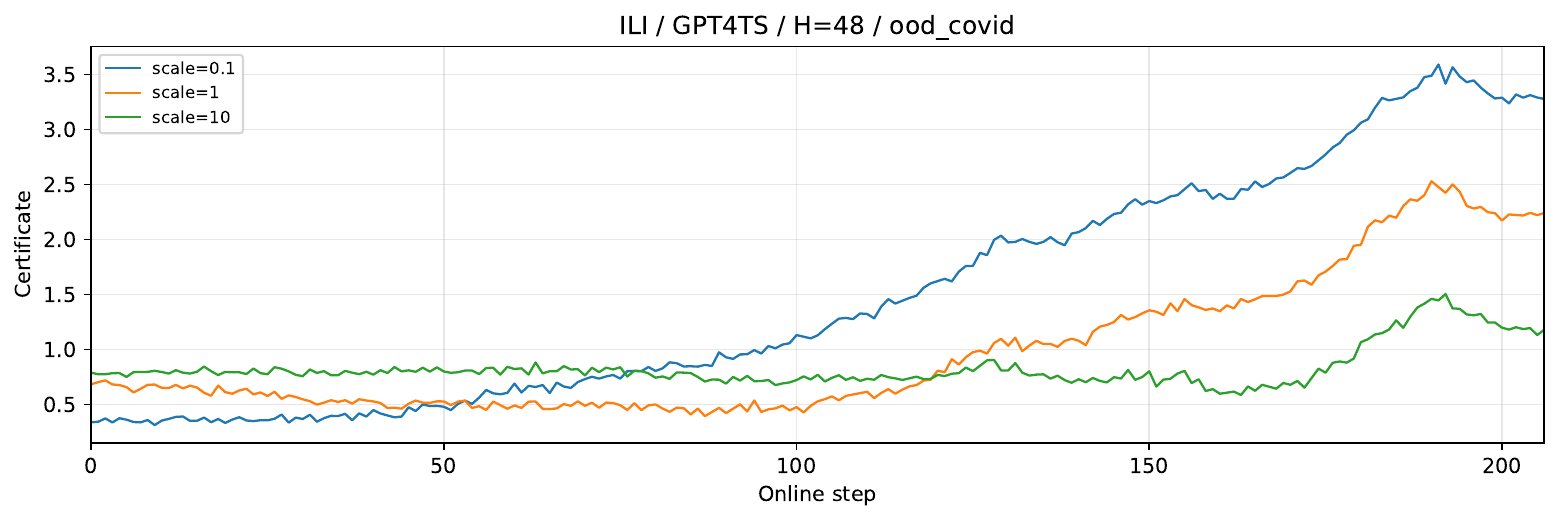}
    \end{minipage}
    \hfill
    \begin{minipage}{0.32\textwidth}
    \centering
    \includegraphics[width=\linewidth]{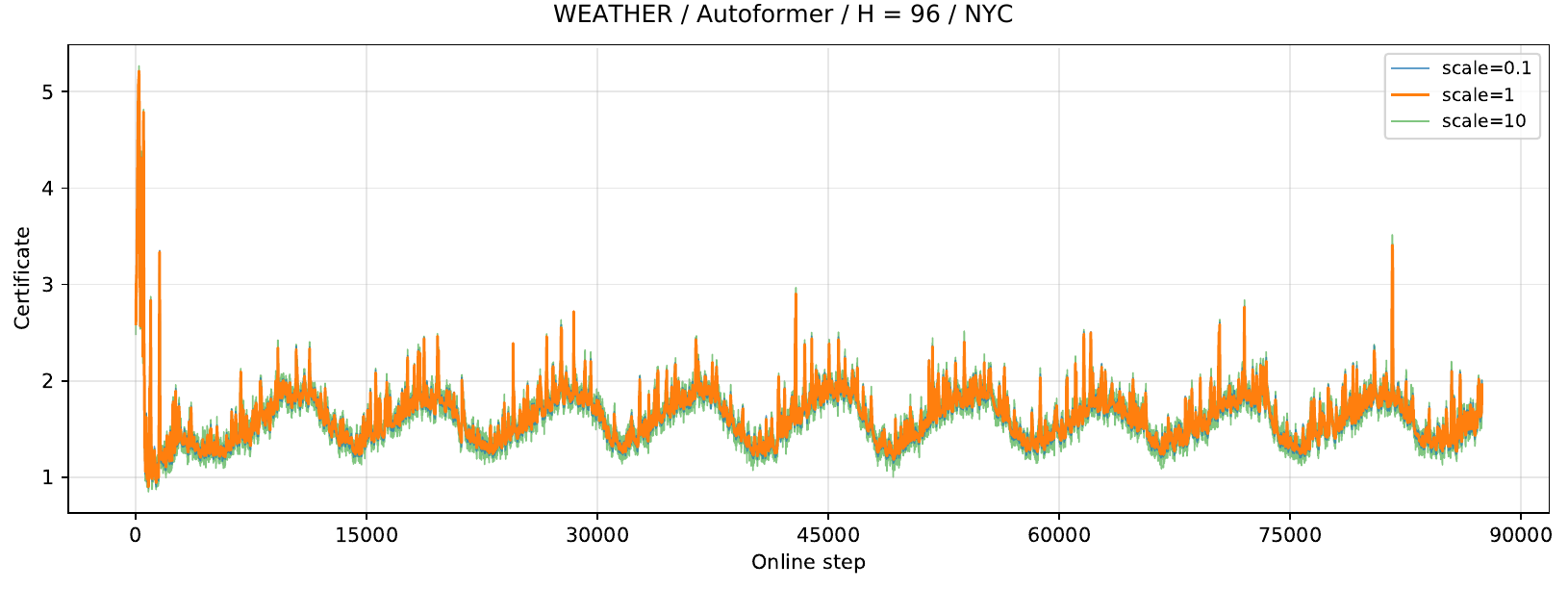}
    \end{minipage}
    \caption{Posterior-miscalibration diagnostics obtained by scaling the posterior variance while keeping the posterior mean fixed. The certificate scale and selectivity change with posterior variance, while the main temporal shift patterns remain visible.}
    \label{fig:posterior_miscalibration_diagnostics}
    \end{figure}
    \FloatBarrier

    \clearpage
    
    \subsection{Extended ETTh Experiment Result}
    \label{sec:extended_ett_result}
    
    Table~\ref{tab:etth_all_split} reports performance under covariate shift from ETTh1 to ETTh2 across
    short, medium, and long forecasting horizons. The shift causes large degradation for unadapted models, and the gap
    generally widens as the horizon increases, especially for the TCN and Autoformer backbones. Across all
    models and all horizons, OMPB achieves the lowest errors in both ID and OOD evaluations,
    indicating that the same calibration mechanism transfers reliably across backbone families. The improvements are most
    pronounced on the OOD ETTh2 split, where OMPB consistently reduces the large error increases seen in the
    original models and stabilizes long-horizon forecasting.
    \begin{table*}[h]
    \centering
    
    \caption{ETTh forecasting results (MAE | MSE) under In-Distribution (ID) and Out-of-Distribution (OOD) settings. Lower values are better, and the best results are shown in \textbf{bold}.}
    
    \label{tab:etth_all_split}
    \renewcommand{\arraystretch}{1}
    \setlength{\tabcolsep}{6pt}
    \begin{adjustbox}{max width=\textwidth}
    \begin{tabular}{cccccccc}
    \toprule
    \multirow{2}{*}{Backbone} &
    \multirow{2}{*}{Method} &
    \multicolumn{2}{c}{H=24} &
    \multicolumn{2}{c}{H=48} &
    \multicolumn{2}{c}{H=96} \\
    \cmidrule(lr){3-4}\cmidrule(lr){5-6}\cmidrule(lr){7-8}
    & & ID (ETTh1) & OOD (ETTh2) & ID (ETTh1) & OOD (ETTh2) & ID (ETTh1) & OOD (ETTh2) \\
    \midrule
    \multirow{5}{*}{TCN}
    & Original
    & 0.4588 | 0.4521 & 2.5145 | 8.9095
    & 0.5091 | 0.5286 & 2.5992 | 8.9657
    & 0.5330 | 0.5715 & 2.5635 | 8.9076 \\
    & SOLID
    & 0.4741 | 0.4333 & 0.7553 | 1.0175
    & 0.5106 | 0.4936 & 0.8347 | 1.2488
    & 0.5363 | 0.5389 & 0.8886 | 1.4424 \\
    & OneNet
    & 0.5838 | 0.6739 & 0.8234 | 1.3034
    & 0.6056 | 0.7339 & 0.8770 | 1.4688
    & 0.6606 | 0.8488 & 1.0073 | 1.9592 \\
    & PROCEED
    & 0.4917 | 0.5161 & 1.1002 | 0.7768
    & 0.5368 | 0.6213 & 0.8590 | 1.4590
    & 0.5889 | 0.7086 & 0.9608 | 1.8119 \\
    \cline{2-8}
    & OMPB
    & \textbf{0.3635} | \textbf{0.2724} & \textbf{0.5292} | \textbf{0.5659}
    & \textbf{0.3608} | \textbf{0.2629} & \textbf{0.5362} | \textbf{0.5689}
    & \textbf{0.3813} | \textbf{0.2800} & \textbf{0.5570} | \textbf{0.5912} \\
    \midrule
    \multirow{5}{*}{Autoformer}
    & Original
    & 0.5147 | 0.5605 & 1.1207 | 1.9961
    & 0.5215 | 0.5412 & 1.1473 | 2.0931
    & 0.5325 | 0.5651 & 1.3298 | 2.6778 \\
    & SOLID
    & 0.4899 | 0.4813 & 0.8410 | 1.2997
    & 0.5407 | 0.5463 & 0.8674 | 1.3298
    & 0.5446 | 0.5629 & 0.9501 | 1.6828 \\
    & OneNet
    & 0.4786 | 0.4815 & 0.7311 | 0.9580
    & 0.5288 | 0.5496 & 0.7748 | 1.0970
    & 0.5571 | 0.6026 & 0.8395 | 1.2940 \\
    & PROCEED
    & 0.5119 | 0.5967 & 0.7580 | 1.1151
    & 0.5821 | 0.7914 & 0.8628 | 1.7272
    & 0.6827 | 1.0923 & 1.0509 | 2.7848 \\
    \cline{2-8}
    & OMPB
    & \textbf{0.3814} | \textbf{0.3018} & \textbf{0.6785} | \textbf{0.8535}
    & \textbf{0.3812} | \textbf{0.2921} & \textbf{0.6738} | \textbf{0.8390}
    & \textbf{0.3845} | \textbf{0.2858} & \textbf{0.6033} | \textbf{0.6882} \\
    \midrule
    \multirow{5}{*}{GPT4TS}
    & Original
    & 0.3794 | 0.3217 & 0.5485 | 0.7119
    & 0.4093 | 0.3736 & 0.5966 | 0.8270
    & 0.4473 | 0.4167 & 0.6738 | 1.0241 \\
    & SOLID
    & 0.3763 | 0.3218 & 0.6172 | 0.7745
    & 0.4172 | 0.3855 & 0.6959 | 0.9761
    & 0.4451 | 0.4160 & 0.7499 | 1.1224 \\
    & OneNet
    & 0.3740 | 0.3184 & 0.6115 | 0.7512
    & 0.4092 | 0.3746 & 0.6800 | 0.9199
    & 0.4466 | 0.4146 & 0.7489 | 1.1222 \\
    & PROCEED
    & 0.4843 | 0.5082 & 0.7210 | 1.0268
    & 0.5506 | 0.6789 & 0.8118 | 1.3270
    & 0.6237 | 0.8802 & 0.9364 | 1.8907 \\
    \cline{2-8}
    & OMPB
    & \textbf{0.3156} | \textbf{0.2165} & \textbf{0.5242} | \textbf{0.6616}
    & \textbf{0.3283} | \textbf{0.2215} & \textbf{0.5605} | \textbf{0.7278}
    & \textbf{0.3535} | \textbf{0.2457} & \textbf{0.5708} | \textbf{0.7377} \\
    \midrule
    \end{tabular}
    \end{adjustbox}
    
    \begin{adjustbox}{max width=\textwidth}
    \begin{tabular}{cccccccc}
    \midrule
    \multirow{2}{*}{Backbone} &
    \multirow{2}{*}{Method} &
    \multicolumn{2}{c}{H=192} &
    \multicolumn{2}{c}{H=336} &
    \multicolumn{2}{c}{H=720} \\
    \cmidrule(lr){3-4}\cmidrule(lr){5-6}\cmidrule(lr){7-8}
    & & ID (ETTh1) & OOD (ETTh2) & ID (ETTh1) & OOD (ETTh2) & ID (ETTh1) & OOD (ETTh2) \\
    \midrule
    \multirow{5}{*}{TCN}
    & Original & 0.5918 | 0.6626 & 2.7849 | 10.5696 & 0.6193 | 0.7304 & 3.1380 | 13.7387 & 0.6362 | 0.7331 & 3.3334 | 14.9363 \\
    & SOLID    & 0.6933 | 0.9309 & 0.9004 | 1.5533 & 0.7110 | 0.9710 & 0.9488 | 1.7629 & 0.7789 | 1.1086 & 1.0423 | 2.0799 \\
    & OneNet   & 0.6924 | 0.8970 & 1.0933 | 2.4125 & 0.7082 | 0.9191 & 1.2228 | 3.3373 & 0.8616 | 1.2662 & 1.5163 | 4.6774 \\
    & PROCEED  & 0.6273 | 0.7819 & 0.9997 | 1.8750 & 0.6490 | 0.8225 & 1.1293 | 2.5799 & 0.8525 | 1.4724 & 1.5979 | 5.6582 \\
    \cline{2-8}
    & OMPB
    & \textbf{0.4206} | \textbf{0.3283} & \textbf{0.5766} | \textbf{0.6315}
    & \textbf{0.4484} | \textbf{0.3674} & \textbf{0.5981} | \textbf{0.6782}
    & \textbf{0.4850} | \textbf{0.4260} & \textbf{0.5802} | \textbf{0.6536} \\
    \midrule
    \multirow{5}{*}{Autoformer}
    & Original & 0.6170 | 0.7415 & 1.2245 | 2.4171 & 0.6830 | 0.8846 & 1.1962 | 2.3453 & 0.7516 | 1.0339 & 1.5399 | 3.5693 \\
    & SOLID    & 0.5640 | 0.6042 & 0.9946 | 1.8232 & 0.5937 | 0.6374 & 1.0386 | 2.0287 & 0.6616 | 0.7611 & 1.1005 | 2.2834 \\
    & OneNet   & 0.6001 | 0.6434 & 0.9121 | 1.5679 & 0.6140 | 0.6861 & 0.9812 | 1.8274 & 0.6736 | 0.7903 & 1.0546 | 2.0943 \\
    & PROCEED  & 0.7813 | 1.4344 & 1.3118 | 3.5772 & 0.8767 | 1.8344 & 1.5051 | 4.6450 & 1.2047 | 3.2827 & 1.5064 | 4.4675 \\
    \cline{2-8}
    & OMPB
    & \textbf{0.4147} | \textbf{0.3197} & \textbf{0.6308} | \textbf{0.7383}
    & \textbf{0.4373} | \textbf{0.3481} & \textbf{0.6461} | \textbf{0.7764}
    & \textbf{0.4997} | \textbf{0.4430} & \textbf{0.6085} | \textbf{0.6823} \\
    \midrule
    \multirow{5}{*}{GPT4TS}
    & Original & 0.4897 | 0.4879 & 0.7538 | 1.2570 & 0.5320 | 0.5588 & 0.8547 | 1.5203 & 0.6283 | 0.7315 & 0.8873 | 1.5625 \\
    & SOLID    & 0.4963 | 0.4980 & 0.8454 | 1.4495 & 0.5333 | 0.5625 & 0.9131 | 1.6823 & 0.6320 | 0.7414 & 1.0114 | 2.0592 \\
    & OneNet   & 0.4947 | 0.4884 & 0.8362 | 1.4102 & 0.5364 | 0.5572 & 0.8975 | 1.6239 & 0.6368 | 0.7551 & 0.9827 | 1.9062 \\
    & PROCEED  & 0.6914 | 1.0896 & 1.0830 | 2.8221 & 0.6911 | 1.0597 & 1.1482 | 2.8742 & 0.9177 | 1.7158 & 1.3659 | 3.8811 \\
    \cline{2-8}
    & OMPB
    & \textbf{0.3851} | \textbf{0.2832} & \textbf{0.5807} | \textbf{0.7453}
    & \textbf{0.4095} | \textbf{0.3158} & \textbf{0.5719} | \textbf{0.7304}
    & \textbf{0.4682} | \textbf{0.3983} & \textbf{0.5664} | \textbf{0.6837} \\
    \bottomrule
    \end{tabular}
    \end{adjustbox}
    \end{table*}
    \FloatBarrier
    \subsection{Extended WEATHER-5K Numerical Result}
    \label{sec:weather_result}
    Table~\ref{tab:weather_all_split} reports WEATHER-5K results under covariate shift from Miami to Atlanta and New York City (NYC) over multiple horizons. Across all horizons and backbones, OMPB achieves the lowest MAE and MSE on the ID Miami stream and both OOD streams, with the largest improvements appearing under the farther shift at NYC and longer horizons where errors accumulate. Competing adaptation baselines are more sensitive to model choice and shift severity, including instability for Autoformer under SOLID and PROCEED with very large MSE spikes, while OMPB remains stable. Overall, it supports that certificate-driven online calibration improves robustness under covariate shift without sacrificing ID accuracy.

    \begin{table*}[h]
    \centering
    
    \caption{WEATHER-5K forecasting results (MAE | MSE) under the In-Distribution (ID) location Miami and Out-of-Distribution (OOD) locations Atlanta and New York City (NYC). Lower values are better, and the best results are shown in \textbf{bold}.}
    
    \label{tab:weather_all_split}
    \renewcommand{\arraystretch}{1}
    \setlength{\tabcolsep}{5pt}
    \begin{adjustbox}{max width=\textwidth}
    \begin{tabular}{cccccccc}
    \toprule
    \multirow{2}{*}{Backbone} &
    \multirow{2}{*}{Method} &
    \multicolumn{3}{c}{H=24} &
    \multicolumn{3}{c}{H=48} \\
    \cmidrule(lr){3-5}\cmidrule(lr){6-8}
    & & ID (Miami) & OOD\_near (Atlanta) & OOD\_far (NYC)
      & ID (Miami) & OOD\_near (Atlanta) & OOD\_far (NYC) \\
    \midrule
    \multirow{5}{*}{TCN}
    & Original  & 0.3353 | 0.2700 & 0.5652 | 0.7230 & 0.8012 | 1.4215 & 0.3950 | 0.3520 & 0.6884 | 1.0718 & 1.0019 | 2.2040 \\
    & SOLID     & 0.3503 | 0.2979 & 0.6423 | 0.7885 & 0.8825 | 1.4050 & 0.4131 | 0.3943 & 0.7579 | 1.0732 & 1.0471 | 2.0462 \\
    & OneNet    & 0.3356 | 0.2740 & 0.6711 | 0.7662 & 0.9032 | 1.3966 & 0.4035 | 0.3651 & 0.8334 | 1.2209 & 1.1553 | 2.3156 \\
    & PROCEED   & 0.3546 | 0.2881 & 0.5658 | 0.5760 & 0.5988 | 0.6640 & 0.3656 | 0.3054 & 0.5783 | 0.5949 & 0.6099 | 0.6826 \\
    \cline{2-8}
    & OMPB      & \textbf{0.2650} | \textbf{0.1965} & \textbf{0.2952} | \textbf{0.2274} & \textbf{0.3487} | \textbf{0.3388}
                & \textbf{0.2822} | \textbf{0.2085} & \textbf{0.3077} | \textbf{0.2350} & \textbf{0.3538} | \textbf{0.3346} \\
    \midrule
    \multirow{5}{*}{Autoformer}
    & Original  & 0.4516 | 0.4274 & 0.6640 | 0.9252 & 0.8713 | 1.6349 & 0.4943 | 0.4997 & 0.7221 | 1.1036 & 0.9551 | 2.0064 \\
    & SOLID     & 0.6349 | 0.8385 & 4.0740 | 113.6623 & 5.6165 | 221.6814 & 0.6209 | 0.7811 & 1.3688 | 3.6325 & 1.7100 | 6.0316 \\
    & OneNet    & 0.4276 | 0.4094 & 0.7397 | 1.1584 & 0.9622 | 1.8210 & 0.5001 | 0.5410 & 0.8979 | 1.8387 & 1.1541 | 2.7226 \\
    & PROCEED   & 0.4011 | 1.0700 & 1.2095 | 61.4170 & 0.9569 | 8.3814 & 0.4590 | 1.8525 & 0.6699 | 4.0502 & 0.9280 | 10.1585 \\
    \cline{2-8}
    & OMPB      & \textbf{0.3284} | \textbf{0.2672} & \textbf{0.3298} | \textbf{0.2693} & \textbf{0.5240} | \textbf{0.7036}
                & \textbf{0.3229} | \textbf{0.2523} & \textbf{0.3249} | \textbf{0.2568} & \textbf{0.4148} | \textbf{0.4587} \\
    \midrule
    \multirow{5}{*}{GPT4TS}
    & Original  & 0.3488 | 0.3101 & 0.4985 | 0.6045 & 0.6424 | 1.0096 & 0.4564 | 0.4381 & 0.6332 | 0.9230 & 1.0861 | 2.5832 \\
    & SOLID     & 0.3499 | 0.3106 & 0.6218 | 0.7641 & 0.8670 | 1.4208 & 0.4156 | 0.4113 & 0.7441 | 1.0358 & 1.0435 | 2.0300 \\
    & OneNet    & 0.3406 | 0.2980 & 0.6174 | 0.7555 & 0.8401 | 1.3305 & 0.4081 | 0.3975 & 0.7444 | 1.0684 & 1.0267 | 1.9687 \\
    & PROCEED   & 0.3718 | 0.3268 & 0.5658 | 0.6219 & 0.7240 | 1.0042 & 0.4153 | 0.3909 & 0.6344 | 0.7504 & 0.7771 | 1.0887 \\
    \cline{2-8}
    & OMPB      & \textbf{0.2759} | \textbf{0.2143} & \textbf{0.3025} | \textbf{0.2435} & \textbf{0.3639} | \textbf{0.3793}
                & \textbf{0.2970} | \textbf{0.2296} & \textbf{0.3192} | \textbf{0.2568} & \textbf{0.3747} | \textbf{0.3690} \\
    \bottomrule
    \end{tabular}
    \end{adjustbox}
    \begin{adjustbox}{max width=\textwidth}
    \begin{tabular}{cccccccc}
    \toprule
    \multirow{2}{*}{Backbone} &
    \multirow{2}{*}{Method} &
    \multicolumn{3}{c}{H=96} &
    \multicolumn{3}{c}{H=192} \\
    \cmidrule(lr){3-5}\cmidrule(lr){6-8}
    & & ID (Miami) & OOD\_near (Atlanta) & OOD\_far (NYC)
      & ID (Miami) & OOD\_near (Atlanta) & OOD\_far (NYC) \\
    \midrule
    \multirow{5}{*}{TCN}
    & Original  & 0.4522 | 0.4378 & 0.8081 | 1.4961 & 1.1416 | 2.8791 & 0.4920 | 0.4984 & 0.8865 | 1.7839 & 1.2355 | 3.3662 \\
    & SOLID     & 0.4845 | 0.5072 & 0.8557 | 1.3057 & 1.1481 | 2.4192 & 0.5111 | 0.5541 & 0.9040 | 1.4319 & 1.2023 | 2.6443 \\
    & OneNet    & 0.4660 | 0.4582 & 1.0049 | 1.8004 & 1.4128 | 3.5001 & 0.5031 | 0.5196 & 1.0728 | 2.0400 & 1.5399 | 4.1293 \\
    & PROCEED   & 0.3738 | 0.3149 & 0.5960 | 0.6219 & 0.6329 | 0.7183 & 0.4869 | 0.4343 & 0.6242 | 0.6807 & 0.6999 | 0.9811 \\
    \cline{2-8}
    & OMPB      & \textbf{0.2949} | \textbf{0.2175} & \textbf{0.3168} | \textbf{0.2433} & \textbf{0.3545} | \textbf{0.3320}
                & \textbf{0.4510} | \textbf{0.4123} & \textbf{0.5649} | \textbf{0.6401} & \textbf{0.6597} | \textbf{0.8936} \\
    \midrule
    \multirow{5}{*}{Autoformer}
    & Original  & 0.5292 | 0.5700 & 0.7967 | 1.2791 & 0.9743 | 2.0091 & 0.5392 | 0.5874 & 0.8029 | 1.3280 & 1.0082 | 2.1582 \\
    & SOLID     & 0.5591 | 0.6568 & 0.9696 | 1.6412 & 1.2473 | 2.8002 & 0.5721 | 0.6874 & 1.0644 | 1.9112 & 1.3505 | 3.3024 \\
    & OneNet    & 0.5524 | 0.6476 & 0.9339 | 1.6906 & 1.2067 | 2.7802 & 0.5322 | 0.5974 & 0.9518 | 1.5729 & 1.2218 | 2.6202 \\
    & PROCEED   & 0.4829 | 1.4929 & 0.7353 | 2.4313 & 0.8631 | 2.4542 & 0.5015 | 2.3647 & 0.7083 | 2.9423 & 0.9928 | 4.7564 \\
    \cline{2-8}
    & OMPB      & \textbf{0.3476} | \textbf{0.2829} & \textbf{0.3815} | \textbf{0.3215} & \textbf{0.4344} | \textbf{0.4720}
                & \textbf{0.4880} | \textbf{0.4698} & \textbf{0.5200} | \textbf{0.5374} & \textbf{0.5905} | \textbf{0.7158} \\
    \midrule
    \multirow{5}{*}{GPT4TS}
    & Original  & 0.4732 | 0.4995 & 0.7388 | 1.2004 & 0.9390 | 2.0504 & 0.5064 | 0.5561 & 0.7973 | 1.3488 & 1.0145 | 2.3644 \\
    & SOLID     & 0.4744 | 0.5021 & 0.8504 | 1.3222 & 1.1637 | 2.5236 & 0.5094 | 0.5633 & 0.9344 | 1.5319 & 1.2622 | 3.0000 \\
    & OneNet    & 0.4635 | 0.4854 & 0.8440 | 1.3160 & 1.1198 | 2.2909 & 0.4984 | 0.5467 & 0.9183 | 1.4836 & 1.1933 | 2.5482 \\
    & PROCEED   & 0.4392 | 0.4322 & 0.6993 | 0.8949 & 0.8822 | 1.3816 & 0.4884 | 0.4663 & 0.7031 | 0.9022 & 0.8816 | 1.3714 \\
    \cline{2-8}
    & OMPB      & \textbf{0.3209} | \textbf{0.2534} & \textbf{0.3368} | \textbf{0.2694} & \textbf{0.3769} | \textbf{0.3553}
                & \textbf{0.4560} | \textbf{0.4372} & \textbf{0.5421} | \textbf{0.5984} & \textbf{0.6289} | \textbf{0.8003} \\
    \bottomrule
    \end{tabular}
    \end{adjustbox}
    
    \begin{adjustbox}{max width=\textwidth}
    \begin{tabular}{cccccccc}
    \toprule
    \multirow{2}{*}{Backbone} &
    \multirow{2}{*}{Method} &
    \multicolumn{3}{c}{H=336} &
    \multicolumn{3}{c}{H=720} \\
    \cmidrule(lr){3-5}\cmidrule(lr){6-8}
    & & ID (Miami) & OOD\_near (Atlanta) & OOD\_far (NYC)
      & ID (Miami) & OOD\_near (Atlanta) & OOD\_far (NYC) \\
    \midrule
    \multirow{5}{*}{TCN}
    & Original  & 0.5093 | 0.5240 & 0.9353 | 2.0136 & 1.3262 | 3.9330 & 0.5132 | 0.5364 & 0.9719 | 2.1831 & 1.3843 | 4.2720 \\
    & SOLID     & 0.5376 | 0.6089 & 0.9120 | 1.4499 & 1.2156 | 2.6454 & 0.5540 | 0.6342 & 0.9196 | 1.4763 & 1.2202 | 2.6069 \\
    & OneNet    & 0.5160 | 0.5493 & 1.0949 | 2.1728 & 1.5602 | 4.3384 & 0.5204 | 0.5536 & 1.1584 | 2.4916 & 1.6788 | 5.1249 \\
    & PROCEED   & 0.4975 | 0.4500 & 0.6416 | 0.7416 & 0.7074 | 1.3665 & 0.5086 | 0.5674 & 0.6667 | 0.8648 & 0.7342 | 1.9637 \\
    \cline{2-8}
    & OMPB      & \textbf{0.4638} | \textbf{0.4300} & \textbf{0.5854} | \textbf{0.6892} & \textbf{0.6791} | \textbf{0.9463}
                & \textbf{0.4787} | \textbf{0.4592} & \textbf{0.6266} | \textbf{0.8021} & \textbf{0.7163} | \textbf{1.0459} \\
    \midrule
    \multirow{5}{*}{Autoformer}
    & Original  & 0.5567 | 0.6227 & 0.8481 | 1.4951 & 1.0467 | 2.3192 & 0.5611 | 0.6281 & 0.8727 | 1.5566 & 1.0858 | 2.4740 \\
    & SOLID     & 0.5914 | 0.7236 & 1.4654 | 6.8197 & 1.9158 | 12.5096 & 0.6559 | 0.9700 & 1.2685 | 5.1476 & 1.6346 | 9.4656 \\
    & OneNet    & 0.6020 | 0.7549 & 1.0413 | 2.5055 & 1.3047 | 3.2655 & 0.5929 | 0.7233 & 1.0089 | 1.7897 & 1.3071 | 2.9676 \\
    & PROCEED   & 0.5366 | 2.8874 & 0.7090 | 1.6006 & 0.9550 | 2.2657 & 0.5575 | 2.8859 & 0.8005 | 2.1923 & 1.0430 | 3.9744 \\
    \cline{2-8}
    & OMPB      & \textbf{0.5043} | \textbf{0.5027} & \textbf{0.5402} | \textbf{0.5751} & \textbf{0.5965} | \textbf{0.7217}
                & \textbf{0.5406} | \textbf{0.5662} & \textbf{0.5873} | \textbf{0.6737} & \textbf{0.6485} | \textbf{0.8298} \\
    \midrule
    \multirow{5}{*}{GPT4TS}
    & Original  & 0.5299 | 0.5968 & 0.8238 | 1.4287 & 1.0442 | 2.4666 & 0.5603 | 0.6480 & 0.8760 | 1.5996 & 1.0967 | 2.6505 \\
    & SOLID     & 0.5340 | 0.6081 & 0.9351 | 1.5195 & 1.2599 | 2.8735 & 0.5581 | 0.6474 & 0.9803 | 1.6520 & 1.2957 | 2.8822 \\
    & OneNet    & 0.5243 | 0.5919 & 0.9394 | 1.5612 & 1.2302 | 2.6965 & 0.5444 | 0.6212 & 0.9607 | 1.5904 & 1.2453 | 2.6588 \\
    & PROCEED   & 0.5018 | 0.5205 & 0.7090 | 0.9175 & 0.9098 | 1.4921 & 0.5205 | 0.5570 & 0.7146 | 0.9361 & 0.9351 | 1.6773 \\
    \cline{2-8}
    & OMPB      & \textbf{0.4721} | \textbf{0.4654} & \textbf{0.5589} | \textbf{0.6342} & \textbf{0.6480} | \textbf{0.8550}
                & \textbf{0.4945} | \textbf{0.5039} & \textbf{0.6038} | \textbf{0.7484} & \textbf{0.6751} | \textbf{0.9105} \\
    \bottomrule
    \end{tabular}
    \end{adjustbox}
    \end{table*}
    
    \FloatBarrier

    \end{document}